%% file: paper.tex
\documentclass[11pt]{article}
\usepackage{amssymb}
\usepackage{amsmath}
\usepackage{fullpage}
\usepackage{fancybox}
\usepackage{graphicx}

\newenvironment{proof}{\noindent{\bf Proof.}}{\hfill$\Box$}
\newenvironment{sketch}{\noindent{\bf Proof Sketch.}}{\hfill$\Box$}
\newtheorem{theorem}{Theorem}
\newtheorem{definition}[theorem]{Definition}
\newtheorem{lemma}[theorem]{Lemma}
\newtheorem{corollary}[theorem]{Corollary}

\newcommand{\map}{\mbox{MAP}}

\newcommand{\xx}{\mathbf{x}}

\newcommand{\ww}{\mathbf{w}}

\title{The Complexity of Reasoning with FODD and GFODD\thanks{A preliminary version of this paper has appeared as \cite{HescottKh14aaai}. This paper includes a broader exposition and a significant amount of additional details in proofs and constructions required to obtain the technical results. 
}}
\author{Benjamin J. Hescott and Roni Khardon \\ Department of Computer Science \\ Tufts University \\ {\{hescott$|$roni\}}@cs.tufts.edu}
\date{\today}

\begin{document}
\maketitle

\begin{abstract}
Recent work introduced Generalized First Order Decision Diagrams (GFODD) as a knowledge representation that is useful in mechanizing decision theoretic planning in relational domains.  GFODDs generalize function-free first order logic and include numerical values and numerical generalizations of existential and universal quantification.  Previous work presented heuristic inference algorithms for GFODDs and implemented these heuristics in systems for decision theoretic planning.  In this paper, we study the complexity of the computational problems addressed by such implementations. In particular, we study the evaluation problem, the satisfiability problem, and the equivalence problem for GFODDs under the assumption that the size of the intended model is given with the problem, a restriction that guarantees decidability. Our results provide a complete characterization placing these problems within the polynomial hierarchy. The same characterization applies to the corresponding restriction of problems in first order logic, giving an interesting new avenue for efficient inference when the number of objects is bounded.  Our results show that for $\Sigma_k$ formulas, and for corresponding GFODDs, evaluation and satisfiability are $\Sigma_k^p$ complete, and equivalence is $\Pi_{k+1}^p$ complete.  For $\Pi_k$ formulas evaluation is $\Pi_k^p$ complete, satisfiability is one level higher and is $\Sigma_{k+1}^p$ complete, and equivalence is $\Pi_{k+1}^p$ complete.
\end{abstract}

\section{Introduction}

The complexity of inference in first order logic has been investigated intensively. It is well known that the problem is undecidable, and that this holds even with strong restrictions on the types and number of predicates allowed in the logical language. For example, the problem is undecidable for quantifier prefix $\forall ^2 \exists^*$ with a signature having a single binary predicate and equality \cite{Graedel03a}. Unfortunately, the problem is also undecidable if we restrict attention to satisfiability under finite structures \cite{Fagin93,Libkin04}. Thus, in either case, one cannot quantify the relative difficulty of problems without further specialization or assumptions. On the other hand, algorithmic progress in AI has made it possible to reason efficiently in some  cases. In this paper we study such problems under the additional restriction that an upper bound on the intended model size is given explicitly. 
This restriction is natural for many applications, where the number of objects is either known in advance or known to be bounded by some quantity. Since the inference problem is decidable under this restriction, we can provide a more detailed complexity analysis.

This paper is motivated by recent work on decision diagrams, known as FODDs and GFODDs, and the computational questions associated with them. Binary decision diagrams \cite{Bryant86,BaharFrGaHaMaPaSo93} are a successful knowledge representation capturing
functions over propositional variables, that allows for efficient manipulation
and composition of functions, and diagrams have been used in various 
applications in program verification and AI \cite{Bryant86,BaharFrGaHaMaPaSo93,HoeyStHuBo99}. Motivated by this success, several authors have attempted
generalizations to handle relational structure and first order
quantification \cite{GrooteTv03,WangJoKh08,SannerBo09,JoshiKeKh11}. In particular FODDs \cite{WangJoKh08} and their generalization GFODDs \cite{JoshiKeKh11} have been
introduced and shown to be useful in the context of decision theoretic
planning \cite{BoutilierRePr01,KerstingOtDe04,HolldoblerKaSk2006,HolldoblerSk2004} for problems with relational structure
\cite{JoshiKeKh10,JoshiKh11}. 

GFODDs can be seen to generalize the function-free portion of first order logic (i.e., signatures with constants but without higher arity functions) to allow for non-binary numerical values generalizing truth values, and for numerical quantifiers generalizing existential and universal quantification in logic. Efficient heuristic inference algorithms for such diagrams have been developed
focusing on the finite model case, and using the notion of ``reasoning from examples" \cite{KhardonRo94m,KhardonRo94l,KMR99}. This paper 
analyses the complexity of
the evaluation, satisfiability, and equivalence problems for such diagrams, focusing on the GFODD subset with $\min$ and $\max$ aggregation that are defined in the next section. To avoid undecidability and get a more refined classification of complexity, we study a restricted form of the problem where the finite size of the intended model is given as part of the input to the problem. As we argue below this is natural and relevant in the applications of GFODDs for solving decision theoretic control problems. The same restrictions can be used for the corresponding (evaluation, satisfiability and equivalence) problems in first order logic, but to our knowledge this has not been studied before.  We provide a complete characterization of the complexity showing an interesting structure. 
Our results are developed for the GFODD representation and require detailed arguments about the graphical representation of formulas in that language. 
The same lines of argument (with simpler proof details)
yield similar results for first order logic.
To translate our results to the language of logic, consider the quantifier prefix of a first order logic formula using the standard notation using $\Sigma_k$, $\Pi_k$ to denote alternation depth of quantifiers in the formula. 
With this translation, our results show that:
 
 (1) Evaluation over finite structures 
spans the polynomial hierarchy, that is, 
evaluation of $\Sigma_k$ formulas is $\Sigma_k^p$ complete, and 
evaluation of $\Pi_k$ formulas is $\Pi_k^p$ complete.

(2) Satisfiability, with a given bound on model size,  follows a different pattern:
satisfiability of $\Sigma_k$ formulas is $\Sigma_k^p$ complete, and 
satisfiability of $\Pi_k$ formulas is $\Sigma_{k+1}^p$ complete.

(3) Equivalence, under the set of models bounded by a given size, depends only on quantifier depth:
both the equivalence of $\Sigma_k$ formulas and 
equivalence of $\Pi_k$ formulas 
are
$\Pi_{k+1}^p$ complete.

The positive results allow for constants in the signature but the hardness results, except for 
satisfiability for $\Pi_1$ formulas, hold even without constants. For signatures without constants, satisfiability of 
$\Pi_1$ formulas is in NP; when constants are allowed, it is $\Sigma_{2}^p$ complete as in the general template.

These results are useful in that they clearly characterize the complexity of the problems solved heuristically by 
implementations of GFODD systems \cite{JoshiKeKh10,JoshiKh11} and can be used to partly motivate or justify the use of these heuristics.
For example, the ``model checking reductions" of \cite{JoshiKeKh11}   that simplify the structure of diagrams
replace equivalence tests with model evaluation on a ``representative" set of models.
When this set is chosen heuristically, as in \cite{JoshiKeKh10}, this 
leads to inference that is correct with respect to these models but otherwise incomplete.
Our results show that this indeed leads to a reduction of the complexity of the inference problem, so that the reduction in accuracy is traded for improved worst case run time. Importantly, it shows that without compromising correctness, the complexity of equivalence tests that are used to compress the representation will be higher. These issues and further questions for future work are discussed in the concluding section of the paper.

The rest of the paper is organized as follows. The next section defines FODDs and GFODDs and provides a more detailed motivation for the technical questions.
Section~\ref{sec:maxfodd} then develops the results for FODDs. We treat the FODD case separately for three reasons. First, this serves for an easy introduction into the results that avoids some of the more involved arguments that are required for GFODDs. Second, as will become clear, for FODDs we do not need the additional assumption on model size, so that the results are in a sense stronger. Finally, some of the proofs for GFODDs require alternation depth of at least two so that separate proofs are needed for FODDs in any case. Section~\ref{sec:gfodd} develops the results for GFODDs. The final section concludes with a discussion and directions for future work.

\input{prelim}

\input{maxfodd}

\input{gfodd}

\section{Discussion}

In this paper we explored the complexity of computations using FODD and GFODD using min and max aggregation, providing a classification placing them within the polynomial hierarchy,
where, roughly speaking, equivalence is one level higher in the hierarchy than evaluation and satisfiability.
These results are useful in that they clearly characterize the complexity of the problems solved heuristically by 
implementations of GFODD systems \cite{JoshiKeKh10,JoshiKh11,JoshiKhRaTaFe13} and can be used to partly motivate or justify the use of these heuristics.
For example, the ``model checking reductions" of \cite{JoshiKeKh10}   
replace equivalence tests with model evaluation on a ``representative" set of models, 
and  %
choosing this set heuristically
leads  to inference that is correct with respect to these models but otherwise incomplete.
Our results here show that this indeed leads to reduction of the complexity of the inference 
problem so that the reduction in accuracy is traded for improved worst case run time. 
As mentioned above, the proofs in the paper can be used (in simpler form) to show the same 
complexity results for the corresponding problems in first order logic. To our knowledge the 
complexity questions with an explicit bound on model size have not been previously studied 
in this context. Yet they can be useful in many contexts where such a bound can 
be given in a practical setting. For example, in such cases existing optimized QBF algorithms can be used for inference in this restricted form of first order logic. 

There are several important directions for further investigation. The first involves using a richer set of aggregation operators. In particular the definition of GFODDs allows for any function to aggregate values, and functions such as sum, product, and average are both natural and useful for modeling and solving Markov Decision Processes, which have been the main application for FODDs.  
The work of \cite{JoshiKhRaTaFe13} extends the model checking reductions to GFODDs with average aggregation.  Clarifying the complexity of these problems and identifying the best algorithms for them is an important effort for the efficiency of such systems. 
In this context it is also interesting to clarify the relationship to query languages in databases that allow for similar aggregations and to formulations of ``logic with counting" that has been developed in the context of database theory \cite{Libkin04}. 

Considering this wider family of GFODDs also raises new computational questions beyond the ones explored in this paper. 
One such question arises from the connection to 
statistical relational models and specifically to lifted inference in such models (see e.g.\ \cite{MLN,Problog,LiftedWMC}). In particular, consider Markov Logic Networks (MLN) \cite{MLN} that can be seen to define a distribution over first order structures through a log linear probability model, where features of this model are defined by simple first order formula templates. It is easy to show how to encode such templates and their weights using 
a GFODD with product aggregation, and how these can be combined using a variant of the Apply procedure. The main computational question in this context has been to calculate the probability of a query given the MLN model, and the number of objects $n$ in the domain. Let ${\cal I}$ be the set of models with $n$ objects over the relevant signature. In our case this question translates to calculating
$$ \sum_{I\in {\cal I}} \map_B(I)$$
for an appropriate $B$ that combines the query and the MLN model. This is closely related to the approaches that solve this problem via lifted weighted model counting \cite{LiftedWMC}. A similarly interesting question would require us to calculate the best $I$ for a particular $B$
$$ \mbox{argmax}_{I\in {\cal I}} \ \map_B(I).$$
In this case, if $B$ captures say profit of some organization, then the computation optimizes the setting so as to maximize profit.
Thirdly, we have defined a logic-inspired language but did not define or study any notion of implication. A natural notion of implication with numerical values, related to the one used by \cite{Gordon2009}, is majorization:
$$ B_1 \models B_2 \ \ \ \Leftrightarrow  \ \ \ \forall_{I\in {\cal I}} \  \map_{B_1}(I) \leq \map_{B_2}(I). $$
Efficient algorithms and complexity analysis for these new questions will expand the scope and applicability of GFODDs.

Finally, efficient algorithms for model evaluation play an important role in GFODD implementations.
The work of  \cite{JoshiKhRaTaFe13} provides a generic algorithm inspired by the variable elimination algorithm from probabilistic inference. Several application areas, including databases, AI, and probabilistic inference have shown that more efficient algorithms are possible when the input formula or graph have certain structural properties such as low graph width. 
We therefore conjecture that similar notions can be developed to provide more efficient evaluation for GFODDs with some structural properties. Coupled with model checking reductions, this can lead to realizations of GFODD systems that combine high expressive power going beyond first order logic with efficient algorithms. 

\section*{Acknowledgments} 
This work is partly supported by NSF grant IIS 0964457.
Some of this work was done when Roni Khardon was on sabbatical leave at Trinity College Dublin.

\bibliographystyle{plain}
\bibliography{foddbib}

\newpage
\appendix
\section{Detailed Proofs}

\input{minkGfoddSatProof}

\input{minGfoddSatProof}

\end{document}

%% file: prelim.tex
\section{FODDs and GFODDs and their Computational Problems}
\label{sec:prelim}

This section introduces the GFODD representation and associated computational problems, and explains how they are motivated by prior work on applying GFODDs in decision theoretic planning.  
We assume familiarity with basic concepts and notation in predicate logic \cite{Lloyd87,RussellNo95,ChangKe90} as well as basic notions from 
complexity theory \cite{HomerSelman01, sipser, papadimitriou94}.

Decision diagrams are similar to expressions in first order logic (FOL). They are defined relative to 
a relational signature, with a finite set of predicates $p_1, p_2, \ldots, p_n$ each with an associated arity (number of arguments), a countable set of variables $x_1, x_2, \ldots$, and a set of constants $c_1, c_2, \ldots, c_m$. We 
do not allow function symbols other than constants (that is, functions with arity $\geq 1$).
In addition, we assume that the arity of predicates is bounded by some numerical constant.
A term is a variable or constant and an atom is 
either an equality between two terms
or a predicate with an appropriate list of terms as arguments.
Intuitively, a term refers to an object in the world of interest 
and an atom is a property which is either true or false. 

To motivate the diagram representation consider first a simpler language of generalized expressions which we illustrate informally by some examples. In FOL we can consider open formulas that have unbound variables. For example, the atom $color(x,y)$ is such a formula and its truth value depends on the assignment of $x$ and $y$ to objects in the world. 
To simplify the discussion, we assume for this example that arguments are typed and $x$ ranges over ``objects" and $y$ over ``colors".
We can then quantify over these variables to get a sentence which will be evaluated to a truth value in any concrete possible world. For example, we can write $\exists y, \forall x, color(x,y)$ expressing the statement that there is a color associated with all objects. 
Generalized expressions allow for more general open formulas that evaluate to numerical values.
For example, 
$E_1=[\mbox{if } color(x,y) \mbox{ then 1 else 0}]$ is similar to the logical expression and 
$E_2 =[\mbox{if } color(x,y) \mbox{ then 0.3 else 0.5}]$ returns non binary values.
Quantifiers from logic are replaced with aggregation operators that combine numerical values
and provide a generalization of the logical constructs. In particular,
when the open formula is restricted to values 0 and 1, the operators $\max$ and $\min$ simulate existential and universal quantification.
Thus, 
$[\max_{y}, \min_{x},  \mbox{if } color(x,y) \mbox{ then 1 else 0}]$ 
is equivalent to the sentence above.
But we can allow for other types of aggregations. For example, 
$[\max_{y}, \mbox{sum}_{x},  \mbox{if } color(x,y) \mbox{ then 1 else 0}]$ 
evaluates to the largest number of objects associated with one color, and
$[\mbox{sum}_{x}, \min_{y},  \mbox{if } color(x,y) \mbox{ then 0 else 1}]$ 
evaluates to the number of objects that have no color association.
GFODDs are also related to work in statistical relational learning \cite{MLN,Problog,LiftedWMC}.
For example, 
if the expression $E_2$ captures probability of ground facts which are mutually independent then 
$[\mbox{product}_{x}, \mbox{product}_{y},  \mbox{if } color(x,y) \mbox{ then 0.3 else 0.5}]$ 
captures the joint probability for all such facts. Of course, the open formulas in logic can include more than one atom and similarly expressions can be more involved. In this manner, a generalized expression represents a function from possible worlds to numerical values. GFODDs capture the same set of functions but provide an alternative representation for the open formulas through directed graphs. GFODDs were introduced together with a set of operations that can be used to manipulate and combine functions and in this way provide a tool for computation with numerical functions over possible worlds. Prior work includes implementation of the FODD fragment where the only aggregation operator allowed is $\max$ \cite{JoshiKh11,JoshiKeKh10}
and more recently implementations for GFODDs with $\max$ and $\mbox{average}$ aggregations \cite{JoshiSKS12,JoshiKhRaTaFe13}.
In this paper we  investigate several computational questions for GFODDs with $\min$ and $\max$ aggregation.

\subsection{Syntax}
First order decision diagrams (FODD) and their generalization (GFODD) were
defined by \cite{WangJoKh08,JoshiKeKh11} inspired by previous work in
\cite{GrooteTv03}. 
GFODDs are composed of two parts, 
including  the aggregation functions and 
the open formula portion which is captured by a diagram or graph. 
The aggregation portion is given by a listing of the variables in the diagram 
in some arbitrary order $(w_{i_1},\ldots,w_{i_m})$ and a corresponding list of length $m$ specifying aggregation over each $w_{i_j}$. In this paper we restrict aggregation operators for each variable to be $\min$ or $\max$.
To reflect the structure of GFODDs,
and distinguish between aggregation list $V$ and the graph portion of a diagram $B$,
we sometimes denote 
a GFODD by $\langle V,B \rangle$. 
However, when clear from the context we use $B$ as a shorthand for $\langle V,B \rangle$.
FODDs are a special case of GFODDs where the aggregation function is $\max$ for all variables. Due to associativity and commutativity of $\max$, the aggregation function for FODDs does not need to be represented explicitly.

As in propositional decision diagrams \cite{Bryant86,BaharFrGaHaMaPaSo93}, the diagram portion 
is a rooted acyclic graph with directed
edges. Each node in the graph is labeled. A non-leaf node is labeled with an
atom from the signature and it has exactly two outgoing edges. The directed edges
correspond to the truth values of the node's atom. A leaf is labeled with a
non-negative numerical value. We sometimes restrict diagrams to have only binary leaves with values 0 or 1. In this case we can consider the values to be the logical values false and true.
An example diagram is shown in Figure~\ref{fig:hampath3}. In this diagram and all other diagrams in this paper, left going edges denote the true branch out of a node and right going edges represent the false branch.

Similar to the propositional case \cite{Bryant86,BaharFrGaHaMaPaSo93}, GFODD syntax is restricted to comply
with a predefined total order on atoms. In the propositional case the ordering constraint
yields a normal form (a unique minimal representation for each function) which is in turn the main source of efficient reasoning. 
For GFODDs, a normal form has not been established but the use of ordering makes
for more efficient simplification of diagrams. In particular, 
following  \cite{WangJoKh08},
we assume a fixed ordering 
on predicate names, e.g., $p_1 \prec p_2 \prec \ldots \prec p_n$, and a fixed ordering 
on variable names, e.g., $x_1 \prec x_2 \prec \ldots$ and constants  $c_1 \prec c_2 \prec \ldots c_m$ and require that 
$c_i\prec x_j$ for all $i$ and $j$.
The order is extended to atoms
by considering them as lists. That is, $p_i(\ldots) \prec p_j(\ldots)$ if $i<j$ and 
$p_i(x_{k_1},\ldots,x_{k_a}) \prec p_i(x_{k'_1},\ldots,x_{k'_a})$ 
if 
$(x_{k_1},\ldots,x_{k_a}) \prec (x_{k'_1},\ldots,x_{k'_a})$ in the lexicographic ordering over the lists. 
Node labels in the GFODD must obey this order so that if node $a$ is above node $b$ 
in the diagram
then the labels satisfy $a\prec b$.
The example of Figure~\ref{fig:hampath3} is ordered with predicate ordering $E\prec ``="$ and lexicographic variable ordering $v_1\prec v_2 \prec v_3$.

The ordering assumption is helpful when constructing systems using GFODDs because it simplifies the computations. Our complexity results hold in general, whether the assumption holds or not, therefore showing that while the assumption is convenient it does not fundamentally change the complexity of the problems. In particular, 
for the positive results, the algorithms showing membership in various complexity classes hold even in the more general case when the diagrams are not sorted. For the hardness results, the reductions developed hold even in the more restricted case when the diagrams are sorted. A significant amount of details in our analysis is devoted to handling ordering issues in hardness results.

Our complexity analysis will use the following classification of GFODD into subclasses.
We say that a GFODD is a $\max$-$k$-alternating GFODD if its set of aggregation operators has $k$ blocks of aggregation operators, where the first includes $\max$ aggregation, the second includes $\min$ aggregation, and so on. We similarly define $\min$-$k$-alternating GFODD where the first block has $\min$ aggregation operators. A GFODD has aggregation depth $k$  if it is in one of these two classes.

\subsection{Semantics}
Diagrams, like first order formulas, are evaluated in possible worlds that provide an interpretation of their symbols.\footnote{Possible worlds are 
known in the literature under various names including {\em first order structures}, {\em first order models}, and {\em interpretations}. In this paper we use the term interpretations.}
In particular, a possible world or {\em Interpretation} , $I$, 
specifies a domain of
objects, an assignment of each constant in the signature to an object in the domain, and the truth values of predicates over these objects.

\begin{figure}[t]
\centering
\includegraphics[scale = 0.95]{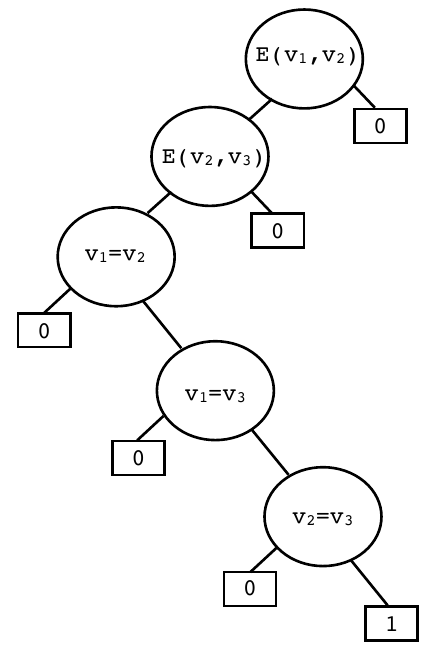}
\caption{An example FODD. In this and all other diagrams in this paper, left going edges represent the true branch out of a node and right going edges represent the false branch. 
Interpreting $E$ as an edge relation of a graph, this FODD tests that the graph has a Hamiltonian path of length 3.}
\label{fig:hampath3}
\end{figure}

The semantics assigns a value, denoted  $\map_B(I)$, for any diagram $B$ on any interpretation $I$ by
considering all possible valuations.
A variable valuation $\zeta$ is a mapping from the set of variables in $B$ 
to domain elements in the interpretation $I$. This
mapping assigns each node label to a concrete (``ground'') atom in the interpretation
and therefore to its truth value and in this way defines a single path from root
to leaf. The value of this leaf is the value of the GFODD $B$ under the
interpretation, $I$, with variable valuation $\zeta$ and is denoted $\map_B(I,
\zeta)$. 
The final value, $\map_B(I)$, is defined by aggregating over $\map_B(I,\zeta)$.
In particular, considering the aggregation order $(w_{i_1},\ldots,w_{i_m})$
we loop with $j$ taking values from $m$ to 1 aggregating values over $w_{i_j}$ using its aggregation operator. 
We denote this by $\map_B(I) = AG_{\zeta}\map_B(I,\zeta)$ where for the special case of FODDs 
this yields $\map_B(I) = \max_{\zeta}\map_B(I,\zeta)$.

Consider evaluating the FODD example in Figure~\ref{fig:hampath3} on interpretation
$I=([1,2,3],\{E(1,3),$ $E(3,1),$ $E(1,2),$ $E(2,1)\})$. Then for 
$\zeta= \{v_1/1, v_2/2, v_3/3\}$ we have $\map_B(I,\zeta)=0$
but for 
$\zeta= \{v_1/3,$ $v_2/1,$ $v_3/2\}$ we have $\map_B(I,\zeta)=1$
and therefore $\map_B(I) = \max_{\zeta}\map_B(I,\zeta)=1$.

\subsection{Computations with GFODDs}

The GFODD representation was introduced as a tool for mechanizing and solving decision problems given by structured Markov Decision Processes (MDP), also known as Relational MDP or First Order MDP. A detailed exposition is beyond scope of this paper (see  \cite{WangJoKh08,JoshiKeKh11}). This section provides some necessary technical details 
and some background to motivate the computational problems investigated in the paper.
In this context, a planning problem world state can be described using an interpretation providing the objects in the world and the relations among them. An action moves the world from one state to another, where in MDPs this transition is non-deterministic. The so-called $Q$ function $Q(s,a)$ provides a quality estimate of each action $a$ in each state $s$. Using this function, one can control the MDP by picking $a=\mbox{argmax}_a Q(s,a)$ in state $s$. There are several algorithms to calculate such $Q$ functions and previous work has introduced GFODDs as a compact representation for these functions. This is done by implementing a symbolic version of the well known Value Iteration (VI) algorithm, where the symbolic algorithm operates by manipulating GFODDs. 
Action selection provides our first computational question, that is, evaluating $Q(s,a)$.
In our context, this means calculating $\map_B(I)$ where $I$ captures $s$ and $a$ and $B$ is the representation of the $Q$ function. The same computational problem occurs in several other steps in the symbolic VI algorithm. We define this problem below as GFODD Evaluation.

Recall that a GFODD represents a function from interpretations to real values. One of the main operations required for the symbolic VI algorithm is combination of such functions. In particular, let $f_1$ and $f_2$ be functions represented by two GFODDs, and let $\odot$ be any binary operation over real values (e.g., plus). The combination operation returns a GFODD representing a function $f_3$ such that for all $I$ we have $f_3(I)=f_1(I) \odot f_2(I)$.
That is, $f_3$ is a symbolic representation of the pointwise operation over function values of $f_1$ and $f_2$.
Note that since $f_1$ and $f_2$ are closed expressions we can standardize apart their variables before taking this operation. 

\begin{figure}[t]%
\vskip 0.2in
\begin{minipage}[][][t]{0.48\textwidth}
\includegraphics[scale = 0.48]{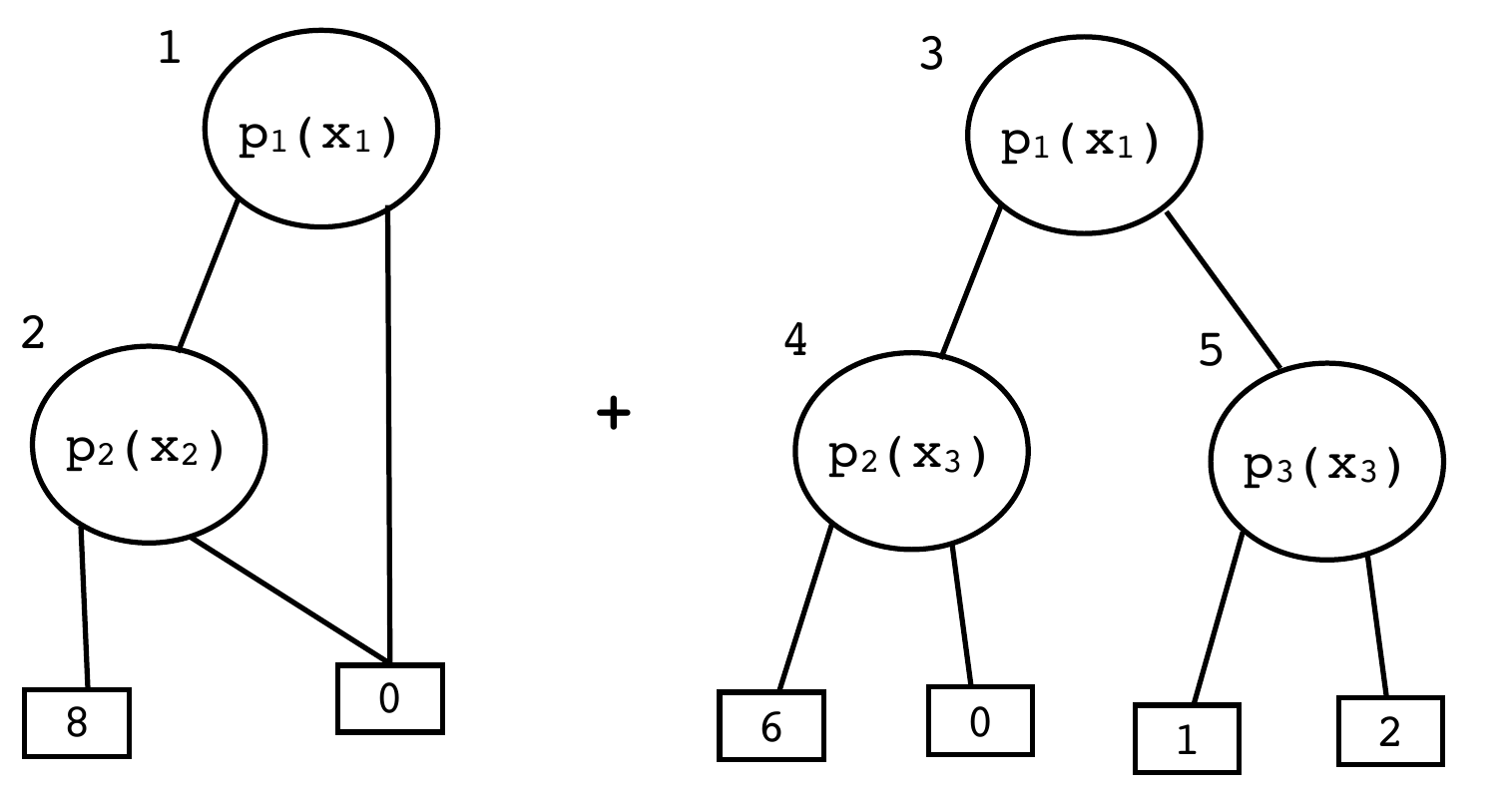}
\end{minipage}
\begin{minipage}[][][c]{0.48\textwidth}
\includegraphics[scale = 0.48]{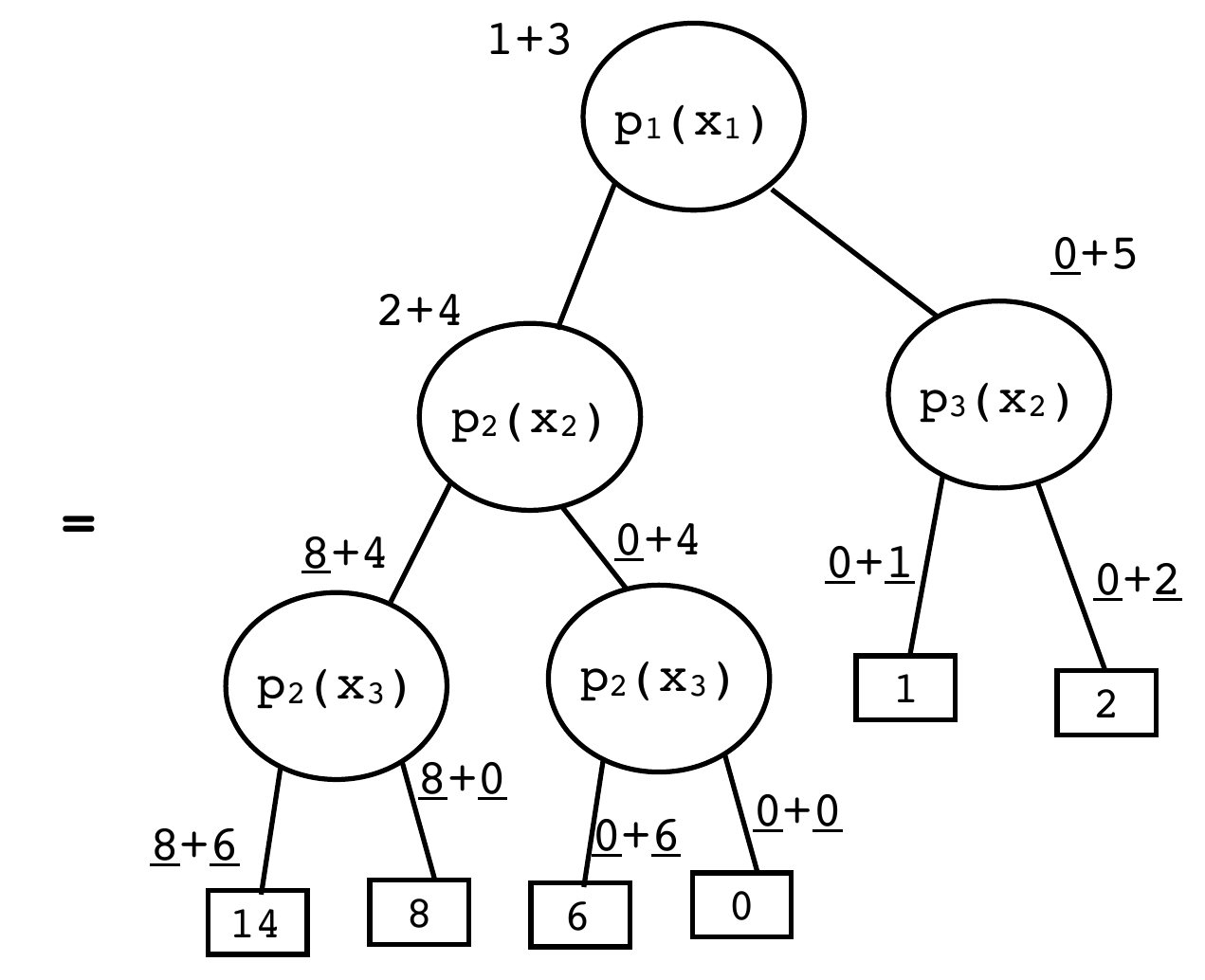}
\end{minipage}
\caption{Overview of the Apply procedure 
of  \cite{WangJoKh08}
for binary operations over two diagrams.
Recall that GFODDs use an ordering over the atoms labeling nodes, so that
atoms lower in the ordering are always higher in the diagram.
Let $p$ and $q$ be the roots of $B_1$ and $B_2$ respectively.
The procedure chooses a new root label (the smaller of $p$ and $q$) 
and recursively combines the corresponding
sub-diagrams, according to the relation between the two labels ($\prec$,
$=$, or $\succ$).
In this example, we
assume predicate ordering $p_1 \prec p_2 \prec p_3$, and parameter ordering $x_1
\prec x_2 \prec x_3$.  
Note that the two diagrams share the argument $x_1$ as would be the case with constants but $x_2$ and $x_3$ are unique to one diagram.
In this example,
non-leaf nodes are annotated with numbers and numerical leaves
are underlined for identification during the execution trace.  For example,
the top level call adds the functions corresponding to nodes 1 and 3. Since
the label at these nodes, $p_1(x_1)$, is identical, we add the two left children and the two right children respectively. 
To illustrate another case, consider the
node marked $2+4$. Here we pick the smaller label from 2, and then
add both left and right child of node 2 to node 4. These
calls are performed recursively and dynamic programming is used to avoid repeated recursive calls where such repetitions arise (this does not happen in the current example).}
\label{fig:combineFODD}
\end{figure} 

Figure~\ref{fig:combineFODD} shows how to combine the diagram portions (i.e., the open expressions) in a semantically coherent manner using the Apply procedure of \cite{WangJoKh08}. The following theorem identifies conditions for correctness of Apply when used with closed expressions.
We say that a binary operation $\odot$ is safe with respect to aggregation operator $agg$ if it distributes with respect to it, that is 
$b \odot agg\{a_1,a_2,\ldots,a_n\} = agg\{(a_1\odot b),(a_2\odot b),\ldots,(a_n\odot b)\}$.  
A list of safe pairs of binary operations and aggregation operators 
was provided by
\cite{JoshiKeKh11}. For the arguments of this paper we recall 
that the binary operations $+$ and $\wedge$ are safe with respect to max and min aggregation. 
For example $5 + \max\{1,2,3,4\}= \max\{6,7,8,9\}$. With this definition we have:

\begin{theorem}[see Theorem 4 of \cite{JoshiKeKh11}]
\label{thm:gfodd-combine}
Let $B_1$ $=$ $\langle V_1, D_1\rangle$ and $B_2$ $=$ $\langle V_2,
D_2\rangle$ be GFODDs that do not share any variables and assume that $op_c$
is safe with respect to all operators in $V_1$ and $V_2$. Let $D$ $=$
apply($B_1, B_2, op_c$). Let $V$ be any permutation of the list of variable in $V_1$ and $V_2$
so long as the relative order of operators in $V_1$ and $V_2$ remains
unchanged, and let $B$ $=$ $\langle V, D\rangle$. Then for any interpretation
$I$, $\map_B(I)$ $=$ $\map_{B_1}(I)\ op_c\  \map_{B_2}(I)$.  
\end{theorem}

Therefore, when adding (or taking the logical-and of) functions represented by diagrams that are standardized apart we 
can use the Apply procedure on the graphical representations of these functions, and at the same time we 
have some flexibility in putting together their list of aggregation functions. This will be useful in our reductions.

The Apply procedure can introduce redundancy into diagrams. By this we mean that a simpler syntactic form, often a sub-diagram, can represent the same function. 
To illustrate, consider the diagrams of Figure~\ref{fig:combineFODD} as FODDs (i.e., with $\max$ aggregation) 
and consider edge marked by $\underline{0}+4$. It is easy to see that this edge can be redirected to a leaf with value zero without changing $\map_B(I)$ for any $I$. This is true because if we can reach the leaf with value 6 using some valuation then we can also reach the leaf with value 14 using another valuation because $x_2$ is not constrained. Therefore, the $\max$ aggregation will always ignore valuations reaching value 6.
It is also easy to see that the edge marked $\underline{8}+4$ can be redirected to value 14 without changing $\map_B(I)$.
Simplification\footnote{
Simplification was called {\em reduction} by \cite{WangJoKh08}; to avoid confusion with the standard complexity theory meaning of the term reduction we use the term simplification instead. 
} 
of diagrams by removing unnecessary portions is crucial for efficiency of GFODD implementations and a significant amount of previous work was devoted to mechanizing this process. 
Note that it is most natural to keep the aggregation portion fixed and simply manipulate the diagram portion. 
In this paper we abstract this process as testing for GFODD Equivalence, that is, testing whether the diagram is equivalent to a second simpler one. Motivated by the focus in the implementations on algorithms that remove one edge at a time, as illustrated in the example, we also formalize this special case.

\subsection{Complexity Theory Notation}

Recall that the 
polynomial hierarchy is defined from P, NP, and co-NP using an inductive constraction with reference to computation with oracles \cite{HomerSelman01, sipser, papadimitriou94}. 
In particular we have that
$\Sigma_1^p=$NP, and $\Pi_1^p=$co-NP.
An algorithm is in the class $A^B$ if it uses computation in $A$ with a polynomial number of calls to an oracle for a problem in class $B$. 
Then we have $\Sigma_{k+1}^p=\mbox{NP}^{\Sigma_{k}^p}$, and $\Pi_{k+1}^p=\mbox{co-NP}^{\Sigma_{k}^p}$.
A problem is in $\Sigma_{k}^p$ iff its complement is in $\Pi_{k}^p$ and thus (since the oracle always answers deterministically and correctly) either of these can serve as the oracle in the definition.

\subsection{Computational Problems}

Before defining the computational problems we must define the representation of inputs. We assume that GFODDs are given using a list of aggregation operators and associated variables and a labelled graph representation of the diagram. This is clearly polynomially related to the number of variables and number of nodes in the GFODD.
Some of our problems require interpretations as input. 
Here we assume a finite domain so as to avoid issues of representing the interpretation. Thus an interpretation is given as a list of objects serving as domain elements, a list specifying the mapping of constants to objects, and the extension of each predicate on these objects. Given that the signature is fixed and the arity of each predicate is constant, this implies that the size of $I$ is polynomially related to the number of objects in $I$. 
As illustrated in the example of Figure~\ref{fig:hampath3}, a graph $G=(V,E)$ can be seen as an interpretation with domain $V$ and with one predicate formed by the edge relation.
We can now define the computational problems of interest. 
We separate the definitions for FODDs and GFODDs because for GFODDs the unrestricted problems are undecidable and they require further refinement.
The simplest problem requires us to evaluate a diagram on a given interpretation. 

\begin{definition}[FODD Evaluation]
Given diagram $B$, interpretation $I$ with finite domain, and value $V\geq 0$: return Yes iff
$\map_B(I)\geq V$.  In the special case when the leaves are restricted to $\{0,1\}$ and $V=1$
this can be seen as a  returning Yes iff $\map_B(I)$ is true. 
\end{definition}

To calculate $\map_B(I)$ we can ``run" a procedure for FODD Evaluation multiple times, once for each leaf value as $V$, and return the highest achievable result. 
Thus, if FODD Evaluation is in complexity class $A$, we can calculate the function value in $P^A$. This fact is used several times in our constructions.

Since diagrams generalize FOL it is natural to investigate satisfiability:

\begin{definition}[FODD Satisfiability]
Given diagram $B$ with leaves in $\{0,1\}$: return Yes iff there is some $I$ such that 
$\map_B(I)$ is true.  
\end{definition}

When $B$ has more than two values in its leaves the satisfiability problem becomes:

\begin{definition}[FODD Value]
Given diagram $B$ and value $V \geq 0$: return Yes iff there is some $I$ such that 
$\map_B(I)= V$.  
\end{definition}

Notice that FODD Value requires that $V$ is achievable but no value larger than $V$ is achievable on the same $I$ and, as the proofs below show, the extra requirement  makes the problem harder. 
On the other hand, if we replace equality with $\geq V$ in FODD Value, the problem is equivalent to FODD Satisfiability because we can simply replace leaf values in the diagram with 0,1 according to whether they are~$\geq V$.

Finally, as motivated above, we investigate the simplification problem and its special case with single edge removal.

\begin{definition}[FODD Equivalence]
Given diagrams $B_1$ and $B_2$: return Yes iff 
$\map_{B_1}(I)=\map_{B_2}(I)$  for all $I$.   
\end{definition}

\begin{definition}[FODD Edge Removal]
Given diagrams $B_1$ and $B_2$, where $B_2$ can be obtained from $B_1$ by redirecting one edge to a zero valued leaf: return Yes iff 
$\map_{B_1}(I)=\map_{B_2}(I)$  for all $I$.   
\end{definition}

Given the discussion above, GFODDs with binary leaves can be seen to capture the function free fragment of  first order logic with equality. It is well known that satisfiability and therefore also equivalence of expressions in this fragment of  first order logic is not decidable. In fact, the problem is undecidable even for very restricted forms of quantifier alternation (see survey and discussion in \cite{Graedel03a}). For example, the problem is undecidable for quantifier prefix $\forall ^2 \exists^*$ with a single binary predicate and equality. The problem is also undecidable if we restrict attention to satisfiability under finite structures. Therefore, without further restrictions, we cannot expect much by way of classification of the complexity of the problems stated above for GFODDs.

We therefore restrict the problems so that the size of interpretations is given as part of the input. 
This makes the problems decidable and reveals the structure promised above. There are two motivations for using such a restriction. The first is that in some applications we might know in advance that the number of relevant objects is bounded by some large constant. 
For example, the main application of GFODDs to date has been for solving decision theoretic planning problems; in this context
the number of objects in an instance 
(e.g., the number of trucks or packages in a logistics transportation problem)
might be bounded by some known quantity.
The second is that our results show that even under such strong conditions the computational problems are hard, providing some justification for the heuristic approaches used in FODD and GFODD implementations \cite{JoshiKeKh10,JoshiKh11,JoshiKhRaTaFe13}.

\begin{definition}[GFODD Model Evaluation]
Given diagram $B$, interpretation $I$ with finite domain, and value $V\geq 0$: return Yes iff
$\map_B(I)\geq V$.  Note that when the leaves are restricted to $\{0,1\}$ and $V=1$
this can be seen as a  returning Yes iff $\map_B(I)$ is true. 
\end{definition}

\begin{definition}[GFODD Satisfiability]
Given diagram $B$ with leaves in $\{0,1\}$ and integer $N$ in unary: return Yes iff there is some $I$, with at most $N$ objects, such that 
$\map_B(I)$  is true.  
\end{definition}

\begin{definition}[GFODD Value]
Given diagram $B$, integer $N$ in unary and value $V \geq 0$: return Yes iff there is some $I$, with at most $N$ objects,  such that 
$\map_B(I)= V$.  
\end{definition}

\begin{definition}[GFODD Equivalence]
Given diagrams $B_1$ and $B_2$ (with the same aggregation functions) and integer $N$ in unary: return Yes iff  
for all $I$ with at most $N$ objects, $\map_{B_1}(I)=\map_{B_2}(I)$.   
\end{definition}

\begin{definition}[GFODD Edge Removal]
Given diagrams $B_1$ and $B_2$  (with the same aggregation functions), where 
where $B_2$ can be obtained from $B_1$ by redirecting one edge to a zero valued leaf,
and given integer $N$ in unary: return Yes iff  
for all $I$ with at most $N$ objects, $\map_{B_1}(I)=\map_{B_2}(I)$.   
\end{definition}

Since we are assuming a fixed arity $k$,
the assumption that $N$ is in unary is convenient because it implies that the size of an intended interpretation $I$ is polynomial in $N$. Therefore, an algorithm for these problems can explicitly represent an interpretation of the required size and test it. Our hardness results use $N$ which is at most linear in the size of the corresponding diagram $B$.

%% file: maxfodd.tex
\section{The Complexity of Reasoning with FODD}
\label{sec:maxfodd}

In this section we develop the complexity results for the special case of FODDs.
Evaluation of FODDs is essentially the same as evaluation of conjunctive queries in databases and can be analyzed similarly. We include the argument here for completeness.

\begin{theorem}
\label{thm:maxfodd-eval}
{ FODD Evaluation} is $NP$-complete.
\end{theorem}
\begin{proof}
Membership in NP is shown by the algorithm that guesses a valuation $\zeta$, calculates $\map_B(I,\zeta)$ and returns Yes iff the leaf reached has value $\geq V$. Yes is returned iff some valuation yields a value $\geq V$ as needed.

For hardness we reduce the directed Hamiltonian path to this problem. As illustrated in Figure~\ref{fig:hampath3}, given the number of nodes in a graph we can represent a
generic Hamiltonian path verifier as a FODD $B$. To do this we simply produce a left going path $E(x_1,x_2), E(x_2,x_3),\ldots, E(x_{n-1},x_n)$ which verifies existence of the edges, followed by equality tests to verify that all nodes are distinct.  All ``failure exits" on this path go to 0 and the success exit of the last test yields 1. Call this diagram $B$. This diagram is ordered with $E\prec ``="$ and lexicographic ordering over arguments.
Now, given any input $G$ for Hamiltonian path, we represent it as an interpretation $I$ and produce $(B,I, 1)$ as the input for FODD Evaluation. Clearly, 
$G$ has a Hamiltonian path iff $\map_B(I)=1$.
\end{proof}

\bigskip
The other results for FODDs rely on the existence of small models: 

\begin{lemma}
\label{lemma:foddsize}
For any FODD $B$ with $k$ variables and constants, 
if $\map_B(I) = V$ for some $I$ then there is an interpretation $I'$ with at most $k$ objects such that
$\map_B(I') = V$.
\end{lemma}
\begin{proof}
Let $I$ be as in the statement. Then there is a valuation 
$\zeta$ such that $\zeta$ reaches a leaf valued $V$ in $B$. Let $I'$ be an interpretation including the objects that are used in the path traversed by $\zeta$ where the truth value of any predicate over arguments from these objects agrees with $I$. We have that $I'$ has at most $k$ objects, $\zeta$ is a suitable valuation for $I'$ and $\map_B(I',\zeta) = V$. In addition, no other valuation $\zeta'$ leads to a value larger than $V$ because, if it did, the same value would be achievable in $I$. Hence, $\map_B(I') = \map_B(I',\zeta) = V$.
\end{proof}

\begin{theorem}
\label{thm:maxfodd-sat}
{ FODD Satisfiability} is NP-Complete.
\end{theorem}
\begin{proof} 
For membership we can guess an interpretation $I$, which by the previous lemma can be small, and guess a valuation $\zeta$ for that interpretation. We return Yes if and only if $\map_B(I,\zeta)=1$. 

We show hardness with a reduction from 3SAT. Let $f$ be an arbitrary 3CNF formula. We create a new FODD variable for each literal in the CNF so that $v_{(i,j)}$ corresponds to the $j$th literal in the $i$th clause. 

Our FODD will have three portions connected in a chain.
The first portion checks that the predicate $P_T()$ in the interpretation can be used to simulate Boolean assignments. To achieve this, we first ensure that the interpretation has 
at least two different objects, referred to by variables $y_1$ and $y_2$.
We then use a small block that ensures that the truth value of $P_T(y_1)$ is not equal to $P_T(y_2)$. 
As a result $P_T(y_1)$ and $P_T(y_2)$  correspond to true and false logical values.
This is shown in Figure~\ref{fig:fodd-sat-pt-verify}.

The second portion ensures that if $v_{(i,j)}$ and   $v_{(i',j')}$ correspond to the same Boolean variable then they map to the same object. 
For every variable $x_i$ we create a shadow FODD variable $w_i$ and equate it to all the $v_{(i',j')}$ that correspond to $x_i$. We call this sequence of equalities a consistency block.
For example, consider the CNF  $$(x_1 \vee \overline{x_2}  \vee x_4) \wedge (\overline{x_1} \vee x_2 \vee x_3) \wedge (x_1 \vee x_3 \vee \overline{x_4})$$ 
where the corresponding FODD variables are
$$ 
v_{(1,1)},v_{(1,2)},v_{(1,3)},\ \ \ 
v_{(2,1)},v_{(2,2)},v_{(2,3)},\ \ \ 
v_{(3,1)},v_{(3,2)},v_{(3,3)}. 
$$

The first block, corresponding to $x_1$,  ensures that $w_1, v_{(1,1)}, v_{(2,1)}, v_{(3,1)}$ are all assigned the same value. In addition to testing that the values are equivalent the block tests that each variable gets bound to the same object as $y_1$ or $y_2$. The only possible way to not get a 0 in these blocks is to ensure that each variable in the block has the same value and that it is equal to either $y_1$ or $y_2$.
Figure~\ref{fig:fodd-sat-block-verify} shows the 
consistency blocks  
for our example.

\begin{figure}[t]
\centering
\includegraphics[scale = 0.40]{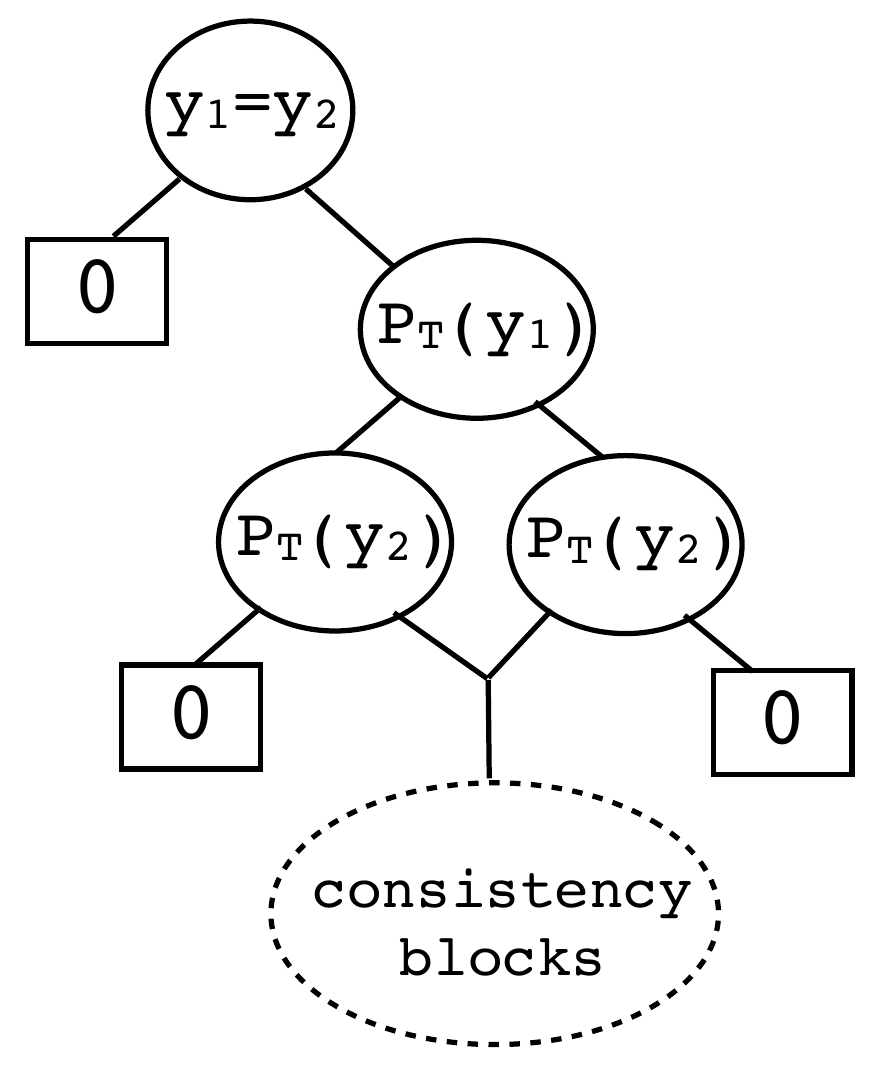}
\caption{The gadget verifying that $P_T()$ simulates Boolean values. 
}
\label{fig:fodd-sat-pt-verify}
\end{figure}

\begin{figure}[t]
\centering
\includegraphics[scale = 0.50]{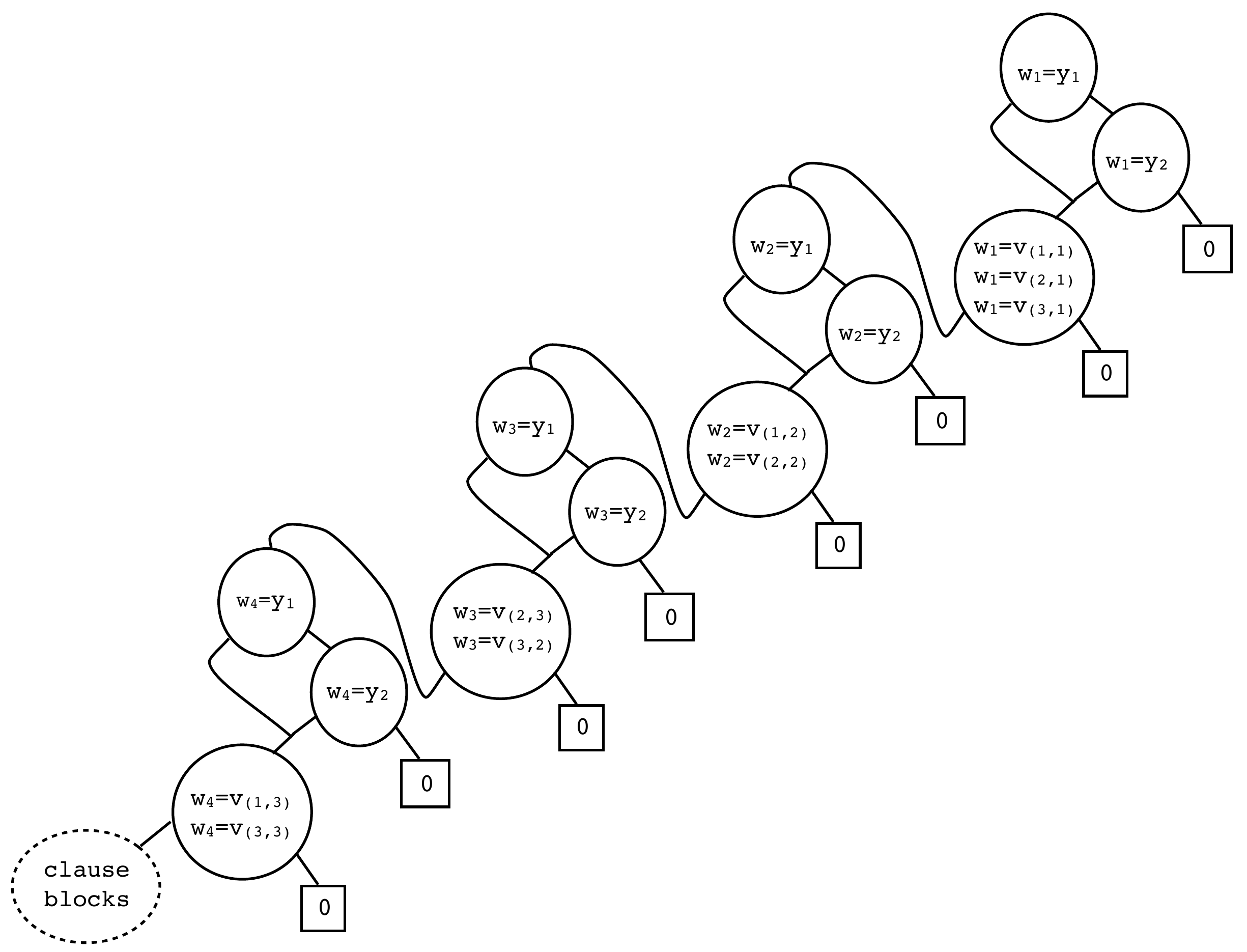}
\caption{The gadgets verifying Boolean values and variable consistency in the max FODD for the formula  $(x_1 \vee \overline{x_2}  \vee x_4) \wedge (\overline{x_1} \vee x_2 \vee x_3) \wedge (x_1 \vee x_3 \vee \overline{x_4})$.
Nodes that  have multiple equalities are a shorthand for a sequence of equality nodes where non-equality on any of the tests leads to the ``false exit", which is a 0 leaf in this diagram.
}

\label{fig:fodd-sat-block-verify}
\end{figure}

The third portion tracks the structure of $f$ to guarantee the same truth value in the FODD. 
To follow the structure of $f$, we build a block for each clause and chain these blocks together. Each block has 3 nodes corresponding to the 3 literals in the clause. In particular, if the $j$th literal in the $i$th clause is positive the true edge (literal satisfied; call this success) continues to the next clause, and the false edge (literal failed) continues to the next literal. For a negative  literal the true and false directions are swapped. The fail exit of the 3rd literal is attached to 0. Clause blocks have one entry and one exit and they are chained together. The success exit of the last clause is connected to the leaf 1. 
The only way to reach a value of 1 is if every clause block was satisfied by the valuations to $v_{(i,j)}$. 
Figure~\ref{fig:maxfodd-sat-clauses} illustrates the clause blocks for our example.

\begin{figure}[t]
\centering
\includegraphics[width = 0.27\textwidth]{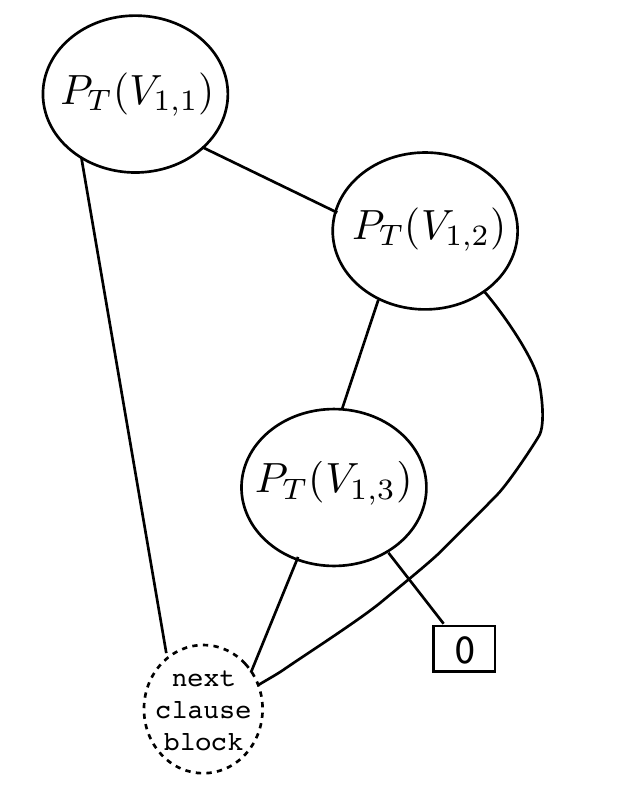}
\includegraphics[width = 0.27\textwidth]{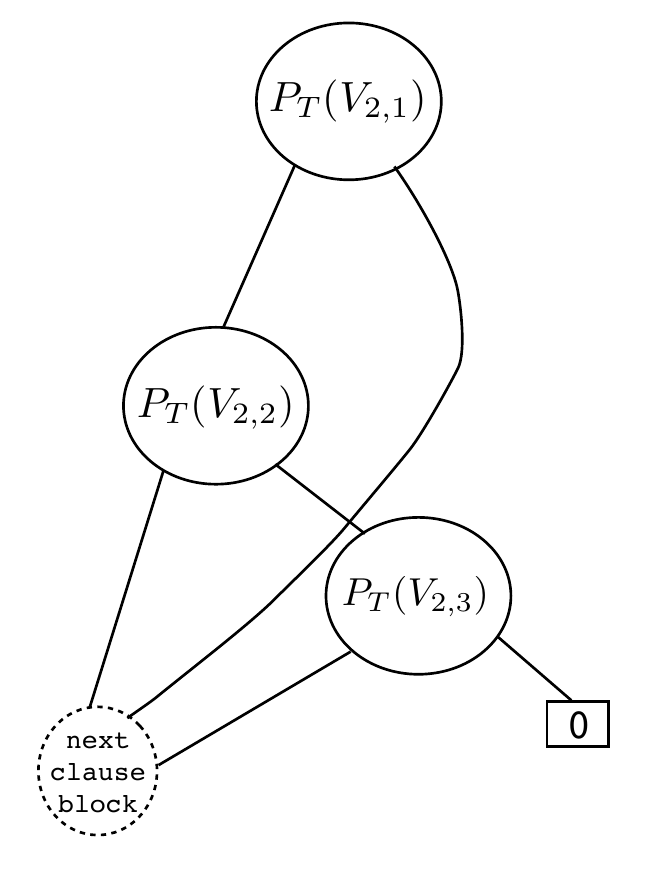}
\includegraphics[width = 0.38\textwidth]{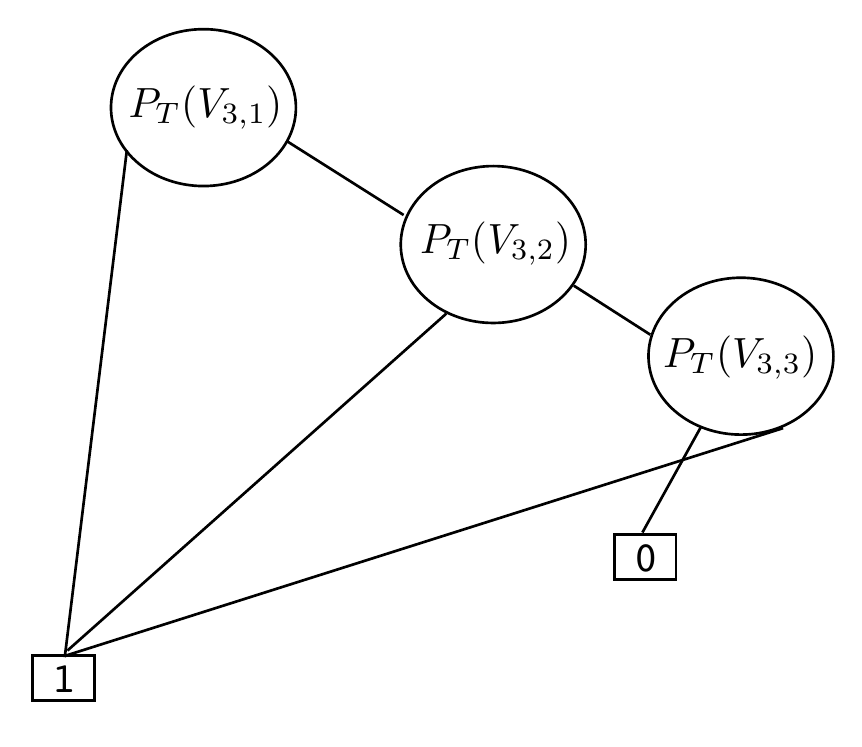}
\caption{The clause blocks for $(x_1 \vee \overline{x_2}  \vee x_4) \wedge (\overline{x_1} \vee x_2 \vee x_3) \wedge (x_1 \vee x_3 \vee \overline{x_4})$.}
\label{fig:maxfodd-sat-clauses}
\end{figure}

Each of the portions, including the clause blocks, has one entry and one exit and we chain them together 
to get the diagram $B$.
For a valuation to be mapped to 1 it must succeed in all three portions. 
We claim that $f$ is satisfied if and only if there is some interpretation $I$ such that $MAP_B(I) = 1$.

Consider first the case where $f$ is satisfiable. We introduce the interpretation $I$
that has two objects, $a$ and $b$, where $P_T(a)=$ true, and $P_T(b)=$ false. 
Let $v$ be a satisfying assignment for $f$ and let $\zeta(v)$ be a valuation for $B$ on $I$ where $y_1=a, y_2=b$ and if $v$ maps $x_i$ to 1 then $w_i$ and its block are mapped to $a$ and otherwise the block is mapped to $b$. Here, $\zeta(v)$ succeeds in all blocks, implying that $MAP_B(I,\zeta(v)) = 1$ and therefore $MAP_B(I) = 1$.

Consider next the case where $MAP_B(I) = 1$ for some $I$ and let $\zeta$ be such that $MAP_B(I,\zeta) = 1$. Then we claim that $\zeta$ identifies a satisfying assignment. First, since $\zeta$ succeeds in the first block we identify two objects that correspond via $P_T()$ to truth values, without loss of generality assume that $P_T(y_1)$ is true. Then success in the second portion implies that we can identify an assignment to the Boolean variables, if the $i$th block is assigned to $y_1$ we let $x_i=1$ and otherwise $x_i=0$. Finally, success in the third portion implies that the clauses in $f$ are satisfied by the assignment to the $x_i$'s. This completes the correctness proof.

Finally we address node ordering in the diagram. The only violation of ordering is the use of $P_T()$ in the first block. Otherwise, we have all equalities above $P_T()$, variable ordering $y_i\prec w_j \prec V_{a,b}$, and lexicographic ordering within a group. Now because our diagram forms one chain of blocks  leading to a single sink leaf with value 1 we can move the three $P_T()$ nodes to the bottom of the diagram in Figure~\ref{fig:fodd-sat-block-verify}. This does not change the map value for any valuation and thus does not affect correctness. We therefore conclude that $B$ is consistently sorted and $f$ is satisfiable iff $\map_B(I)=1$ for some $I$.
\end{proof}

\bigskip
This proof illustrates the differences in arguments needed for FODDs and GFODDs vs.\ First Order Logic.
For the latter, the reduction can use the sentence
$\exists v, (p_{x_1}(v) \vee \overline{p_{x_2}(v)} \vee p_{x_4}(v)) \wedge (\overline{p_{x_1}(v)} \ldots)$ to show the hardness result. However, this cannot be easily represented as a FODD because the literals appearing in the clauses will violate predicate order and, if we try to reorder the nodes from a naive FODD encoding, the result might be exponentially larger. 
An alternative formulation can use
$\exists x_1\ldots (p(x_1) \vee \overline{p(x_2)} \vee p(x_4)) \wedge (\overline{p(x_1)} \ldots$
to avoid the problem with predicate order. However, similar ordering issues now arise for the arguments. Our reduction introduces additional variables as well as the variable consistency gadget to get around these issues. 
The same structure of reduction from 3SAT instances and their QBF generalizations will be used in the results for GFODD.

\begin{theorem}
\label{thm:maxfodd-equiv}
{ FODD Equivalence}  and { FODD Edge Removal}  are $\Pi_2^p$-complete.
\end{theorem}
\begin{proof}
Since Edge Removal is a special case of Equivalence it suffices to show membership for Equivalence and hardness for Edge Removal. The hardness result is given in two stages;
we first present a reduction which does not respect the constraint on node ordering 
and Edge Removal structure, and 
then show how to fix the construction to respect these restrictions.

\paragraph{Membership in $\Pi_2^p$:}
First observe that, by Lemma~\ref{lemma:foddsize}, if the diagrams are not equivalent then there is a small interpretation that serves as a witness for the difference.
Using this fact, we  can
show that non-equivalence is in $\Sigma_2^p$. Given $B_1$, $B_2$ we guess an interpretation $I$ of the appropriate size, and then appeal to an oracle for FODD Evaluation to calculate $\map_{B_1}(I)$ and $\map_{B_2}(I)$.   Using these values we return Yes or No accordingly.
To calculate the map values, let $B$ be one of these diagrams, and let the leaf values of the diagram be $v_1,v_2,\ldots,v_k$. We make $k$ calls to FODD Evaluation 
with $(B,I,v_i)$ as input. $\map_{B}(I)$ is the largest value on which the oracle returns Yes.
If a witness $I$ for non-equivalence exists then this process can discover it and say No, and otherwise it will always say Yes. Therefore non-equivalence is in $\Sigma_2^p$, and equivalence is in $\Pi_2^p$.

\paragraph{Reduction basics:}
To show hardness, 
consider the problem of deciding arrowing from the Ramsey theory of graphs \cite{Schaefer01}. 
Given two graphs $G_1$, $G_2$ we say that $G_1$ includes an {\em embedding} of $G_2$ if there is a 1-1 mapping $g$ from nodes of $G_2$ to nodes of $G_1$, such that for every edge $(v_1,v_2)$ of $G_2$, the edge $(g(v_1),g(v_2))$ is in $G_1$. 
We say that $G_1$ includes an {\em isomorphic embedding} of $G_2$ if, in addition, $g$ satisfies that 
for every edge $(v_1,v_2)$ not in $G_2$, the edge $(g(v_1),g(v_2))$ is not in $G_1$.

We say that  
$F$ arrows $(G,H)$, denoted
$F \rightarrow (G,H)$, if for every 2-color edge-coloring of $F$ into colors red and blue,
the red subgraph of $F$ includes an embedding of $G$ or the 
the blue subgraph of $F$ includes an embedding of $H$.

The arrowing problem is as follows: Given $3$ graphs $F, G, H$ as input, return Yes iff $F$ arrows $(G,H)$. This problem was shown to be $\Pi_2^p$-complete by \cite{Schaefer01}.
We reduce this problem to FODD equivalence. 
The signature includes equality and two arity-2 predicates $E_F$ and $E_C$, where $E_F$ captures the edge relation of the main graph $F$ and $E_C$ is a coloring of all possible edges such that when $E_C(x_i,x_j)$ is true the edge is colored red and when it is false the edge is colored blue.

\begin{figure}[t]
\centering
\includegraphics[scale = 0.50]{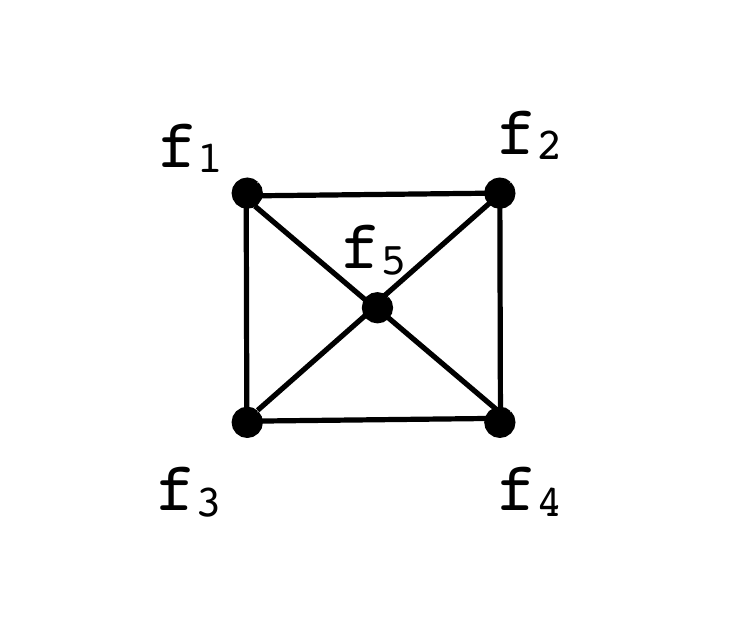}
\includegraphics[scale = 0.50]{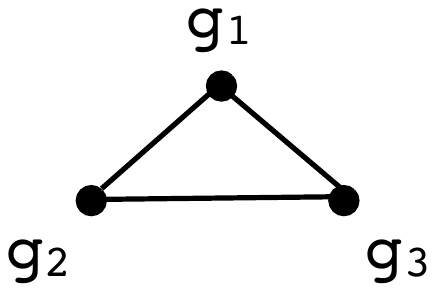}
\includegraphics[scale = 0.50]{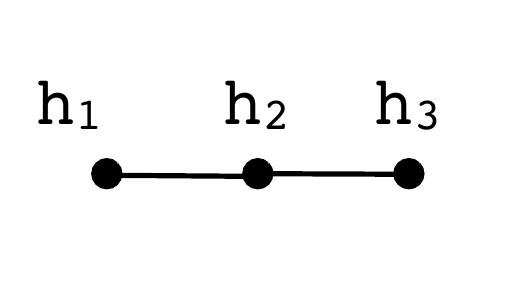}
\caption{Three input graphs to the Arrowing problem F, G, H. This is a positive instance as (one can verify that) every two coloring of F has either a red triangle (G) or two blue edges connected by a single vertex (H).}
\label{fig:arrowing-graphs}
\end{figure}

\begin{figure}[t]
\centering
\includegraphics[scale = 0.55]{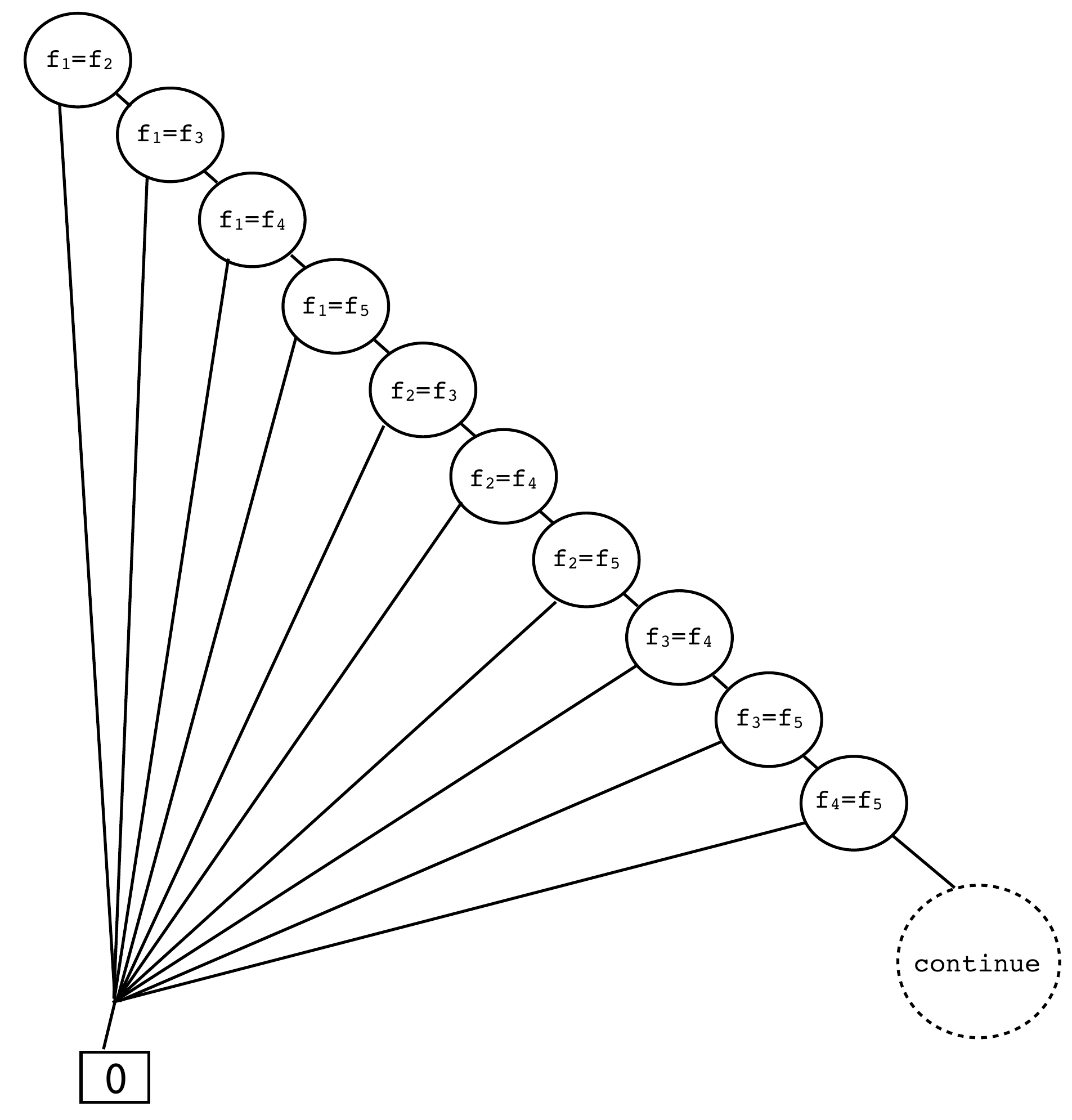}
\caption{A FODD verifying that variables $f_i$ are mapped to distinct objects in $I$.}
\label{fig:fodd-f-unique-nodes}
\end{figure}

\begin{figure}[t]
\centering
\includegraphics[scale = 0.30]{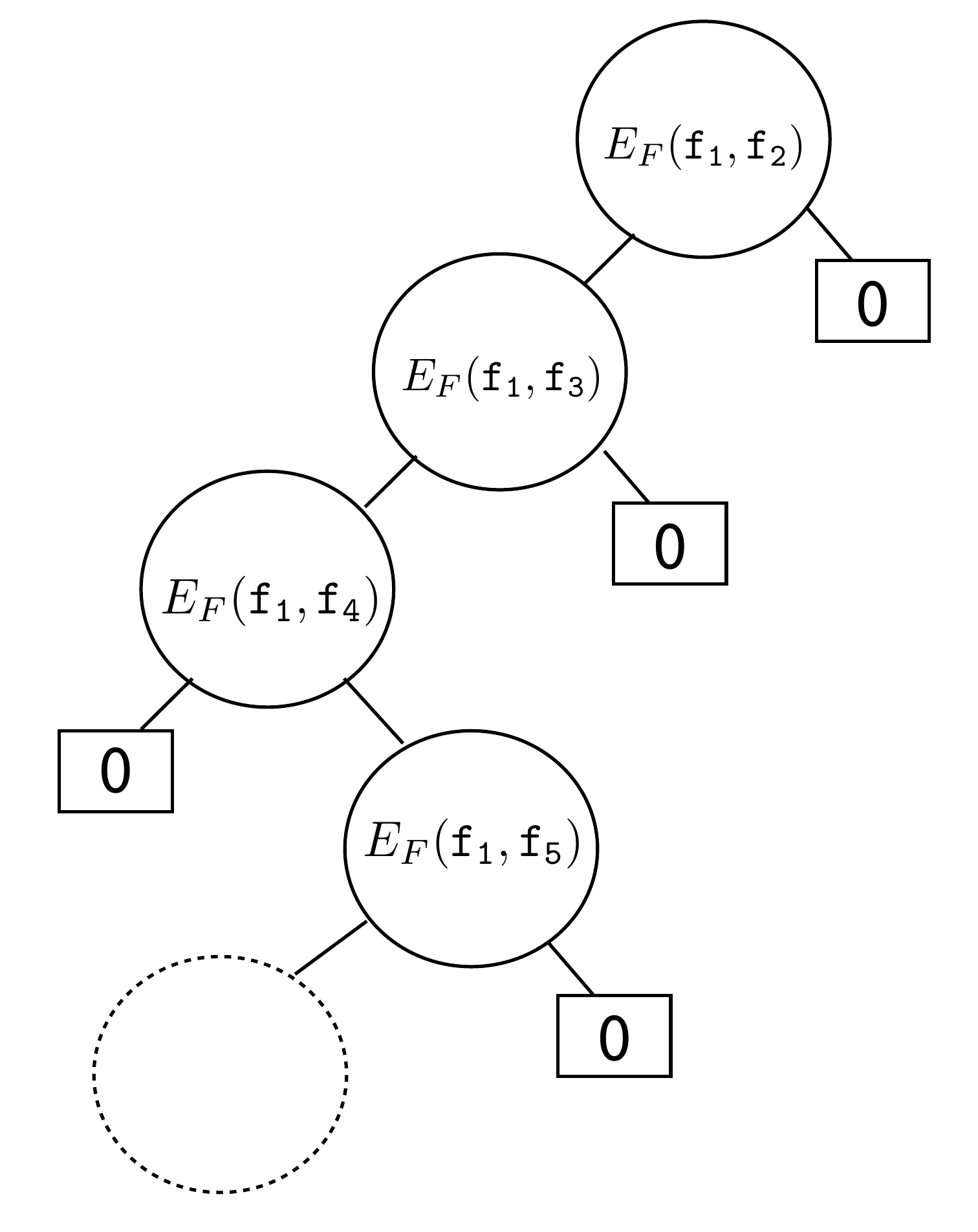}
\caption{A fragment of a FODD verifying the graph F. Here every neighbor of $f_1$ is tested. Since $f_1$ is connected to $f_2$, $f_3$ we expect the corresponding atoms to be true and thus continue on the left branch; since it is not connected to vertex $f_4$, the gadget continues on the false side; $f_5$ is connected and we continue to the left again. 
}
\label{fig:arrowing-vertex-verify}
\end{figure}

\begin{figure}[t]
\centering
\includegraphics[scale = 0.50]{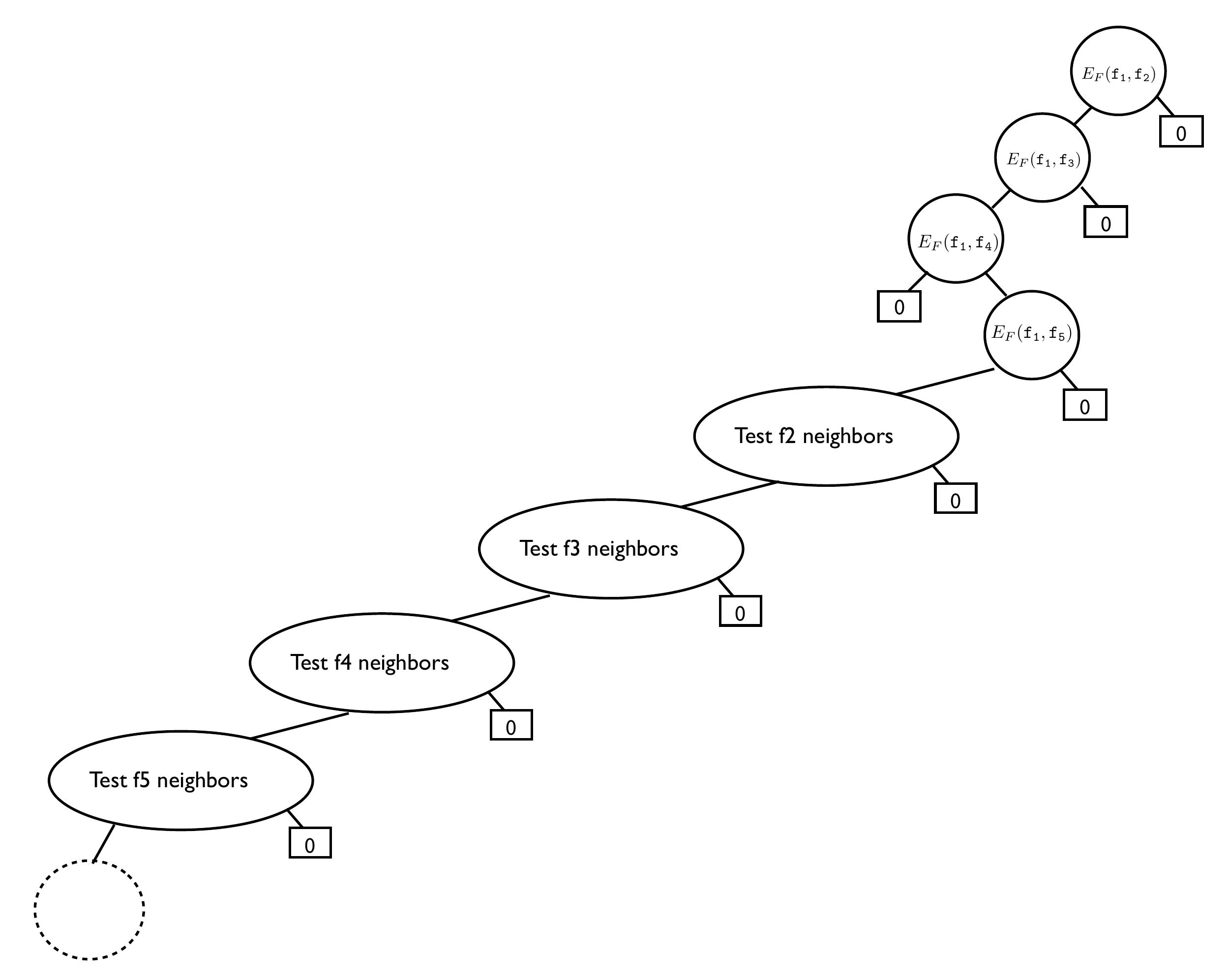}
\caption{The FODD verifying the structure of the graph F. }
\label{fig:arrowing-F-verify}
\end{figure}

\paragraph{The main construction:}
To transform arrowing into an instance of FODD equivalence we build two FODDs with binary leaves. The first FODD is satisfied iff $I$ includes an isomorphic embedding of $F$ in its edge relation $E_F$.  The second FODD is satisfied iff the same condition holds {\em and} the coloring defined by $E_C$ has a red embedding of $G$ or a blue embedding of $H$. 
Note that, due to the 1-1 requirement, $I$ must have at least as many objects as there are nodes in $F$.
We illustrate the construction using 
the example input in Figure~\ref{fig:arrowing-graphs}. Here the input graphs F, G, H are a positive instance of arrowing.

To build a FODD which verifies that $I$ has an isomorphic embedding of $F$, we map each node to a variable in the FODD and test that each node has its correct neighbors. 
We first build a ``node mapping" gadget that makes sure that each variable in the FODD is mapped to a different object in the interpretation.
This is done by following a path of ${V \choose 2}$ inequalities, where off-path edges go to 0 and the final exit continues to the next portion.
This gadget, for our example graph $F$ with 5 nodes,  is shown in Figure~\ref{fig:fodd-f-unique-nodes}.
To test isomorphism to $F$ we test the neighbors of each node in sequence to verify that edges exist iff they are in $F$.
The FODD fragment in Figure~\ref{fig:arrowing-vertex-verify} shows how this can be tested for vertex $f_1$ in the example. If the edge is present in the graph we continue left  (using the {\tt true} branch) to the next neighbor and if the edge is not in $F$ we continue to the right child (the {\tt false} branch). Edges off this path are directed to the zero leaf. The endpoint of the path will connect to the next portion of the FODD. 
This construction can be done for each node and the fragments can be connected together to yield the 
$F$ verifier. 
This is illustrated in Figure~\ref{fig:arrowing-F-verify}. 
Finally, the diagram $B_1$ is built by connecting the $F$ verifier at the bottom of node mapping gadget, and replacing the bottom node of the $F$ verifier with a leaf valued 1. We refer to this diagram as the ``complete $F$ verifier" below.
This construction can be done in polynomial time for any graph $F$.
It should be clear form the construction that $\map_{B_1}(I)=1$ iff $I$ includes an isomorphic embedding of $F$ in its edge relation $E_F$. In addition, the verifier diagram is ordered where we have $``="\prec E_F$, and where variables are ordered lexicographically. 

\clearpage

The second digram $B_2$ includes the complete $F$ verifier and additional FODD fragments that are described next to capture the conditions on $G$ and $H$ respectively. In order to verify the embedding of colored subgraph $G$ we first define a node mapping capturing the mapping of $G$ nodes into $F$ nodes, and then verify that the required edges exist and that they have the correct color. 
The FODD fragment in Figure~\ref{fig:arrowing-G-iso1} shows how we can select a node mapping for vertex $g_1$. This fragment returns 0 unless $g_1$ is mapped to one of the nodes in $F$ that are identified in the $B_1$ portion.
As depicted in Figure~\ref{fig:arrowing-G-full-iso}, this can be repeated for all the nodes in $G$, verifying that each node in $G$ is mapped to a node in $F$. 
Next we need to verify that the mapping is one to one. This can be done by using a path of inequalities between the variables referring to nodes of $G$. This FODD fragment is given in Figure~\ref{fig:arrowing-g-map}. 
For correctness, we need to chain the two tests together, but this will violate node ordering. We therefore interleave the tests
putting the uniqueness equality tests for a variable exactly after the equalities selecting its value. This change is possible because each such block has exactly one exit point. The resulting diagram, for our running example, is shown in Figure~\ref{fig:arrowing-G-equalities}.

\begin{figure}[t]
\centering
\includegraphics[scale = 0.40]{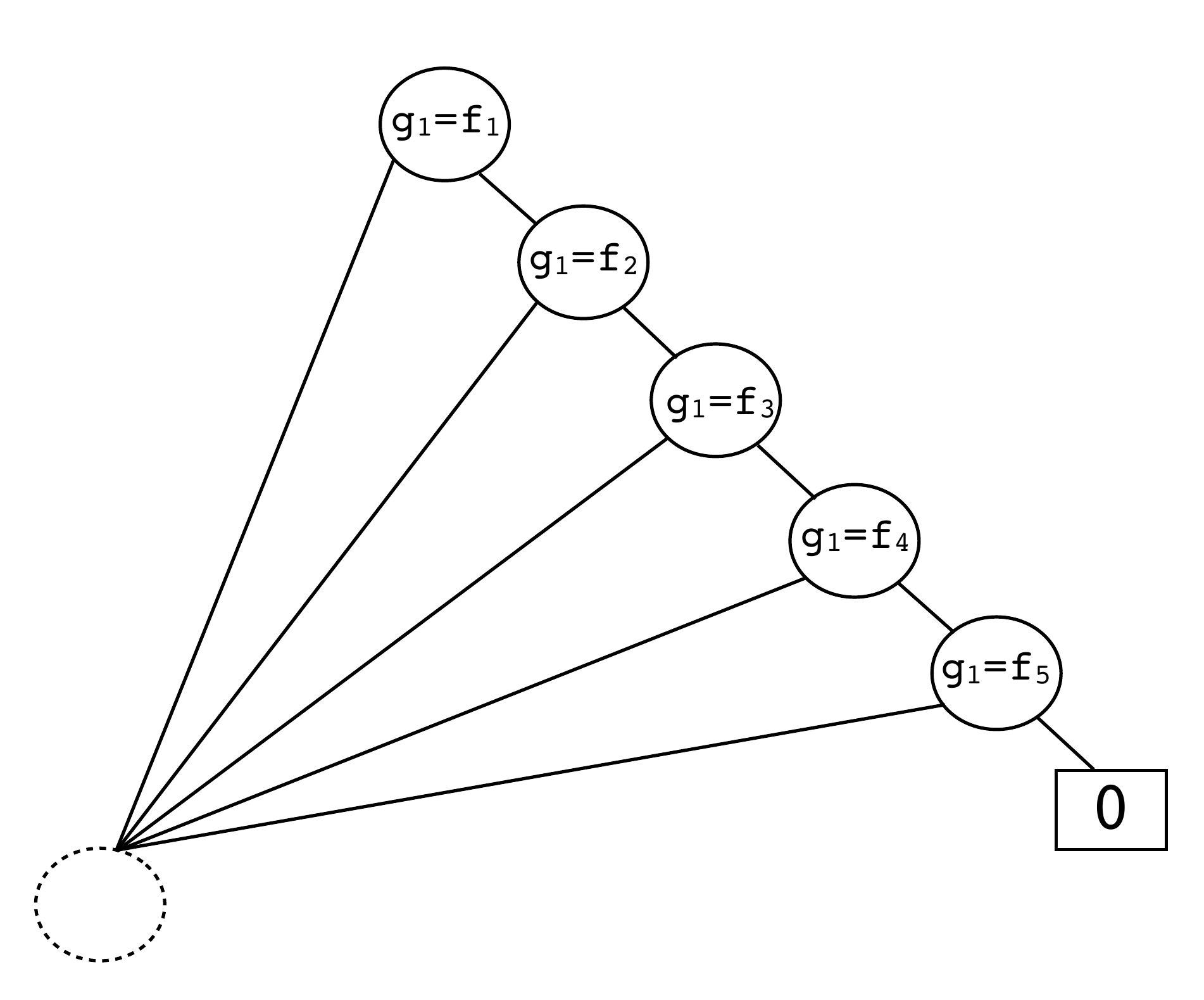}
\caption{A FODD verifying the mapping of vertex $g_1$ to some vertex in F.}
\label{fig:arrowing-G-iso1}
\end{figure}

\begin{figure}[t]
\centering
\includegraphics[scale = 0.50]{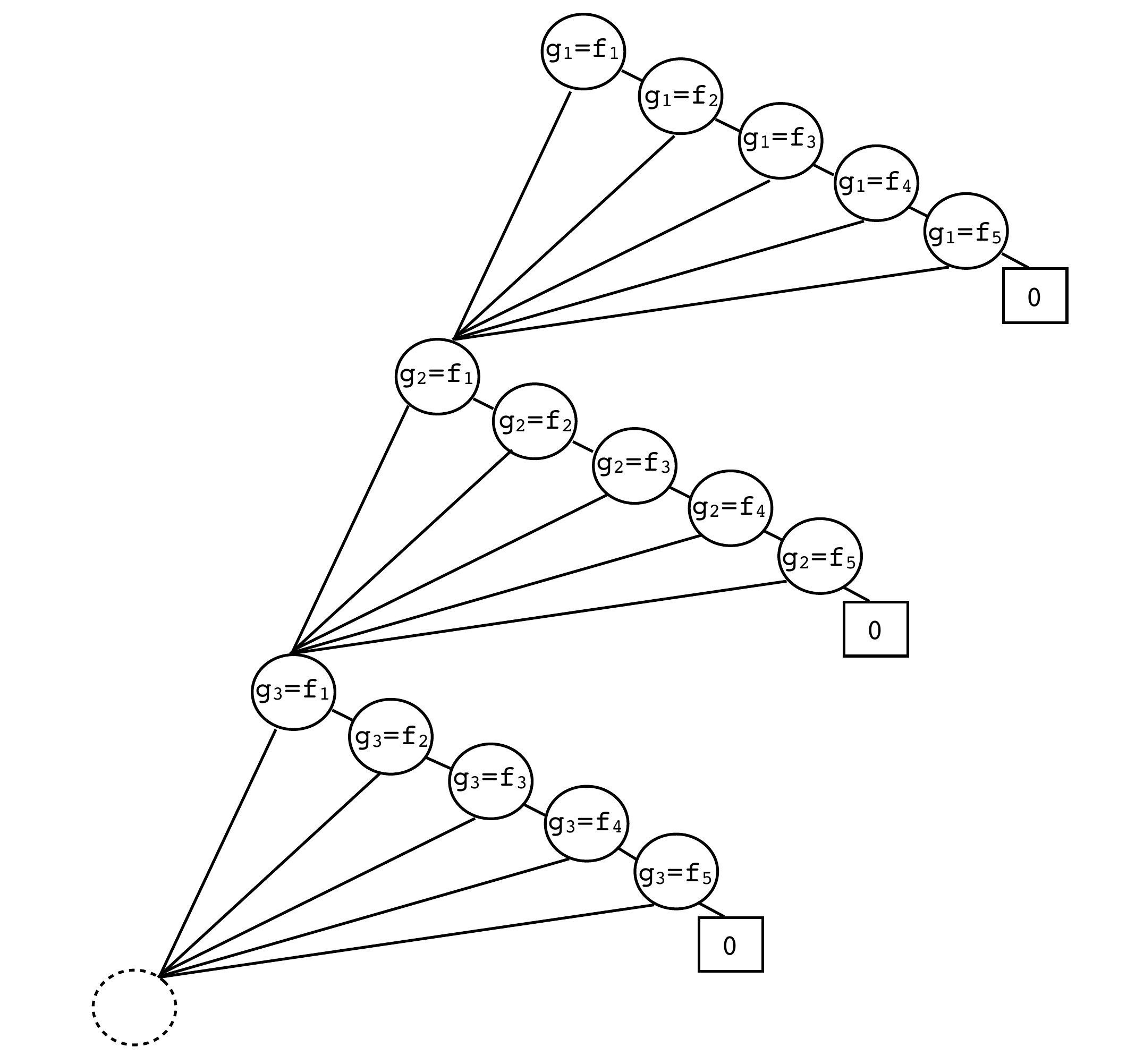}
\caption{The complete FODD verifying the vertices of G are mapped to vertices in F.}
\label{fig:arrowing-G-full-iso}
\end{figure}

\begin{figure}[t]
\centering
\includegraphics[scale = 0.40]{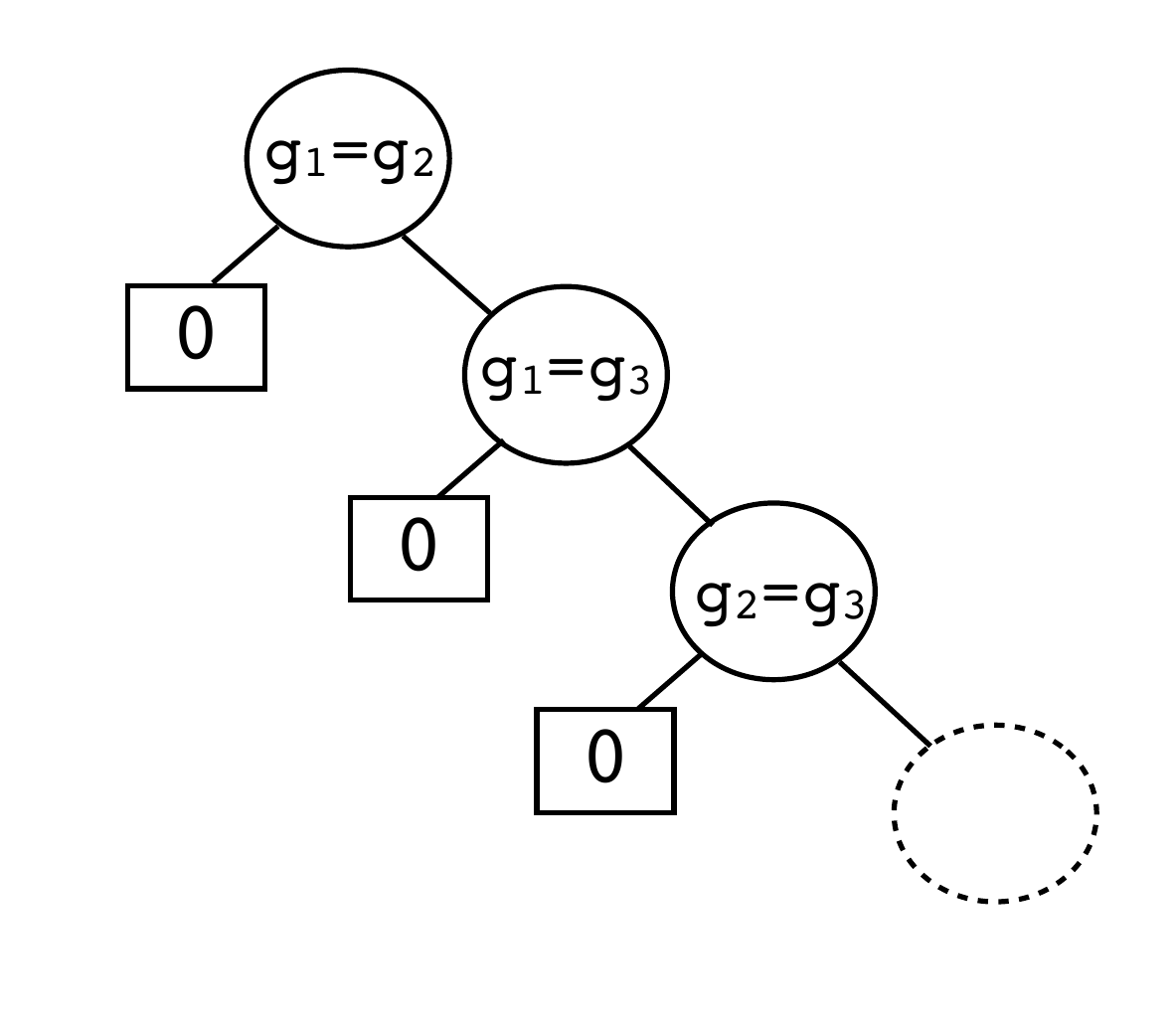}
\caption{A FODD verifying the mapping is one to one.}
\label{fig:arrowing-g-map}
\end{figure}

\begin{figure}[t]
\centering
\includegraphics[scale = 0.45]{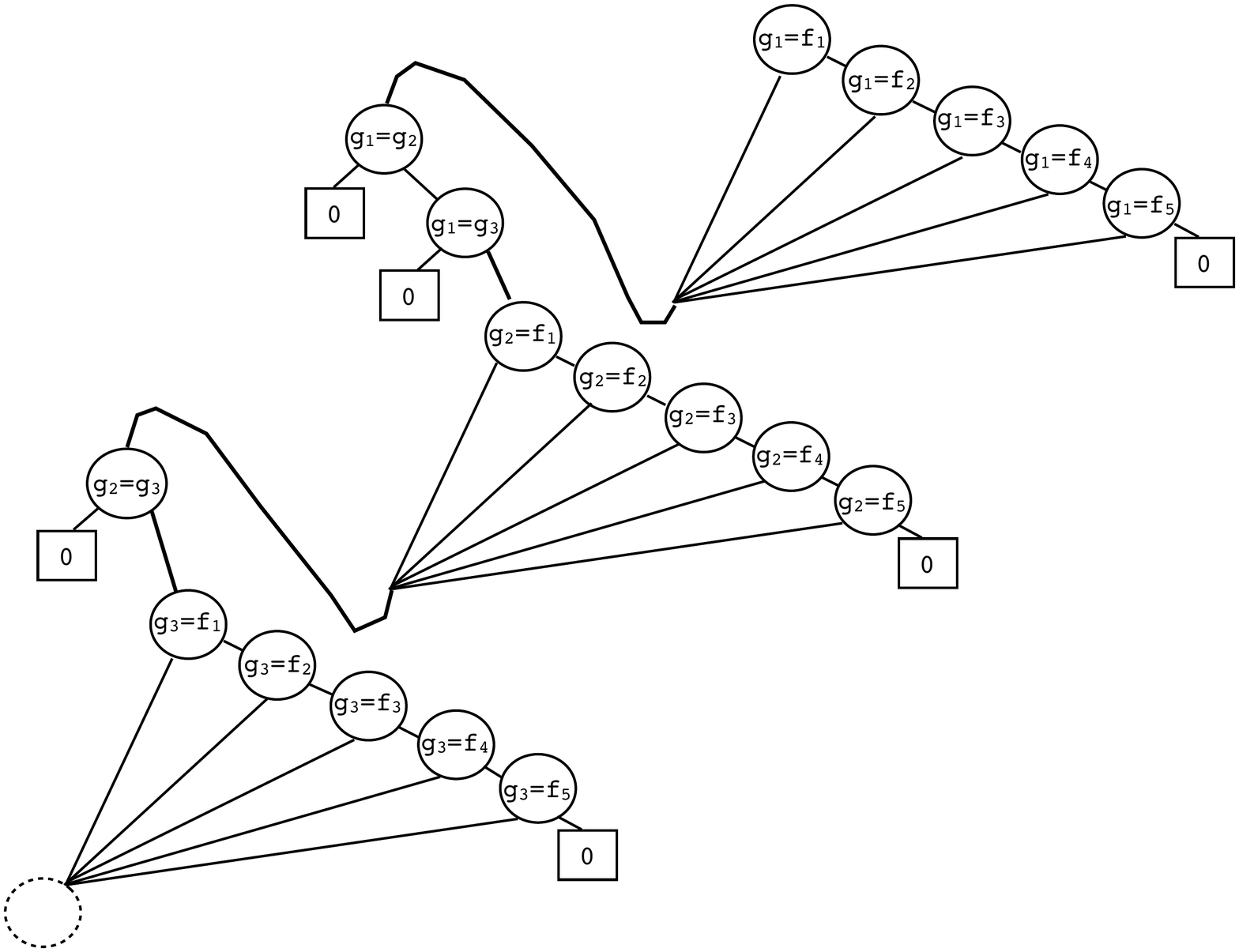}
\caption{Node mapping construction for $G$ reordered to comply with sorted order. }
\label{fig:arrowing-G-equalities}
\end{figure}

To complete the embedding test, we need to check that the edges are preserved and that they have the correct color. 
We do this by first checking that the corresponding edges in $G$ are in $F$. We can do this using a left going path testing each edge in turn, where we test both $E_F(g_i,g_j)$ and $E_F(g_j,g_i)$ to account for the fact that the graph is undirected.\footnote{The test of both directions of the edge is not necessary, because  a different portion of the diagram already verifies that the embedding of $F$ is undirected, but we include it here to simplify the argument.}
This is illustrated on the left hand side of Figure~\ref{fig:arrowing-g-edge}. Note that, because we are testing for an embedding (i.e., not for an isomorphic embedding) we test only for the edges in $G$ and do not need to verify nonexistence of the edges not in $G$ (it just happens here that $G$ is a clique so this is not visible in the example). The same FODD structure 
is repeated with predicate $E_C$ replacing $E_F$ to verify that the edges of G are colored red, as shown on the right of Figure~\ref{fig:arrowing-g-edge}.

\clearpage

A similar construction with node mapping, edge verifier, and color verifier can be used for $H$. 
The node mapping construction is identical.
Figure~\ref{fig:arrowing-h3} shows the edge and color verifiers.
The only difference in construction is that the color verifier tests that the edge is not in $E_C$ to capture the color blue and therefore has a mirror structure to the one verifying the $E_F$ edges.
Note that in this case $H$ is not a complete graph and we are indeed only testing for the edges in $H$.
This construction can be done in polynomial time for any $G$ and $H$.

Finally we connect the three portions together to obtain $B_2$ as follows. The final output of the complete $F$ verifier is connected to the root of the $G$ verifier. The final output of the $G$ verifier is connected to 1. The zero leaf of the $G$ verifier is removed and instead connected to the root of the $H$ verifier. The final output of the $H$ verifier is connected to 1. Therefore, there are exactly two edges leading to the 1 leaf in this diagram, corresponding to the positive outputs of the $G$ and $H$ verifiers. Figure~\ref{fig:arrowing-both-fodd} shows an overview of the two FODDs, $B_1$ and $B_2$, generated by the reduction.

The diagrams $B_1$ and $B_2$ are not consistent with any sorting order over node labels, and thus we need to modify them to get a consistent ordering.
We show below how this can be done with only a linear growth in  the size of the diagrams and without changing the semantics of  $B_1$ and $B_2$.
Before presenting this transformation
we show that $F \rightarrow (G,H)$ iff $B_1$ and $B_2$ are equivalent. 

\begin{figure}[t]
\centering
\includegraphics[scale = 0.50]{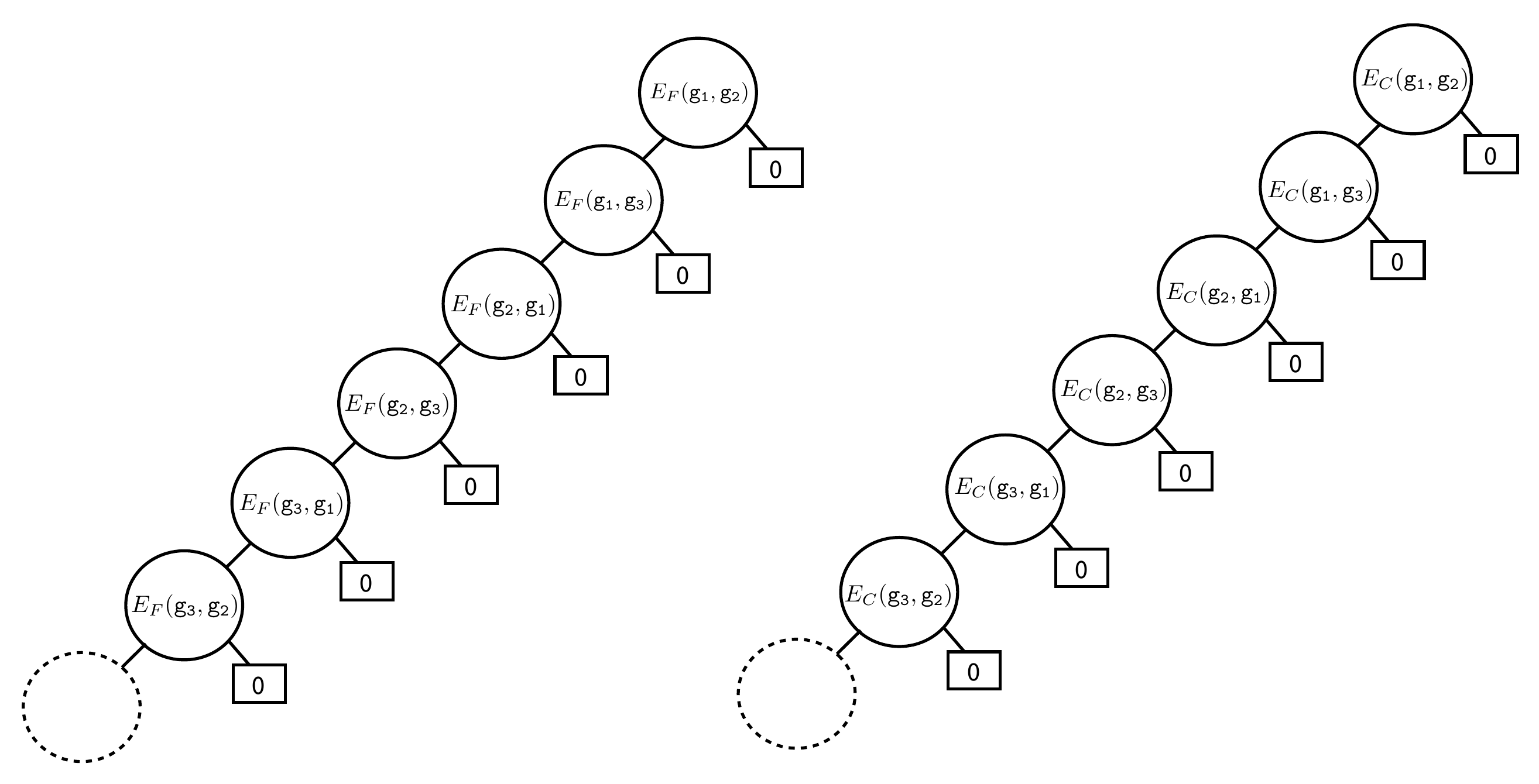}

\caption{A FODD verifying that the edges in G are in F.}
\label{fig:arrowing-g-edge}
\end{figure}

\begin{figure}[t]
\centering
\includegraphics[scale = 0.40]{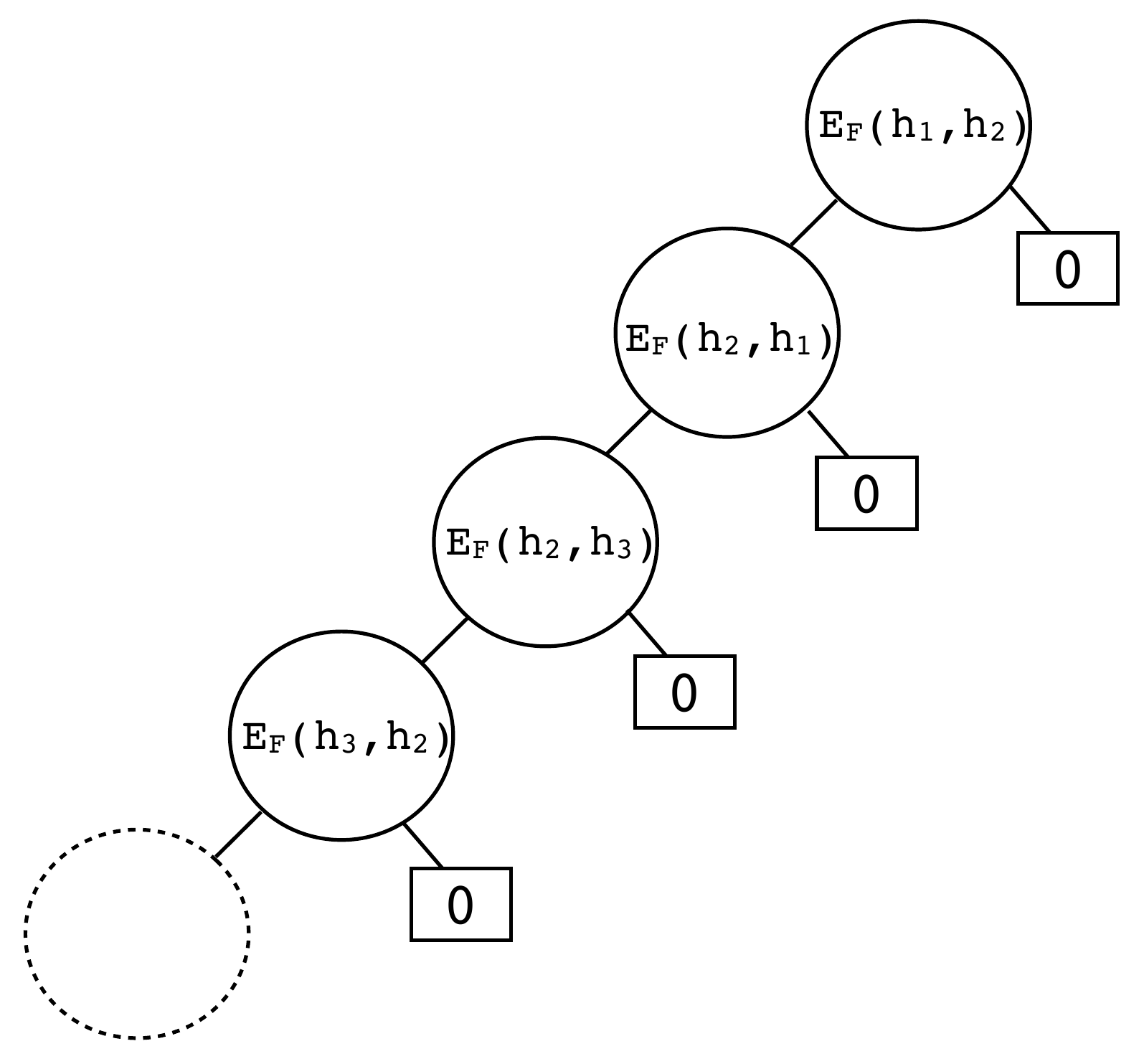}
\includegraphics[scale = 0.40]{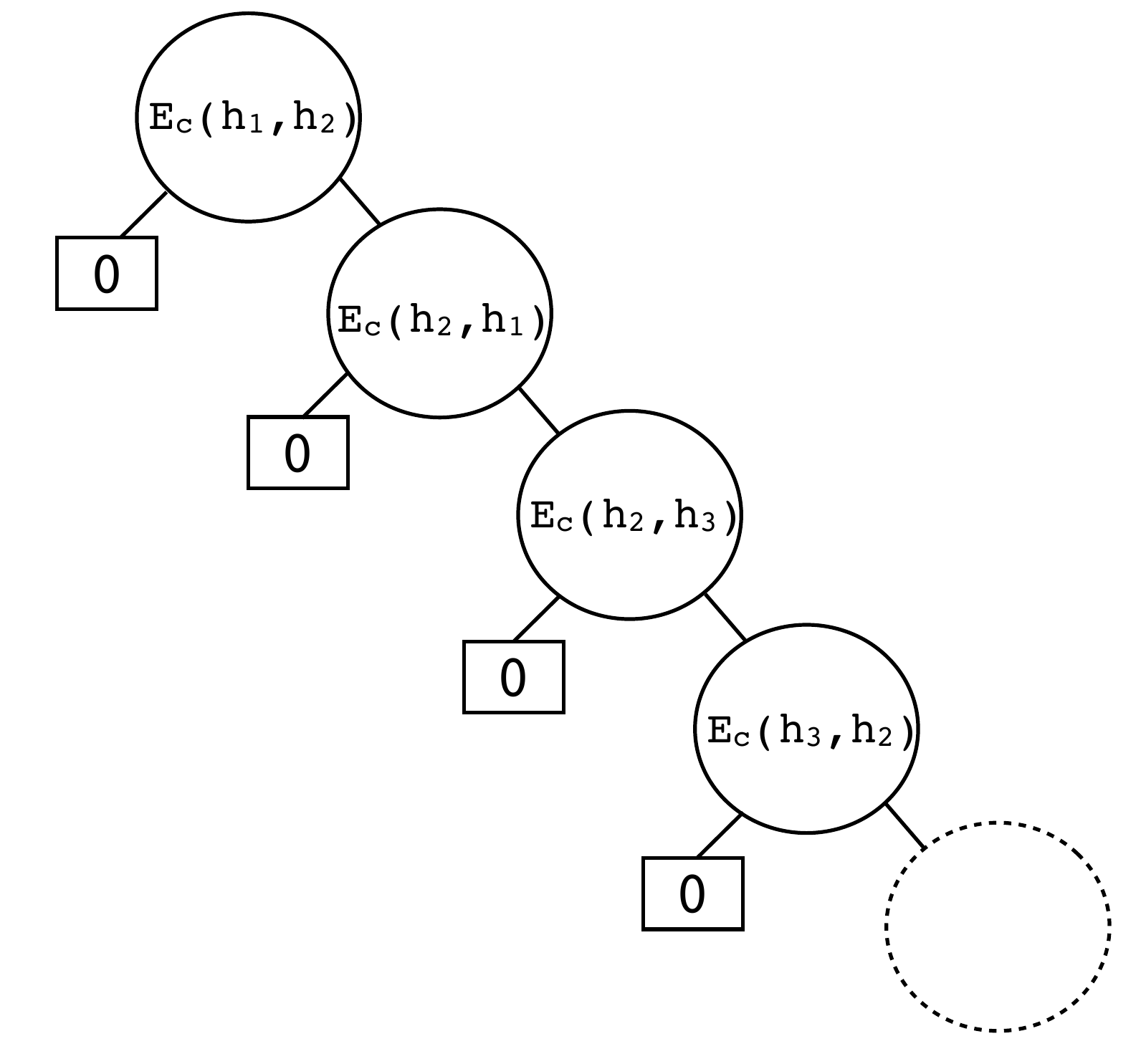}

\caption{Left: a FODD verifying that the edges in H are in F. Right: a FODD verifying that the color of $H$ is blue (i.e., its edges are not in $E_C$).}
\label{fig:arrowing-h3}
\end{figure}

\begin{figure}[t]
\centering
\includegraphics[scale = 0.45]{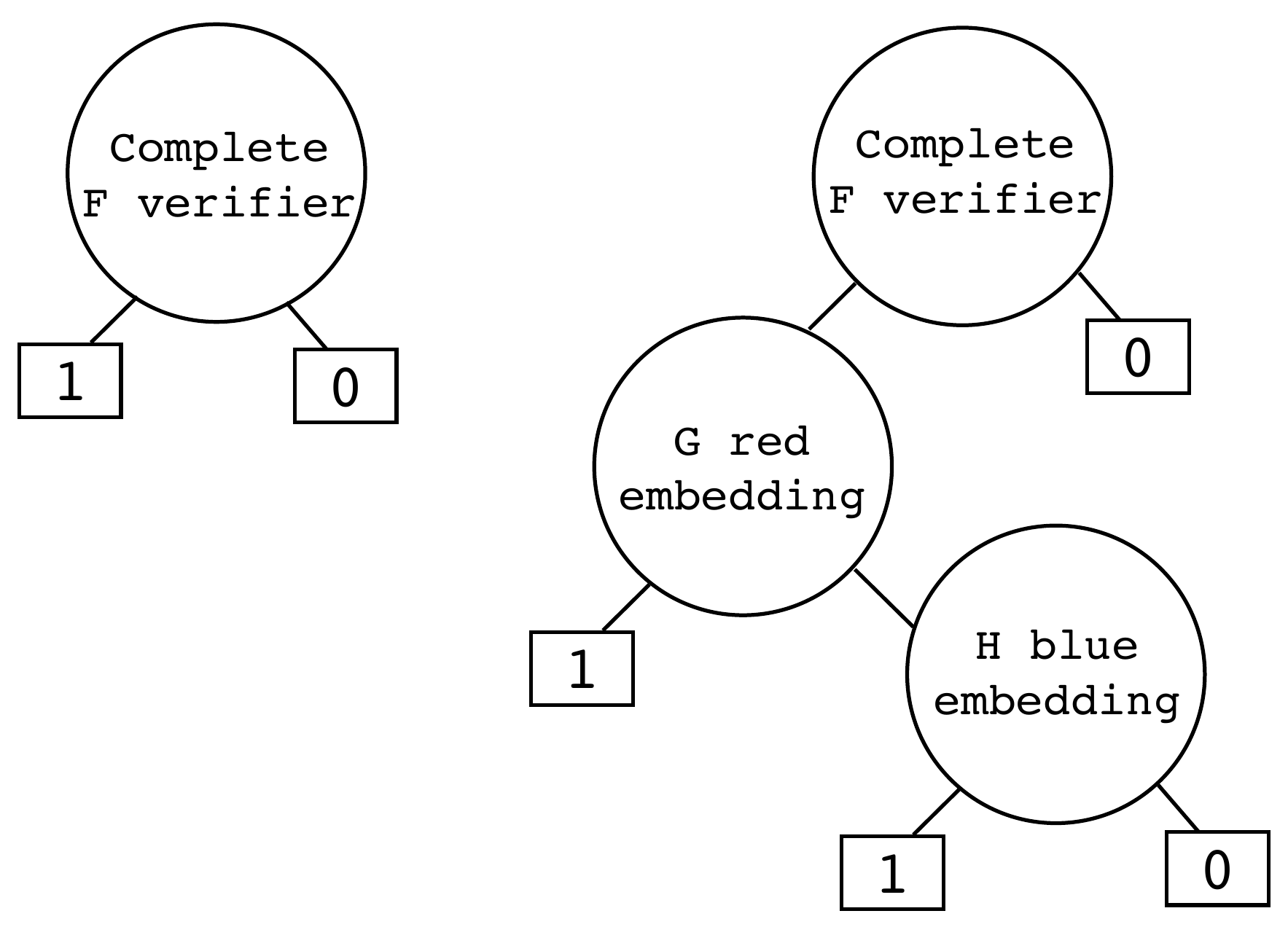}

\caption{High level structure of the two FODDs created from the graphs F, G, H.}
\label{fig:arrowing-both-fodd}
\end{figure}

\paragraph{Correctness of the construction:}
Consider the case when $F \rightarrow (G,H)$, that is, for every 2-color edge-coloring of $F$ there is a red $G$ or a blue $H$. We show that the two FODDs are equivalent by way of contradiction. Assume that $B_1$ and $B_2$ are not equivalent and let $I$ be any
witness to this fact.
Now,
$\map_{B_1}(I)=0$ implies  $\map_{B_2}(I)=0$ because the only paths to 1 in $B_2$ go through a copy of $B_1$. 
Therefore, for the assumed witness $I$, it must be the case that
$\map_{B_1}(I)=1$ and $\map_{B_2}(I)=0$. 

By construction, $\map_{B_1}(I)=1$ implies that $I$ has an isomorphic embedding of $F$. Because $F \rightarrow (G,H)$,
any coloring of that embedding, including the coloring captured by $E_C$ in $I$, has a red $G$ or a blue $H$. Assume that the embedding in $I$ has a red $G$. Then we can construct the appropriate node mapping in a valuation $\zeta$ to show that $\map_{B_2}(I,\zeta)=1$, contradicting the assumption. 
The same argument handles the case when the embedding has a blue $H$.

Consider the case when $F$ does not arrow $(G,H)$. Then there is a valid 2-color edge-coloring of $F$ which does not have a red $G$ and does not have a blue $H$. Construct the corresponding interpretation $I$ that represents $F$ and this edge-coloring. We claim that $\map_{B_1}(I)=1$ and $\map_{B_2}(I)=0$.  The fact $\map_{B_1}(I)=1$ follows by mapping the nodes in $F$ to the variables that represent them. 
Now if $\map_{B_2}(I)=1$ then $\map_{B_2}(I,\zeta)=1$ for some $\zeta$ and we can trace the path that $\zeta$ traverses in $B_2$. This path together with $\zeta$ can be used to identify either a red G or a blue H in $I$ and therefore in the corresponding coloring of $F$. This contradicts the assumption that the coloring is a witness for non-arrowing.

\begin{figure}[thb]
\centering
\includegraphics[scale = 0.80]{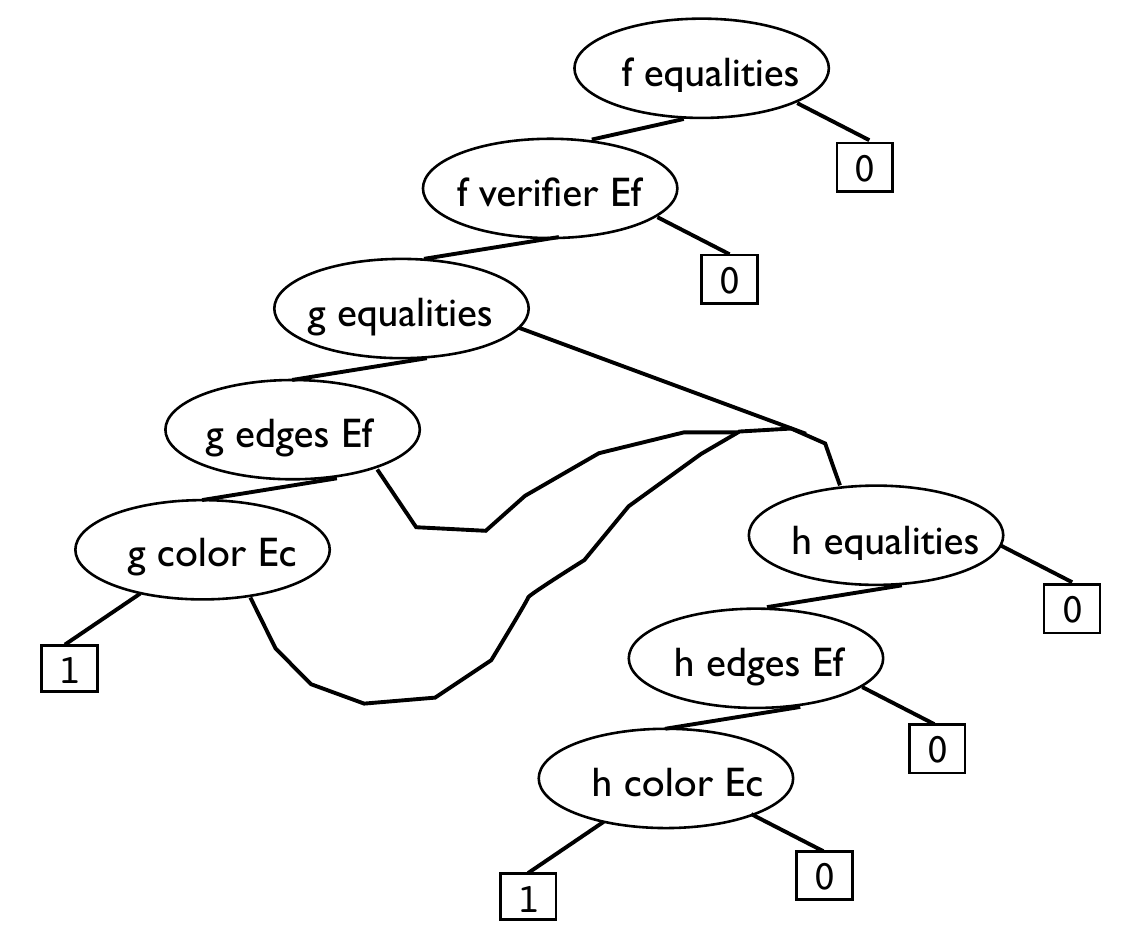}
\caption{Expanded view of the structure of $B_2$. }
\label{fig:arrowing-B2expanded}
\end{figure}

\begin{figure}[thb]
\centering
\includegraphics[scale = 0.40]{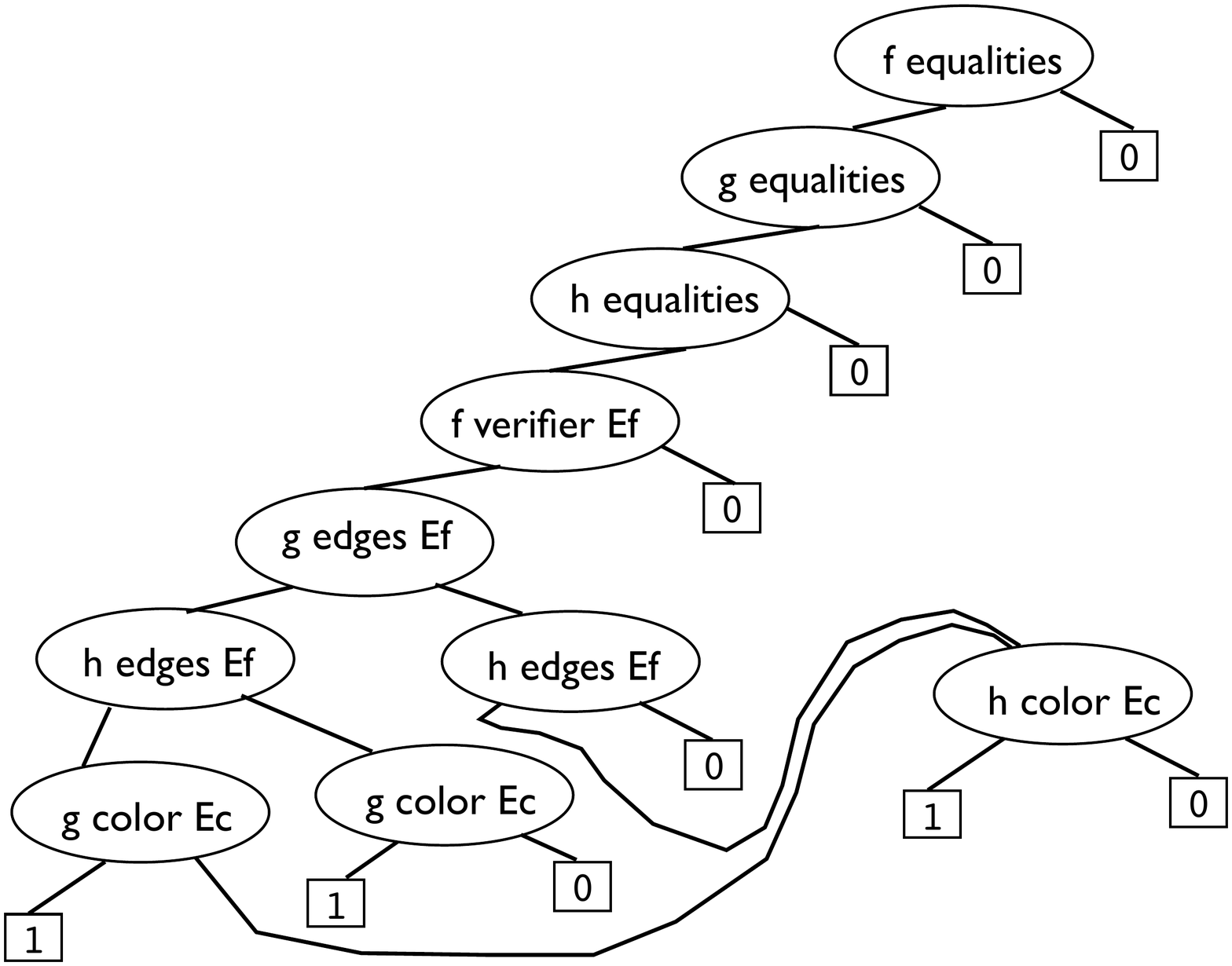}
\caption{Expanded view of the reordered structure of $B_2$. }
\label{fig:arrowing-final-B2}
\end{figure}

\begin{figure}[thb]
\centering
\includegraphics[scale = 0.40]{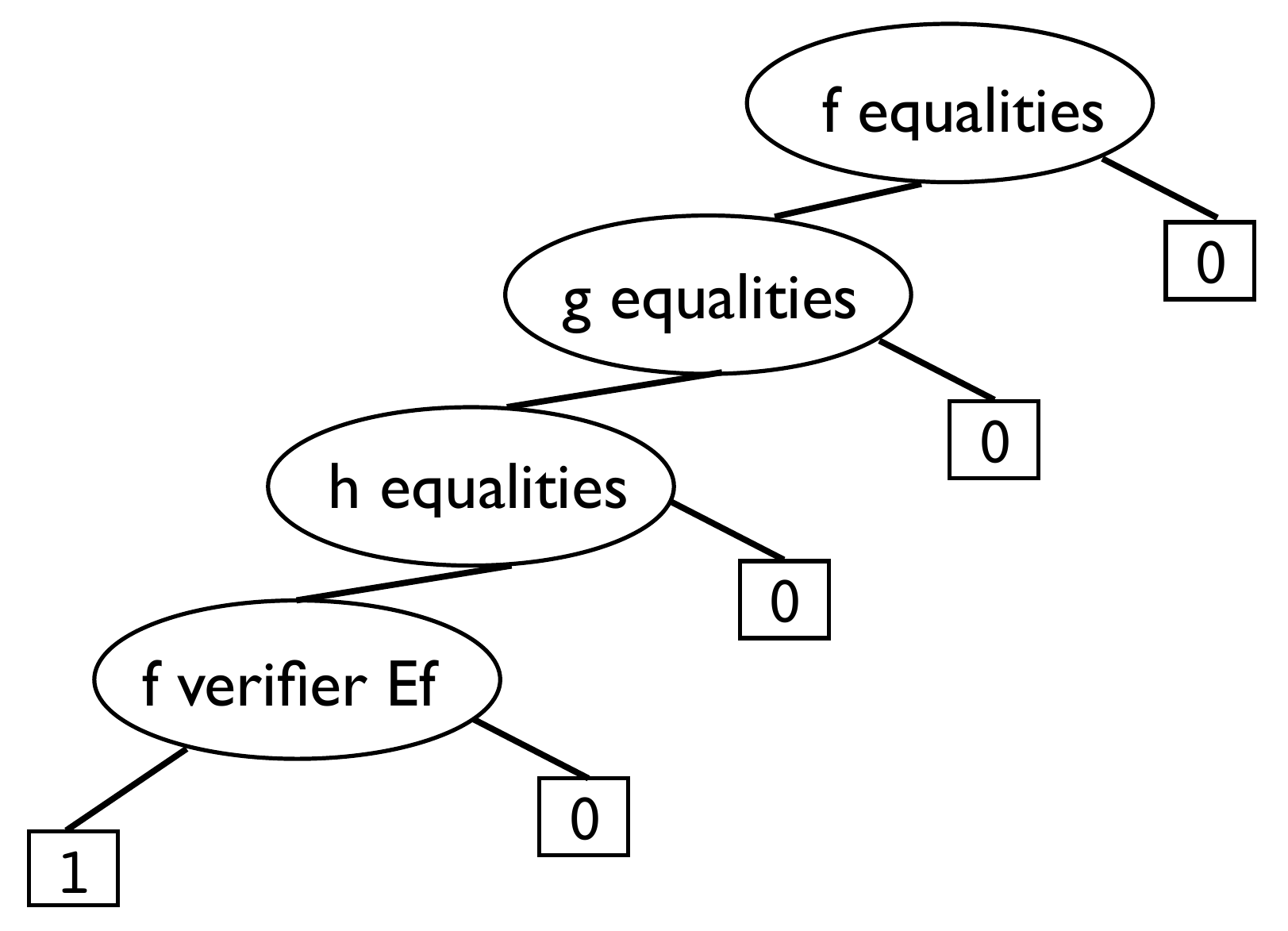}
\caption{Expanded view of the modified version of $B_1$. }
\label{fig:arrowing-b1-inflated-byGH}
\end{figure}

\paragraph{Fixing the construction to handle ordering and edge removal special case:}
We next consider the node ordering in $B_1$ and $B_2$. The diagram $B_1$ is sorted, 
where predicate order puts equalities above $E_F$ and arguments are lexicographically ordered.
For $B_2$ we consider the sub-block structure of the construction. Expanding each of the sub-blocks of $F$, $G$, $H$  in 
Figure~\ref{fig:arrowing-both-fodd} we observe that $B_2$ has the structure shown in Figure~\ref{fig:arrowing-B2expanded}. 
We further observe that each block is internally sorted, but blocks of equalities, $E_F$ and $E_C$ are interleaved.
By analyzing this structure we see that the blocks can be 
reordered at the cost of duplicating some portions yielding the structure in Figure~\ref{fig:arrowing-final-B2}. 
It is easy to see that $B_2$ is satisfied in $I$ if and only if the reordered diagram is satisfied in $I$. The diagrams yield the same value for any valuation $\zeta$ which does not exit to 0 due to bad node mapping for $G$ or $H$. Thus the original version might yield 1 (e.g., through $G$ path) when the reordered diagram yields 0 on such a valuation (e.g., via the $H$ equalities). But in such a case there is another valuation that is identical to $\zeta$ except that it modifies the bad node mapping (the $H$ equalities) and that yields 1 for the new diagram.
The final diagram is consistent with predicate ordering $ ``="  \prec  E_F \prec E_C$ and variable ordering where 
$f_i \prec g_j \prec h_k$ for all $i, j, k$.

Finally, we further change $B_1$ by adding the equality blocks of $G$ and $H$ to the construction, so that the modified $B_1$ is as shown in Figure~\ref{fig:arrowing-b1-inflated-byGH}. Using the same argument as in $B_2$ one can see that this does not change the semantics of $B_1$. Moreover, with this change $B_1$ can be obtained from $B_2$  by one edge removal (of the edge below the $F$ verifier in $B_2$) so that the reduction holds for this more restricted case.
\end{proof}

\bigskip
As mentioned above, FODD Value is defined similarly to FODD Satisfiability but requires more stringent conditions. The next result shows that this difference is important and FODD Value is one level higher in the hierarchy.

\begin{theorem}
\label{thm:maxfodd-value}
{ FODD Value} is $\Sigma_2^p$-complete.
\end{theorem}
\begin{proof}
The algorithm showing membership is as follows. 
We first observe
that by Lemma~\ref{lemma:foddsize} we can restrict our attention to small interpretations.
Given input $B$ and $V$ we guess an interpretation $I$ of the appropriate size. We then make two calls to an oracle for FODD Evaluation.
Let $V'$ be either the least leaf value greater than $V$ or one greater than the max leaf if $V$ is the maximum. We query the oracle for FODD Evaluation
 on $(B,I,V)$ and  $(B,I,V')$
and return Yes iff the oracle returns Yes on the first and No on the second. The algorithm returns Yes iff there is an interpretation $I$ with value $V$.  

For hardness we present a reduction from non-Equivalence of FODDs with binary leaves, 
which was shown to be $\Sigma_2^p$-hard in Theorem~\ref{thm:maxfodd-equiv}.
We are given $B_1=\max_{x_1} B_1(x_1)$ and $B_2=\max_{x_2} B_2(x_2)$ as input for FODD non-Equivalence where $B_1$ and $B_2$ are standardized apart so that
$x_1$, $x_2$ stand for disjoint sets of variables. We construct the diagram $B=\max_{x_1} \max_{x_2} B(x_1,x_2)$ where $B(x_1,x_2)=B_1(x_1)+B_2(x_2)$ can be calculated directly on the graph representation of $B_1$ and $B_2$ using the apply procedure of \cite{WangJoKh08}
(see Figure~\ref{fig:combineFODD}). Because $x_1$ and $x_2$ are disjoint, the diagram $B$ has the following behavior for any interpretation $I$: if $\map_{B_1}(I)=1$ and $\map_{B_2}(I)=1$ then $\map_B(I)=2$; otherwise if exactly one of them evaluates to 1 then $\map_B(I)=1$; and otherwise $\map_B(I)=0$. We produce $(B,V=1)$ as input for FODD Value.

Now, if $B_1$ and $B_2$ are not equivalent then there is an interpretation such that their maps are different, and without loss of generality we may assume $\map_{B_1}(I)=1$ and $\map_{B_2}(I)=0$. As argued above in this case  $\map_B(I)=1$ as needed.
For the other direction let $I$ be such that $\map_B(I)=1$. Then, again using the argument above, we have $\map_{B_1}(I)=1$ and $\map_{B_2}(I)=0$ or vice versa and the diagrams are not equivalent.
\end{proof}

\clearpage

%% file: gfodd.tex
\section{The Complexity of Reasoning with GFODD}
\label{sec:gfodd}

In this section we analyze the computational problems for GFODD.
We start with some observations on a notion of ``complements" for GFODDs.
Let $B$ be a GFODD associated with the ordered list of variables $(w_{i_1},\ldots,w_{i_m})$, and aggregation list $(A_1,\ldots, A_m)$ where each $A_i$ is $\min$ or $\max$. Let $B'= \mbox{complement}(B)$ 
(with respect to maximum value $M$) 
be the diagram corresponding to $B$ where we change leaf values and aggregation operators as follows:
Let $M$ be any value greater or equal to the max leaf value in $B$. Any leaf value $v$ is replaced with $M-v$. Each aggregation operator $A_i$ is replaced with $A'_i$ where
where if $A_i$ is $\min$ then $A'_i$ is $\max$ and vice versa.

\begin{theorem}
\label{thm:complement}
Let $B$ be a GFODD with $min$ and $max$ aggregation and maximum leaf value $\leq M$, and let $B'= \mbox{complement}(B)$.
For any interpretation $I$, $\map_B(I)=M-\map_{B'}(I)$.
\end{theorem}
\begin{proof}
By the construction of $B'$, for any valuation $\zeta$,  we have that $\map_B(I,\zeta)=M-\map_{B'}(I,\zeta)$.
Considering the aggregation process,
note that $A_i\  \map_B(\ldots w_i) =$ $A_i$  $\ [M-\map_{B'}(\ldots w_i) ]$ $= M - A'_i\  \map_{B'}(\ldots w_i)$.
Now using this fact, we can argue by induction backward from the innermost (rightmost) aggregation that for any prefix of variables $P=w_{i_1},\ldots,w_{i_p}$, valuation $\zeta_p$ for these variables, and remaining variables $R=w_{i_{p+1}},\ldots,w_{i_m}$ , we have $A_R\ \map_B(I,(P=\zeta_p;R))=M-A'_R \ \map_{B'}(I,(P=\zeta_p;R))$. When the prefix is empty we get the statement of the theorem.
\end{proof}

Notice that for diagrams with binary leaves this yields $\map_B(I)=1-\map_{B'}(I)$, that is, negation. 
As an immediate application we get the following:

\begin{corollary}
\label{cor:equiv-symmetric}
The complexity of GFODD Equivalence for $\min$-$k$-alternating GFODD is the same as the complexity of 
GFODD Equivalence for 
$\max$-$k$-alternating GFODD.
\end{corollary} 
\begin{proof}
By Theorem~\ref{thm:complement},
two $\min$ diagrams $B_1, B_2$ are equivalent if and only if their complements $B'_1, B'_2$ are equivalent where we can use the maximum among the leaf values of the two diagrams as $M$.
\end{proof}

\begin{corollary}
\label{cor:minfodd}
The equivalence problem for $\min$-GFODD is $\Pi_2^p$-complete. 
\end{corollary}

We can now turn to analysis of the computational problems.
Evaluation is similar to the FODD case but the hardness proof is more involved due to the interaction between quantifier order and node ordering in the diagram. 

\begin{theorem}
\label{thm:gfodd-eval}
GFODD Evaluation for $\max$-$k$-alternating GFODDs is $\Sigma_k^p$-complete.
GFODD Evaluation for $\min$-$k$-alternating GFODDs is $\Pi_k^p$-complete.
\end{theorem}
\begin{proof}
We prove membership by induction on $k$. Since the inductive step includes diagrams that  do not satisfy the sorting order we show that the claim holds in this more general case.
Consider the input $(B,I,V)$.
For the base case, $k=1$, 
we guess a valuation $\zeta$, calculate $v=\map_B(I,\zeta)$, and return Yes iff $v\geq V$.
In the $\max$ case, if the true value is at least $V$ then we say Yes for some $\zeta$, and if the true value is less than $V$ then $\map_B(I,\zeta)<v$ for all $\zeta$ and therefore we always say No. Thus the problem is in NP.
In the $\min$ case, if the true value is at least $V$ then all $\zeta$ yield Yes, and if the true value is less than $V$ then some $\zeta$ yields No. Thus the problem is in co-NP.

For the inductive step assume that the claim holds for $k-1$ and consider the input $(B,I,V)$ with an interpretation $I$, value bound $V$ and a $max$-$k$-alternating
diagram $B=\max_{\ww_1}\min_{\ww_2}\ldots Q^A_{\ww_k}$ $B(\ww_1,\ldots,\ww_k)$
where in order to simplify the notation each $\ww_i$ may be a single variable or a set of variables
and we use the boldface notation to denote this fact.

Now for each tuple $i$ of domain objects in $I$  (which is appropriate for the number of variables in $\ww_1$) let diagram $B'$ be $B'=\min_{\ww_2}\ldots Q^A_{\ww_k} B(\ww_1=i,\ldots,\ww_k)$. Clearly $B'$ is appropriate for evaluation on $I$ and by the inductive hypothesis we can appeal to a $\Pi_{k-1}^p$ oracle 
to solve GFODD Evaluation on $(B',I,V)$. Our algorithm guesses a value $i$, calculates $B'$, appeals to the oracle, and returns the same answer.
Now, if the true value is $< V$ then by definition any call to the oracle would yield No and we correctly answer No.
If the true value is $\geq V$ then for some $i$ the oracle would return Yes. Therefore we nondeterministically return Yes
and our algorithm is in NP$^{\Sigma_{k-1}^p}$.
The argument for the other aggregation prefix is symmetric and argued in the same manner yielding an algorithm in co-NP$^{\Sigma_{k-1}^p}$.

To show hardness we give a reduction from $QBF_k$. Given a quantified 3CNF Boolean formula we transform this into a GFODD $B$ and interpretation $I$ so that the following claim holds:

\bigskip
\noindent
{\bf Claim 1}:
$B$ evaluates to 1 in $I$ if and only if the quantified Boolean formula is satisfied. 
\smallskip

This claim establishes the theorem.
The reduction uses a similar structure to the one used for FODD satisfiability with two main differences. First because here we consider evaluation and we can control $I$ we do not need to test for an embedding of a Boolean predicate in $I$, that is, the first portion in that construction is not needed. On the other hand the construction and proof are more involved because of the alternation of quantifiers. 

The interpretation $I$ has two objects, a and b, where $P_T(a)=$true, and $P_T(b)=$false. Namely, $I = \{ [a, b], P_T(a)=$ true, and $P_T(b)=$ false$\}$.

Let the QBF formula be $Q_1x_1 Q_2x_2 \ldots Q_m x_m f$ where $Q$ is a quantifier $\forall$ or $\exists$ and the quantifiers come in $k$ alternating blocks.
As above, we start the construction by creating a set of ``shadow variables" corresponding to each QBF variable $x_i$.
The corresponding GFODD variables include $w_i$ and the set of $v_{(a,b)}$ that refer to $x_i$ or $\overline{x_i}$ in the QBF. We define $\ww_i$ to be the set of variables in the block corresponding to $x_i$ and associate these variables with an aggregation operator $Q^A_i$ where if
$Q_i$ is a $\exists$ then $Q^A_i$ is $\max$ and if 
$Q_i$ is a $\forall$ then $Q^A_i$ is $\min$. 
Using these variables, we build GFODD fragments we call variable consistency blocks. For each $x_i$, this gadget ensures that if two literals in the QBF refer to the same variable then the corresponding variables in the GFODD will have the same value. 
If this holds then a valuation goes through the block and continues to the next block. Otherwise, it exits to a default value, where for $\max$ blocks the default value is 0, and for $\min$ blocks the default value is 1.

\begin{figure}[t]
\centering
\includegraphics[scale = 0.50]{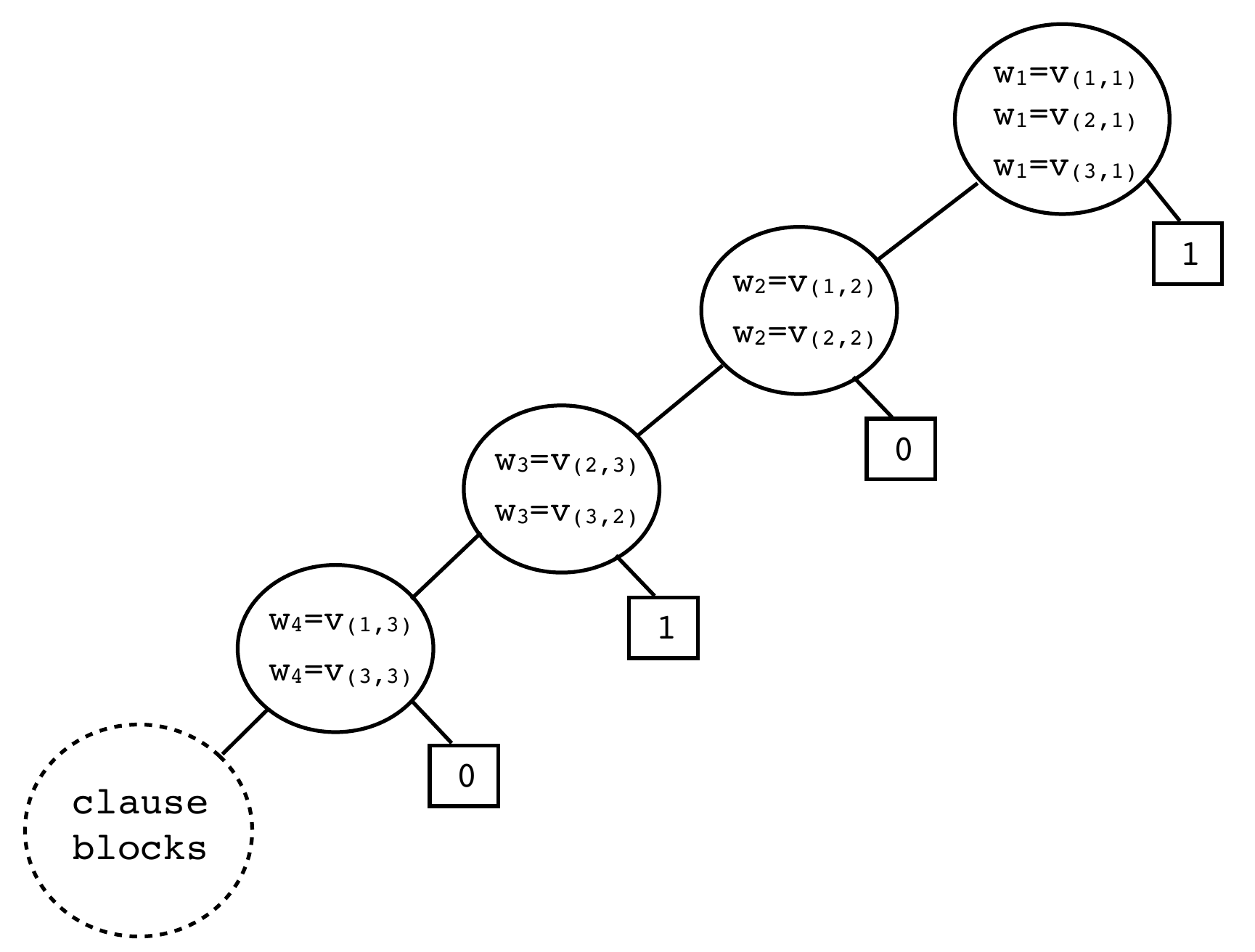}
\caption{Example of variable consistency blocks for reduction from QBF to GFODD Evaluation for the formula $\forall x_1 \exists x_2 \forall x_3 \exists x_4 (x_1 \vee \overline{x_2}  \vee x_4) \wedge (\overline{x_1} \vee x_2 \vee x_3) \wedge (x_1 \vee x_3 \vee \overline{x_4})$. }
\label{fig:gfodd-eval-consistency}
\end{figure}

Consider the expression 
$$\forall x_1 \exists x_2 \forall x_3 \exists x_4 (x_1 \vee \overline{x_2}  \vee x_4) \wedge (\overline{x_1} \vee x_2 \vee x_3) \wedge (x_1 \vee x_3 \vee \overline{x_4}),$$ 
which has the same clauses as in the previous proof but where we have changed the quantification.
Figure~\ref{fig:gfodd-eval-consistency} shows the variable consistency blocks for this example. 
Since,  $v_{(1,1)}$, $v_{(2,1)}$, and $v_{(3,1)}$  refer to $x_1$ we need to ensure that when they are evaluated they are evaluated consistently and
this is done by the first block. Because $x_1$ is a $\forall$ variable the default output value is 1. The consistency blocks are chained {\em in the same order as in the quantification of the QBF}. 
Once every consistency block has been checked, we continue to the clause blocks whose construction is exactly the same as in the previous proof (see Figure~\ref{fig:maxfodd-sat-clauses}). This yields the diagram $B$ where
we set the aggregation function to be $Q^A_1 \ww_1, Q^A_2 \ww_2, \ldots, Q^A_m \ww_m$.
Note that if the QBF has $k$ alternating blocks of quantifiers then $B$ has aggregation depth $k$.
The output of the reduction is the pair $(B,I)$.
The diagram is ordered with $``="\prec P_T$ and variables ordered lexicographically. 

We next show that Claim 1 holds.
We start by showing a correspondence between assignments to the Boolean formula $f$ and object assignments from $B$ to $I$. 
Let $v$ be a Boolean assignment. If $v$ assigns $x_i$ to 1 then $\zeta(v)$ maps the entire $\ww_i$ block to $a$. Otherwise $\zeta(v)$ maps the block to $b$. It is then easy to see that for all $v$, $\zeta(v)$ satisfies the consistency blocks and $f(v)=1$ if and only if $\map_B(I,\zeta(v))=1$.
This, however, does not complete the proof because $\map_B(I)$ must also consider valuations $\zeta$ that do not arise as maps of assignments $v$.

We divide the set of valuations to the GFODD into two groups. The first group of legal valuations, called {\em Group 1} below, is the set of valuations that is consistent with some $v$. 

The second group, {\em Group 2}, includes valuations that do not arise as $\zeta(v)$ and therefore they violate at least one of the consistency blocks. 
Let  $\zeta$ be such a valuation
and let $Q^A_j$ be the first block from the left whose constraint is violated. By the construction of $B$, in particular the order of equality blocks along paths in the GFODD, we have that the evaluation of the diagram on $\zeta$ ``exits" to a default value on the first violation. Therefore, if $Q_j$ is a $\forall$ then $\map_{B}(I,\zeta)=1$ and if $Q_j$ is a $\exists$ then $\map_{B}(I,\zeta)=0$.

We can now show the correspondence in truth values. Consider any partition of the blocks $1,\ldots, m$ into a prefix $1,\ldots, j$ and remainder $(j+1),\ldots, m$, and any Boolean assignment $v$ to the prefix blocks. We claim that for all such partitions

\begin{eqnarray*}
\lefteqn{ Q_{j+1} x_{j+1}, \ldots, Q_m x_m, \ \ f\left((x_1, \ldots,x_j)= v, (x_{j+1},\ldots,x_m)\right) = } \\
& & 
Q^A_{j+1} \ww_{j+1}, \ldots, Q^A_m \ww_m, \ \  \map_{B}(I,[(\ww_1,\ldots,\ww_j)=\zeta(v),(\ww_{j+1},\ldots,\ww_m)]).
\end{eqnarray*}

Note that when $j=0$, that is, the prefix is empty, the claim implies that $\map_{B}(I)$ is equal to 
$Q_1 x_1, \ldots, Q_m x_m, f$, completing the proof.
We prove the claim by induction, backwards from $m$ to 0. 
For the base case, $j=m$, and the second part is empty. The claim then follows because the prefix includes all variables and there is a 1-1 correspondence in truth values for substitutions in group 1. 

For the inductive step, the valuation $v$ covers the first $j-1$ blocks. 
Note that, by the inductive assumption, for any group 1 substitution $v_j$ for $x_j$ and corresponding, $\zeta(v_j)$ for $\ww_j$,
\begin{eqnarray*}
\lefteqn{ Q_{j+1} x_{j+1}, \ldots, Q_m x_m, \ \ f\left( (x_1,\ldots,x_{j-1})=v, (x_{j}=v_j),(x_{j+1},\ldots,x_m) \right)  = }\\
& & 
Q^A_{j+1} \ww_{j+1}, \ldots, Q^A_m \ww_m, \ \ \map_{B}(I,[(\ww_1,\ldots, \ww_{j-1})=\zeta(v),(\ww_{j}=\zeta(v_j)),(\ww_{j+1},\ldots,\ww_m)]).
\end{eqnarray*}

On the other hand, for any group 2 substitution $\zeta_j$ for $\ww_j$ and any values for $(\ww_{j+1},\ldots,\ww_m)$ we have that
the leftmost block whose constraint is violated for the corresponding combined $\zeta$ is block $j$ and therefore
$\map_{B}(I,[(\ww_2,\ldots,\ww_{j-1})=\zeta(v),(\ww_{j}=\zeta_j),(\ww_{j+1},\ldots,\ww_m)])$ gets the default value for that block. 
Therefore, 
the aggregation over the $j$th block is determined by group 1 valuations, which are in turn identical to the QBF value and 

\begin{eqnarray*}
\lefteqn{ Q_{j} x_{j}, \ldots, Q_m x_m, \ \ f(x_1,\ldots,x_{j-1})=v, (x_{j},\ldots,x_m)) = } \\
& &
Q^A_{j} \ww_{j}, \ldots, Q_m \ww_m,  \ \ \map_{B}(I,[(\ww_2,\ldots,\ww_{j-1})=\zeta(v),(\ww_{j},\ldots,\ww_m)])
\end{eqnarray*}
as required.
\end{proof}

\begin{figure}[t]
\centering
\includegraphics[scale = 0.50]{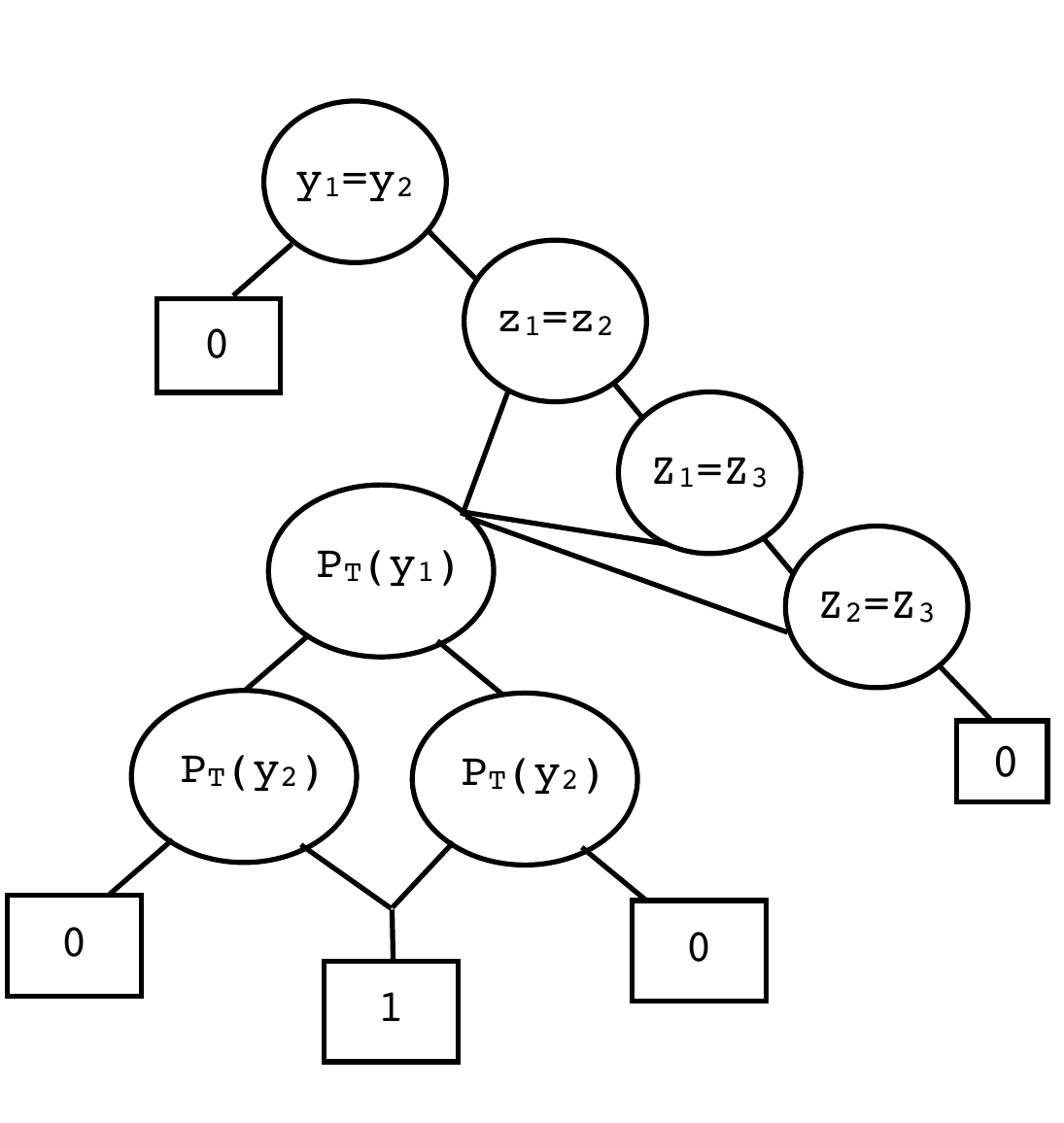}
\caption{The $B_1$ diagram for GFODD Satisfiability reduction.
The aggregation function is $\max_{y_1} \max_{y_2} \min_{z_1} \min_{z_2} \min_{z_3}$. 
}
\label{fig:gfodd-sat-B1}
\end{figure}

\bigskip
It turns out that the complexity of satisfiability is different for min and max diagrams, and their analysis requires different proofs. We therefore start with $\max$-$k$-alternating GFODDs. The case of $\min$-$k$-alternating  GFODDs is analyzed after the analysis of equivalence because it is using similar techniques.

\begin{theorem}
\label{thm:godd-sat}
GFODD  Satisfiability for $\max$-$k$-alternating GFODDs  (where $k\geq 2$) is $\Sigma_k^p$-complete.
\end{theorem}
\begin{proof}
We first show membership.
Let $B$ be a GFODD with aggregation  $\max \ww_1, \min \ww_2, \ldots, Q^A_k \ww_k$. Our algorithm nondeterministically chooses an interpretation $I$ and a tuple of values for $\ww_1$, from the domain of objects for $I$. Let $i$ refer to this tuple of objects. We create a new GFODD,  $B'=\min \ww_2 \ldots Q^A_k B(\ww_1=i,\ldots,w_k)$ and appeal to a $\Pi_{k-1}^p$ oracle to solve GFODD evaluation on $(B', I)$. If the oracle query returns 1 then we accept and otherwise we reject. 
The result is clearly correct using an algorithm in $NP^{\Sigma_{k-1}^p}$. 

The hardness argument is similar to the proof for GFODD evaluation.
The main extension is that in the current proof we verify that any satisfying $I$ embeds the interpretation from the previous proof. 
The reduction gets a $QBF_k$ formula, $Q_1x_1Q_2x_2 \ldots Q_mx_m f$, with $Q_i$ either a $\forall$ or $\exists$ quantifier. 
We first construct two diagrams $B_1$ and $B_2$, where $B_2$, the QBF validation diagram, is exactly as in the proof of Theorem~\ref{thm:gfodd-eval}, that is, it includes consistency blocks followed by clause blocks. 
The diagram $B_1$ has two portions. The first verifies that $I$ has at least two objects and the truth values of $P_T()$ on these objects are different. 
The second portion verifies that $I$ has at most two objects. 
This is implemented using $\min$ variables such that if we identify three distinct objects we set the value to 0. 
The two portions are put together so as to respect predicate order, and the final diagram $B_1$ is shown in Figure~\ref{fig:gfodd-sat-B1}.
The aggregation function for $B_1$ is $\max_{y_1} \max_{y_2} \min_{z_1} \min_{z_2} \min_{z_3}$. 

Let $I^* = \{ [a, b], P_T(a)=$true, and $P_T(b)=$false$\}$ be the intended interpretation. 
We have the following two claims:
\\ 
(C1) for all $I$, $\map_{B_1}(I)=1$ if and only if $I$ is isomorphic to $I^*$.
\\
(C2)
if $I$ is isomorphic to $I^*$ then $\map_{B_2}(I)=1$ if and only if $(Q_1x_1Q_2x_2 \ldots Q_mx_m f) = 1$.

C2 is exactly the same as Claim~1 in the proof of Theorem~\ref{thm:gfodd-eval}. 
For C1, given $I$ which is isomorphic to $I^*$, the valuation of $y_1,y_2$ to $a,b$ and any valuation to the $z$'s yields a map value of 1. 
Therefore, considering the aggregation order we see that for $(y_1,y_2) = (a,b)$  in $B_1$ the minimum over $z$ yields 1, and then the maximum over $y$'s is 1. 
For the other direction, we need to consider interpretations not isomorphic to $I^*$.
If $I$ has only one object then its map is 0 for all valuations, and therefore the aggregated value is 0. 
If $I$ has at least 3 objects then for any fixed valuation for $y$ the minimum over $z$ is 0, implying that the maximum over $y$ also yields 0 and $\map_{B_1}(I)=0$.
Finally consider any $I$ with two objects where $P_T()$ has the same truth value on the two objects. In this case the map is 0 for any valuation and
thus the final map value is 0.
We have therefore shown that C1 holds.

For our reduction, we produce $B=apply(B_1,B_2,\wedge)$ where for the aggregation we make use of Theorem~\ref{thm:gfodd-combine} and interleave the aggregation functions of $B_1$ and $B_2$ so that $B$ has at most $k$ alternations of quantifiers. 
This is always possible because the QBF starts with a $\exists$ quantifier and $k\geq 2$.

By the claims C1 and C2 and Theorem~\ref{thm:gfodd-combine} we get that $\map_B(I)=1$ if and only if $I$ is isomorphic to $I^*$  and $(Q_1x_1Q_2x_2 \ldots Q_mx_m f) = 1$. Therefore, the QBF is true if and only if there exists an interpretation $I$ (which must be isomorphic to $I^*$) that satisfies $B$.
\end{proof}

\bigskip
Equivalence is one level higher in the hierarchy; using a reduction from QBF we show how to ``peel off" one level of quantifiers and push that into the ``existential quantification" over interpretations that potentially witness non-equivalence.

\begin{theorem}
\label{thm:gfodd-equiv}
GFODD Equivalence  and {GFODD Edge Removal}  for diagrams with aggregation depth $k$ (where $k\geq 2$)
are $\Pi_{k+1}^p$-complete.
\end{theorem}
\begin{proof}
By Corollary~\ref{cor:equiv-symmetric}, it suffices to show that this holds for 
$\max$-$k$-alternating GFODDs.
Since Edge Removal is a special case of Equivalence it suffices to show membership for Equivalence and hardness for Edge Removal. 

To show membership we show that the complement, nonequivalence, is in $\Sigma_{k+1}^p$. Given two 
$\max$-$k$-alternating GFODDs
$B_1$ and $B_2$ as input,  we guess an interpretation $I$ of the appropriate size, and then appeal to an oracle for GFODD Evaluation to calculate $\map_{B_1}(I)$ and $\map_{B_2}(I)$.   Using these values we return Yes or No accordingly.
To calculate the map values, let $B$ be one of these diagrams, and let the leaf values of the diagram be $v_1,v_2,\ldots,v_\ell$. We make $\ell$ calls to GFODD Evaluation 
with $(B,I,v_i)$ as input. 
Each call requires an oracle in $\Sigma_k^p$ and
$\map_{B}(I)$ is the largest value on which the oracle returns Yes.
Clearly if a witness for nonequivalence exists then this process can discover it and say Yes (per non-equivalence), and otherwise it will always say No. Therefore non-equivalence is in NP$^{\Sigma_{k}^p}$, that is $\Sigma_{k+1}^p$ and equivalence is in $\Pi_{k+1}^p$.

We reduce QBF satisfiability with $k\geq 3$ alternations of quantifiers to equivalence of $\max$-$(k-1)$-alternating GFODDs. 
The reduction is conceptually similar to the one from the previous theorem but the details are more involved. 
In particular, here we assume a QBF whose first quantifier is $\forall$,  that is,
$\forall x_1, Q_2x_2 \ldots Q_mx_m f(x_1,x_2,\ldots,x_m)$ where this form has 
$k$ blocks of quantifiers. To simplify the notation it is convenient to collect adjacent variables having the same quantifiers into groups so that the QBF 
has the form
$\forall \xx_1 \ldots Q_k\xx_k f(\xx_1,\xx_2,\ldots,\xx_k)$ where $\xx_i$ refers to a set of variables.

We next define a notion of ``legal interpretations" for our diagrams. 
A legal interpretation embeds the binary interpretation $I^*$ from the previous proof and in addition includes a truth setting for all the variables in the first $\forall$ block of the QBF. 
The reduction constructs diagrams $B_1$, $B_2$, and $B=apply(B_1,B_2,\wedge)$ such that the following claims hold:
\\ 
(C1) for all $I$, $\map_{B_1}(I)=1$ if and only if $I$ is legal.
\\
(C2)
if $I$ is legal and it embeds the substitution $\xx_1=\alpha$ then $\map_{B_2}(I)=1$ if and only if $Q_2\xx_2 \ldots Q_k\xx_k f((\xx_1=\alpha),\xx_2,\ldots,\xx_k) = 1$.

We then output the diagrams $(B_1,B)$ for GFODD equivalence. Note that, by C1 and
Theorem~\ref{thm:gfodd-combine},
for non-legal interpretations we have $\map_{B_1}(I)= \map_{B}(I)=0$ and therefore if the diagrams are not equivalent it must be because of legal interpretations. 
Now, if the QBF is satisfied then, by definition, for all $\xx_1=\alpha$ we have that $Q_2\xx_2 \ldots Q_k\xx_k f((\xx_1=\alpha),\xx_2,\ldots,\xx_k) = 1$. Therefore, by C2, for all legal $I$, 
$\map_{B_2}(I)=1$ and by C1 and the construction $\map_{B}(I)=1$. Thus $B$ and $B_1$ are equivalent. 

On the other hand, if the QBF is not satisfied then there is a substitution $\xx_1=\alpha$ where $Q_2\xx_2 \ldots Q_k\xx_k f((\xx_1=\alpha),\xx_2,\ldots,\xx_k) = 0$. Therefore, by C2, for the corresponding interpretation $I'$, 
$\map_{B_2}(I')=0$ and by Theorem~\ref{thm:gfodd-combine}
we also have $\map_{B}(I')=0$. But by C1, $\map_{B_1}(I')=1$ and therefore $B_1$ and $B$ are not equivalent. 

We now proceed with the reduction, starting first with a simplified construction ignoring ordering of node labels and edge removal structure, and then elaborating to enforce these constraints. 
The set of predicates includes $P_T()$ and for every QBF variable $x_i$ in the first $\forall$ block we use a predicates $P_{x_i}()$. 
Notice that each $x_i$ is a member of $\xx_1$ (the first $\forall$ group) where the typeface distinguishes the individual variables in the first block, from blocks of variables.
In the simplified construction, a legal interpretation has exactly two objects, say $a$ and $b$, where $P_T(a)\not = P_T(b)$ and where for each $P_{x_i}()$ we have $P_{x_i}(a) = P_{x_i}(b)$. In other words, the assignment of an object to $v$ in $P_T(v)$ simulates an assignment to Boolean values, but the truth value of $P_{x_i}(v)$ is the same regardless of which object is assigned to $v$. 

\begin{figure}[t]
\centering
\includegraphics[scale = 0.85]{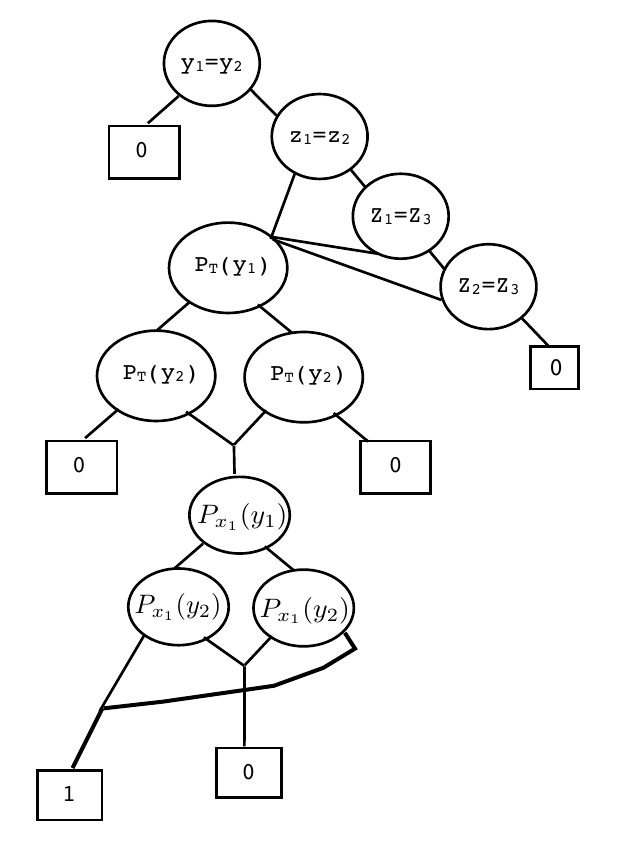}
\caption{The $B_1$ diagram for GFODD Equivalence reduction.}
\label{fig:gfodd-equiv-B1}
\end{figure}

In our example QBF
$\forall x_1 \exists x_2 \forall x_3 \exists x_4 (x_1 \vee \overline{x_2}  \vee x_4) \wedge (\overline{x_1} \vee x_2 \vee x_3) \wedge (x_1 \vee x_3 \vee \overline{x_4})$
the first block includes only the variable $x_1$ and the following interpretation is legal:
$I = \{ [a, b], P_T(a)=$true, $P_T(b)=$false, $P_{x_1}(a) = P_{x_1}(b)=$false$\}$.

The diagram $B_1$ has three portions where the first two are exactly as in the previous proof, thus verifying that $I$ has two objects and that $P_T()$ behaves as stated. The third portion verifies that each $P_{x_i}()$ behaves as stated, where we use a sequence of blocks, one for each $P_{x_i}()$.
The combined diagram $B_1$ is shown in Figure~\ref{fig:gfodd-equiv-B1} and the aggregation function is $\max_{y_1,y_2},\min_{z_1,z_2,z_3}$.

To see that C1 holds consider all possible cases for non-legal interpretations. If $I$ has at most one object the map is 0 for all valuations and thus the aggregation is 0. If $I$ has at least 3 objects, then for any values for $y_1,y_2$ the min aggregation over $z$ yields 0, and therefore the map is 0. If $I$ has 2 objects but it violates the condition on $P_T$ or $P_{x_i}$ then again the map is 0 for any valuation and the aggregation is 0. On the other hand, if $I$ is legal, then the $z$ block never yields 0 and the correct mapping to $y_1,y_2$ yields 1. Therefore the aggregation is 1.

The diagram $B_2$ is constructed by modifying $B_2$ from the proof of Theorem~\ref{thm:gfodd-eval}. The first modification is to handle the first $\forall$ block differently.  
As it turns out, all we need to do is replace the $\min$ aggregation for the $\ww_1$ block with maximum aggregation and accordingly replace the default value on that block to 0. 
To avoid confusion we recall that the current proof uses a slightly different notation from the previous one. In the current proof $\xx_i$ is a set of variables from the QBF and therefore $\ww_i$ is a set of blocks of variables, all of which have the same aggregation function.
The modified variable consistency diagram is shown in Figure~\ref{fig:gfodd-equiv-consistency}.
The clause blocks have the same structure as in the previous construction but use $P_{x_i}(V_{(i_1,i_2)})$ when $x_i$ is a $\forall$ variable from the first block and use $P_{T}(V_{(i_1,i_2)})$ otherwise.
This is shown in Figure~\ref{fig:gfodd-equiv-clauses}. 
$B_2$ includes the variable consistency blocks followed by the clause blocks. Note that the new clause blocks are not sorted in any consistent order because the predicates $P_{x_i}()$ and $P_{T}()$ appear in an arbitrary ordering determined by the appearance of literals in the QBF. Other than this violation, all other portions of the diagrams described are sorted where the predicate order has $= \prec P_T \prec P_{x_i}$ and where variables $w_i$ are before $v_{(i_1,i_2)}$ and variables within group are sorted lexicographically. 
The combined aggregation function is $\max_{\ww_1}, \max_{\ww_2},\min_{\ww_3},\ldots,Q^A_{\ww_k}$.

\begin{figure}[t]
\centering
\includegraphics[scale = 0.50]{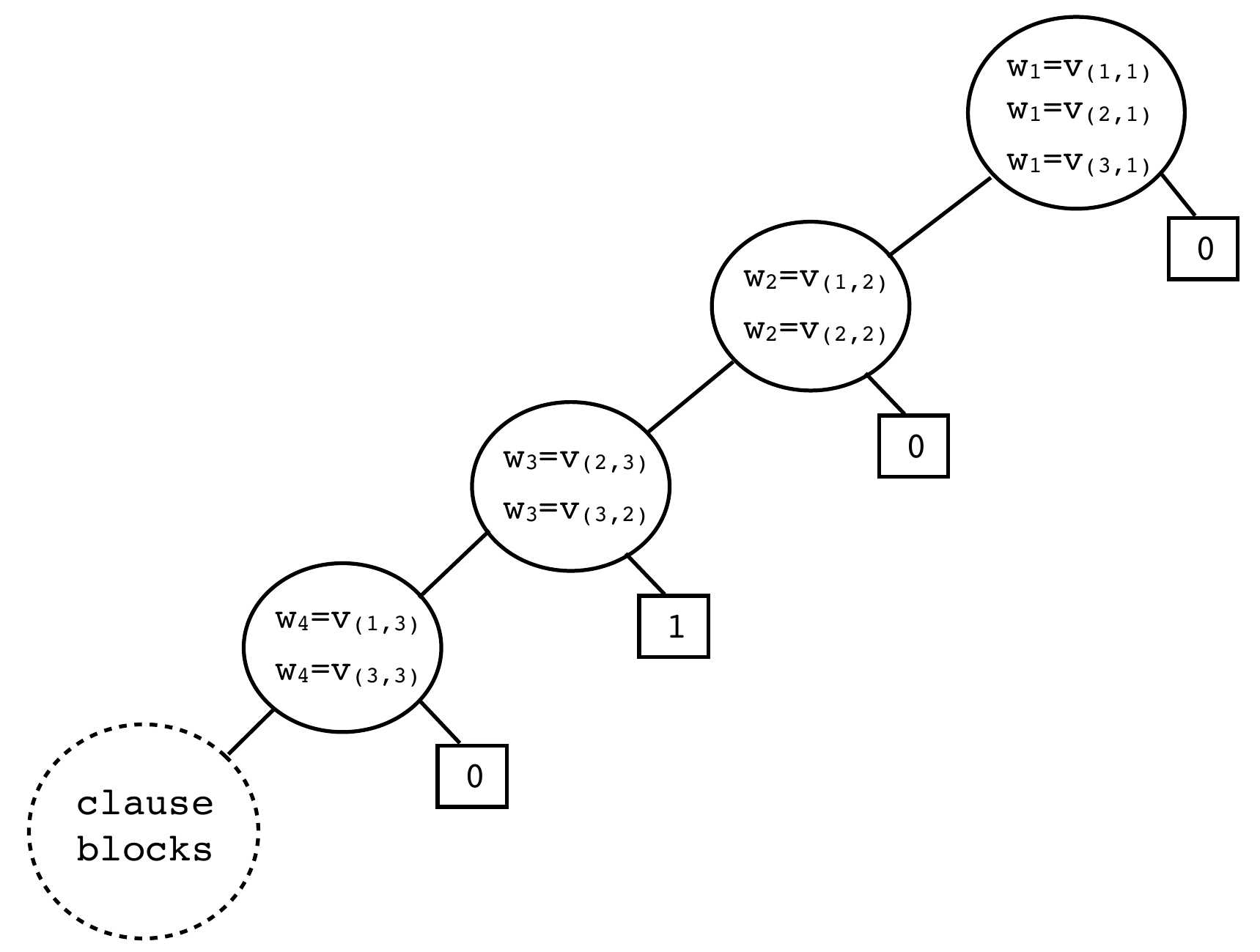}
\caption{The block consistency diagram for the GFODD equivalence reduction. The only difference from Figure~\ref{fig:gfodd-eval-consistency} is the default value on the $\ww_1$ block.}
\label{fig:gfodd-equiv-consistency}
\end{figure}

\begin{figure}[t]
\centering
\includegraphics[width = 0.27\textwidth]{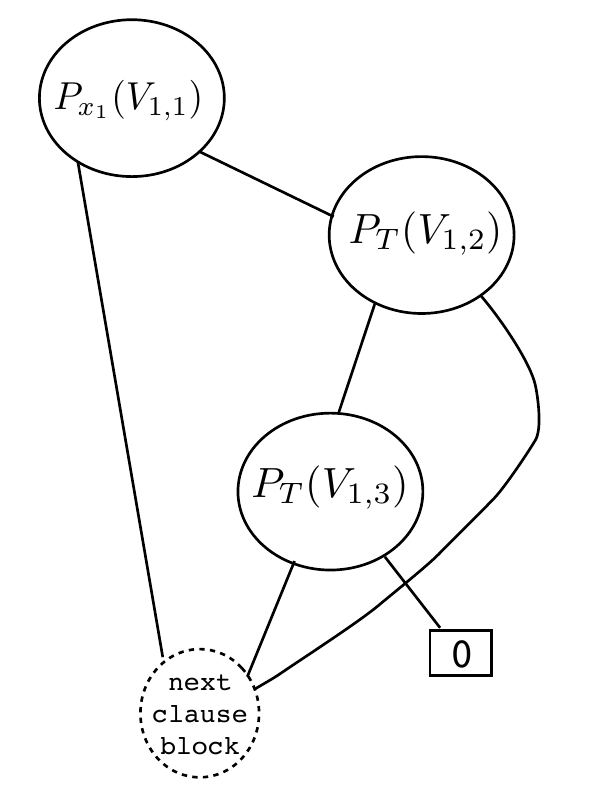}
\includegraphics[width = 0.27\textwidth]{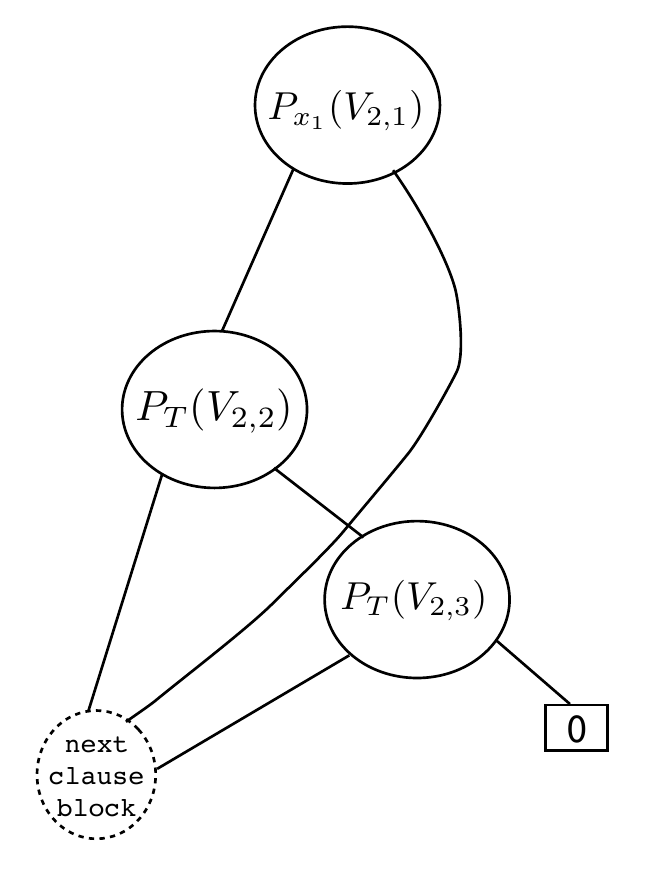}
\includegraphics[width = 0.38\textwidth]{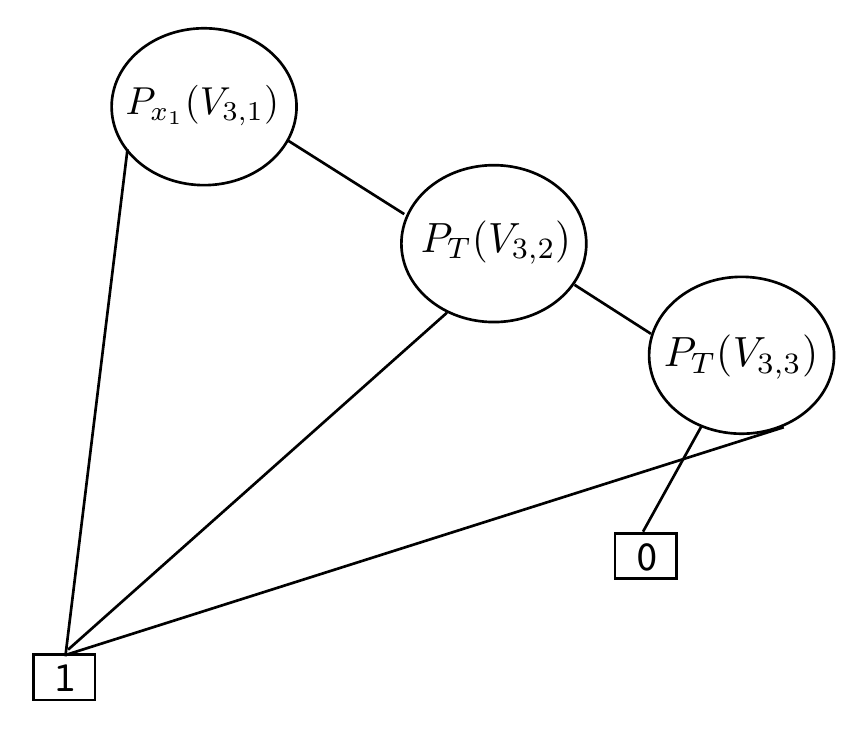}
\caption{The clause blocks for the GFODD equivalence reduction. For all occurrences of the universal variables of the first block ($x_1$ in this example) we have been replaced $P_T$ with $P_{x_i}$.}
\label{fig:gfodd-equiv-clauses}
\end{figure}

We next show that claim C2 holds, which will complete the proof of the simplified construction. 
Consider any legal $I$, let the corresponding truth values for variables in $\xx_1$ be denoted $\alpha$, and consider valuations for the QBF extending $\xx_1=\alpha$. Now consider any valuation $v$ to the remaining variables in the QBF and the induced substitution to the GFODD variables $\zeta(v)$ that is easily identified from the construction. Add any consistent group assignment to $\ww_1$ (that is, we assign $a$ or $b$ to each subgroup of variables in that group)
to $\zeta(v)$ to get $\hat{\zeta}(v)$.
By the construction of $B_2$ we have that $f([\xx_1=\alpha ,(\xx_2,\ldots,\xx_k)=v]) = \map_{B_2}(I,\hat{\zeta}(v))$. To see this note that there are no quantifiers in this expression, there is a 1-1 correspondence between the valuations of $\xx_2,\ldots,\xx_k$ and $\ww_2,\ldots,\ww_k$, and that as long as the assignment to the $\ww_1$ block is group consistent it does not affect the value returned. 
We call this set of valuations, that arise as translations of substitutions for QBF variables, {\em Group 1}.

The second group, {\em Group 2}, includes valuations that do not arise as $\hat{\zeta}(v)$ and therefore they violate at least one of the consistency blocks. 
Let  $\zeta$ be such a valuation
and let $Q^A_j$ be the first block from the left whose constraint is violated. By the construction of $B_2$, in particular the order of equality blocks along paths in the GFODD, we have that the evaluation of the diagram on $\zeta$ ``exits" to a default value on the first violation. Therefore, if $j=1$, that is the violation is in the block of $\ww_1$, $\map_{B}(I,\zeta)=0$ and for $j\geq 2$ if
$Q_j$ is a $\forall$ then $\map_{B}(I,\zeta)=1$ and if $Q_j$ is a $\exists$ then $\map_{B}(I,\zeta)=0$.

We can now show the correspondence in truth values. Consider any partition of the blocks $2,\ldots, k$ into a prefix $2,\ldots, j$ and remainder $(j+1),\ldots, k$, and any valuation $v$ to the prefix blocks. We claim that for all such partitions 
\begin{eqnarray*}
\lefteqn{Q_{j+1} \xx_{j+1}, \ldots, Q_k \xx_k, f(\xx_1=\alpha,(\xx_2,\ldots,\xx_j)=v, \xx_{j+1},\ldots,\xx_k) = }
\\
&  &
Q_{j+1} \ww_{j+1}, \ldots, Q_k \ww_k, \map_{B_2}(I,[\ww_1=a,(\ww_2,\ldots,\ww_j)=\zeta(v),(\ww_{j+1},\ldots,\ww_k)]).
\end{eqnarray*}
Note that when $j=1$, that is, the prefix is empty, this yields that $Q_{2} \xx_{2}, \ldots, Q_k \xx_k, f(\xx_1=\alpha,\xx_2,\ldots,\xx_k)$ is equal to 
$Q^A_{2} \ww_{2}, \ldots, Q^A_k \ww_k, \map_{B_2}(I,[\ww_1=a,\ww_2,\ldots,\ww_k)])$.
Now because the default value for violations of $\ww_1$ is 0 and because the aggregation for $\ww_1$ is $\max$ the latter expression is equal to 
$Q^A_{1} \ww_{1}, \ldots, Q^A_k \ww_k, \map_{B_2}(I,[\ww_1,\ww_2,\ldots,\ww_k)])$.
This means that $\map_{B_2}(I)$ is equal to 
$Q_{2} \xx_{2}, \ldots, Q_k \xx_k, f(\xx_1=\alpha,\xx_2,\ldots,\xx_k)$
completing the proof of C2.

We prove the claim by induction, backwards from $k$ to 1. 
For the base case, $j=k$, and the second part is empty. The claim then follows because the prefix includes all variables and there is a 1-1 correspondence in truth values for substitutions in group 1. 

For the inductive step, the valuation $v$ covers the first $j-1$ blocks. 
Note that, by the inductive assumption, for any group 1 substitution $v_j$ for $\xx_j$ and corresponding, $\zeta(v_j)$ for $\ww_j$,
$
Q_{j+1} \xx_{j+1}, \ldots, Q_k \xx_k,$ $ f(\xx_1=\alpha,(\xx_2,\ldots,\xx_{j-1})=v, (\xx_{j}=v_j),\xx_{j+1},\ldots,\xx_k)$
 $= Q^A_{j+1} \ww_{j+1}, \ldots, Q^A_k \ww_k,$ $\map_{B_2}(I,[\ww_1=a,(\ww_2,\ldots,$ $\ww_{j-1})=\zeta(v),(\ww_{j}=\zeta(v_j)),(\ww_{j+1},\ldots,\ww_k)])
$.
On the other hand, for any group 2 substitution $\zeta_j$ for $\ww_j$ and any values for $(\ww_{j+1},\ldots,\ww_k)$ we have that
the violating block for the corresponding combined $\zeta$ is block $j$ and therefore
$\map_{B_2}(I,[\ww_1=a,(\ww_2,\ldots,\ww_{j-1})=\zeta(v),(\ww_{j}=\zeta_j),(\ww_{j+1},\ldots,\ww_k)])$ gets the default value for that block. 
Therefore, 
the map is determined by group 1 valuations, which are in turn identical to the QBF value and 
$
Q_{j} \xx_{j}, \ldots, Q_k \xx_k,$ $f(\xx_1=\alpha,(\xx_2,\ldots,\xx_{j-1})=v, (\xx_{j},\ldots,\xx_k))
= Q^A_{j} \ww_{j}, \ldots, Q^A_k \ww_k, \map_{B_2}(I,[\ww_1=a,(\ww_2,\ldots,\ww_{j-1})$ $=\zeta(v)$ $,(\ww_{j},\ldots,\ww_k)])
$ as required. Therefore, the claim on the correspondence of values holds. This completes the proof of C2.

\paragraph{Extending the reduction to handle ordering  and edge removal special case:}
The main idea in the extended construction is to replace the unary predicates $P_T$ and $P_{x_i}$ with one binary predicate $q(\cdot,\cdot)$ where the ``second argument" in $q()$ serves to identify the corresponding predicate and hence its truth value. In addition we force $q()$ to be symmetric so that for any $A$ and $B$ the truth value of $q(A,B)$ is the same as the truth value of $q(B,A)$. In this way we have freedom to use either $q(A,B)$ or $q(B,A)$ as the node label which provides sufficient flexibility to handle the ordering issues. To implement this idea we need a few additional constructions.

Let the set of variables in the first $\forall$ block of the QBF be $x_1,x_2,\ldots,x_\ell$. 
The set of predicates in the extended reduction includes
unary predicates $T(),  y_1(), y_2(), x_1(), x_2(),\ldots,x_\ell(), $ and one binary predicate $q(\cdot,\cdot)$. A legal interpretation includes exactly $\ell+3$  objects which are uniquely identified by the unary predicates. We therefore slightly abuse notation and use the same symbols for the objects and predicates. 
In particular, the  atoms $y_1(y_1), y_2(y_2)$, $T(T)$, and $x_1(x_1), x_2(x_2), \ldots, x_\ell(x_\ell)$ are true in the interpretation and only these atoms are true for these unary predicates (e.g., $x_1(T)$ is false).
The truth values of $q()$ reflect the simulation of $P_T()$ and $P_{x_i}()$ in addition to being symmetric. 
Thus  the truth values of $q(y_1,T)$ and $q(T,y_1)$ are the same and they are the negation of the truth values of $q(y_2,T)$ and $q(T,y_2)$. 
For all $i$, the truth values of $q(y_1,x_i)$, $q(x_i,y_1)$, $q(y_2,x_i)$ and $q(x_i,y_2)$ are the same.
The truth values of other instances of $q()$, for example, $q(x_2,T)$, can be set arbitrarily. For example, the following interpretation is legal when $\ell=1$:
$I = \{ [a, b, c, d], y_1(a), y_2(b), T(c), x_1(d), q(c,a)=q(a,c)=$ true, $q(c,b)=q(b,c)=$ false, $q(d,a)=q(a,d)=q(d,b)=q(b,d)=$ false, $q(\cdot,\cdot)=$ false$\}$ where $q(\cdot,\cdot)$ refers to any instance not explicitly mentioned in the list.

We next define the diagram $B_1$ that is satisfied only in legal interpretations. 
We enforce exactly $\ell+3$ objects using two complementary parts. The first includes ${\ell+4 \choose 2}$ inequalities on a new set of $\ell+4$ variables $z_1,\ldots,z_{\ell+4}$ with min aggregation. If we identify $\ell+4$ distinct objects we set the value to 0.
This is shown in Figure~\ref{fig:gfodd-equiv-few-objects}.

To enforce at least $\ell+3$ objects and identify them we use the following gadget. For each of the unary predicates we have a diagram identifying its object and testing its uniqueness where we use both max and min variables. This is shown for the predicate $T()$ in Figure~\ref{fig:gfodd-equiv-unary-T}. The node $T(T)$ with max variable $T$ identifies the object $T$. The nodes $T(r_1),T(r_2)$ with min variables $r_1,r_2$ make sure that $T$ holds for at most one object. 
We chain the diagrams together as shown in Figure~\ref{fig:gfodd-equiv-sorted-unary} where the variables $r_1,r_2$ are shared among all unary predicates. The corresponding aggregation function is $\max_{T,y_1,y_2,x_1,\ldots, x_{\ell}}, \min_{r_1,r_2}$. This diagram associates each of the $\ell+3$ objects with one of the unary predicates and in this way provides a reference to specific objects in the interpretation.

The symmetry gadget for $q()$ is shown in Figure~\ref{fig:gfodd-equiv-q-symmetric} where the variables $m_1,m_2$ are min variables. If an input interpretation has two objects $A,B$ where $q(A,B)$ has a truth value different than $q(B,A)$ then minimum aggregation will map the interpretation to 0.

The truth value gadget for the simulation of $P_T$ is shown in Figure~\ref{fig:gfodd-equiv-q-of-PT} 
and
the truth value gadget for the simulation of $P_{x_i}$ is shown in Figure~\ref{fig:gfodd-equiv-q-of-xi}.
These diagram fragments refer to variables in other portions and they will be connected and aggregated together.

\begin{figure}[t]
\centering
\includegraphics[scale = 0.55]{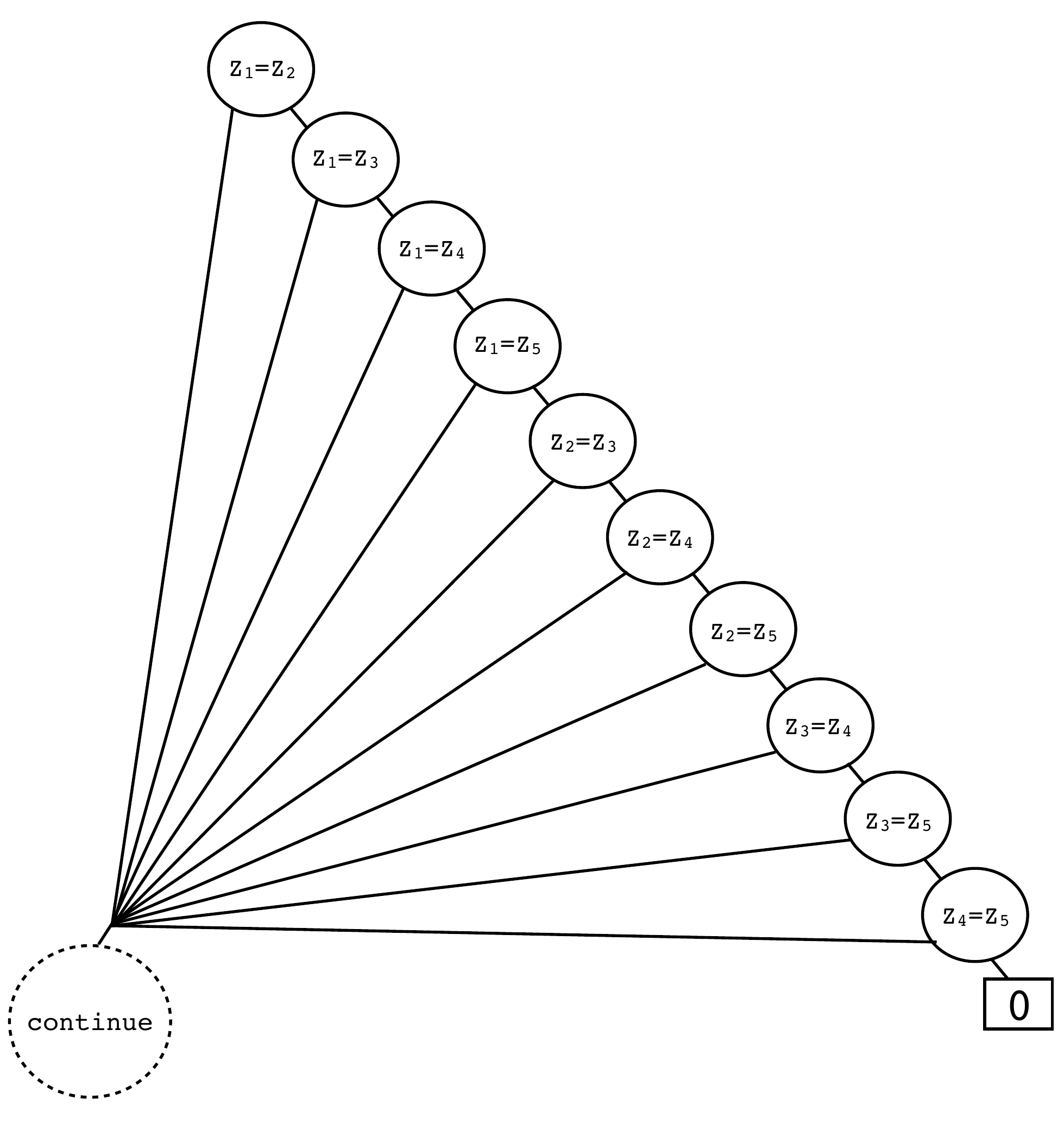}
\caption{GFODD equivalence reduction. The diagram verifying that the interpretation has less than ${\ell +4}$ objects. Here $\ell=1$ and $\ell+4=5$.}
\label{fig:gfodd-equiv-few-objects}
\end{figure}

\begin{figure}[t]
\centering
\includegraphics[scale = 0.50]{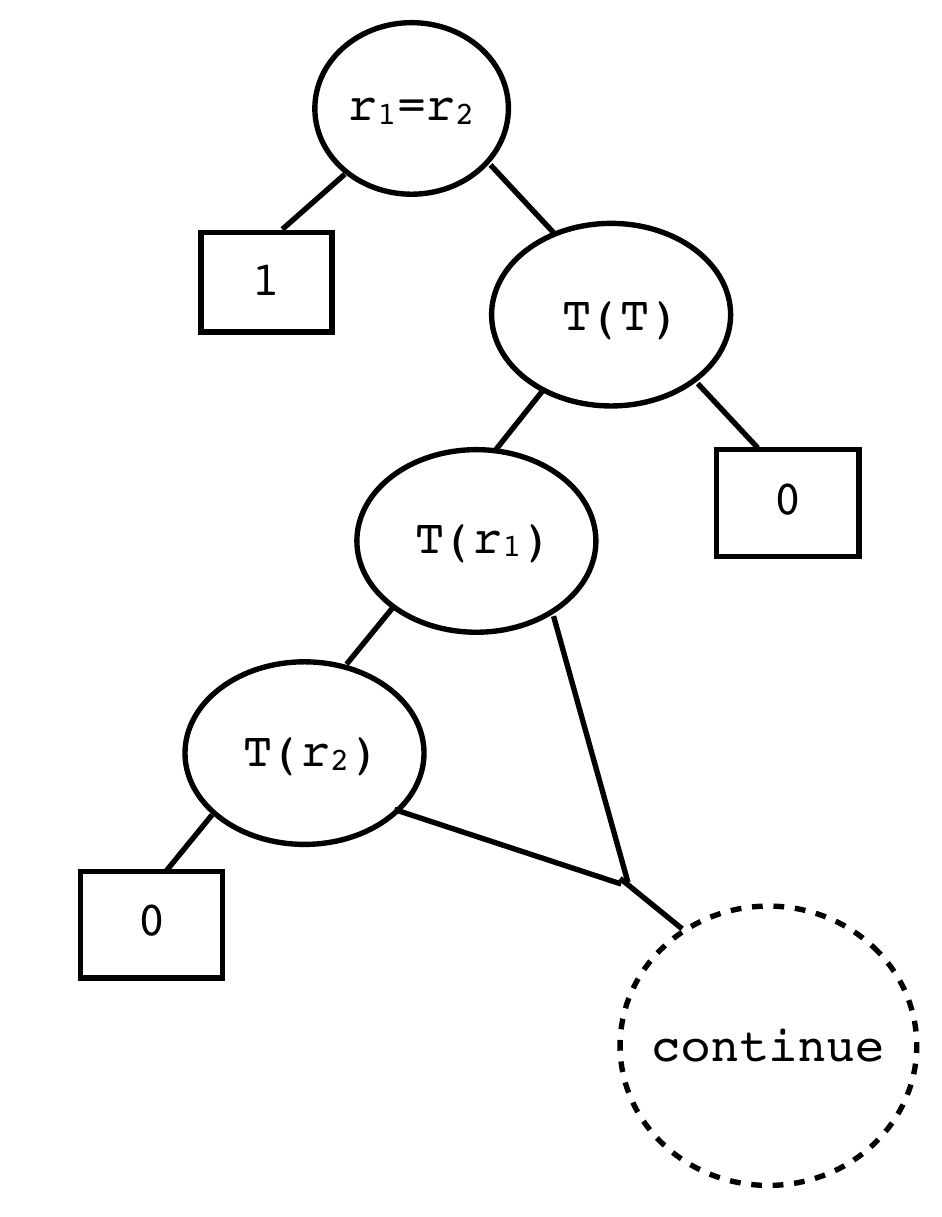}
\caption{GFODD equivalence reduction. The diagram for the unary predicate $T()$ identifies one object for which $T()$ holds, and makes sure that $T()$ does not hold for two objects.}
\label{fig:gfodd-equiv-unary-T}
\end{figure}

\begin{figure}[t]
\centering
\includegraphics[scale = 0.50]{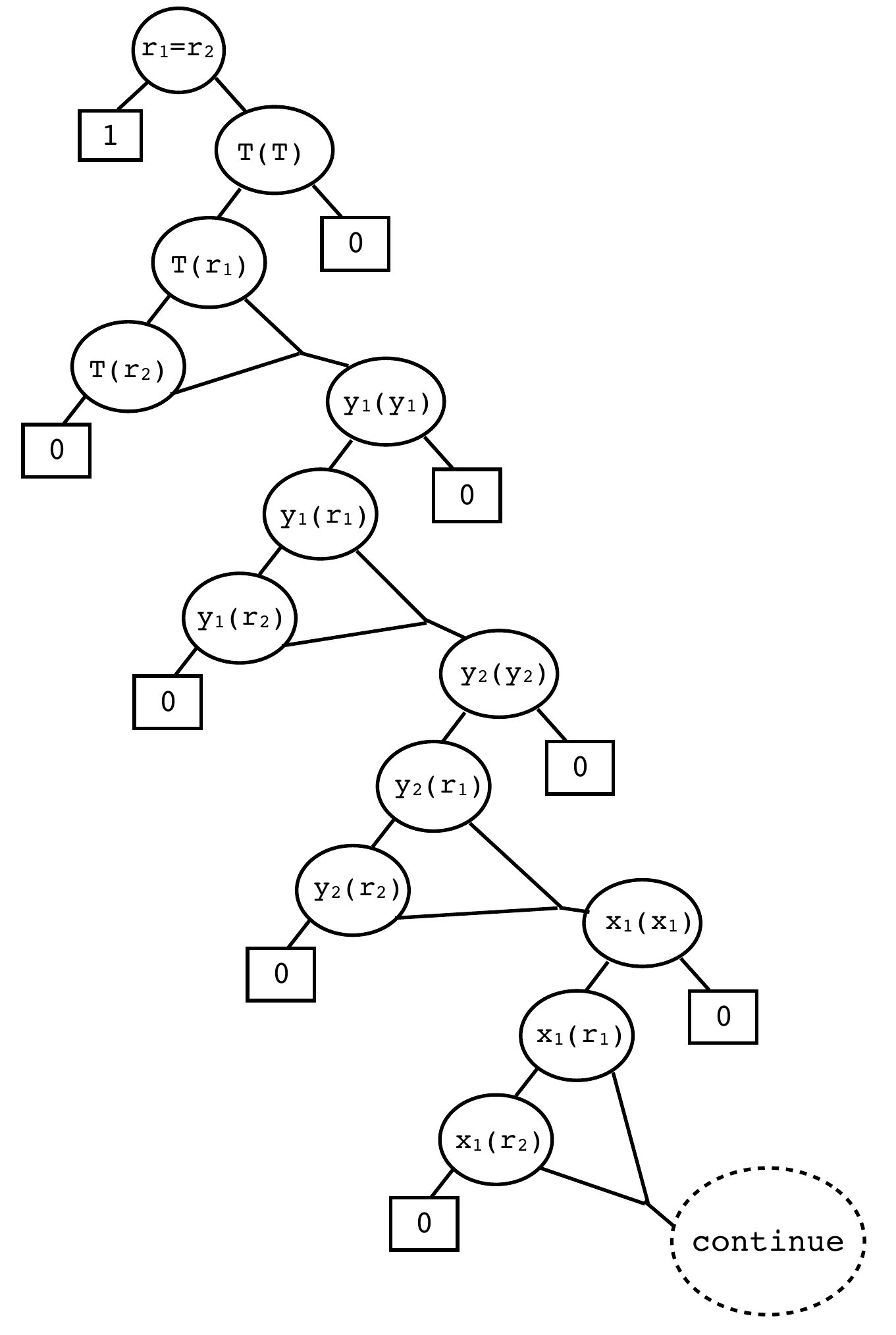}
\caption{GFODD equivalence reduction. The diagram verifying that each unary predicate corresponds to exactly one object in the interpretation.
The corresponding aggregation function is $\max_{T,y_1,y_2,x_1,\ldots, x_{\ell}}, \min_{r_1,r_2}$.
}
\label{fig:gfodd-equiv-sorted-unary}
\end{figure}

\begin{figure}[t]
\centering
\includegraphics[scale = 0.50]{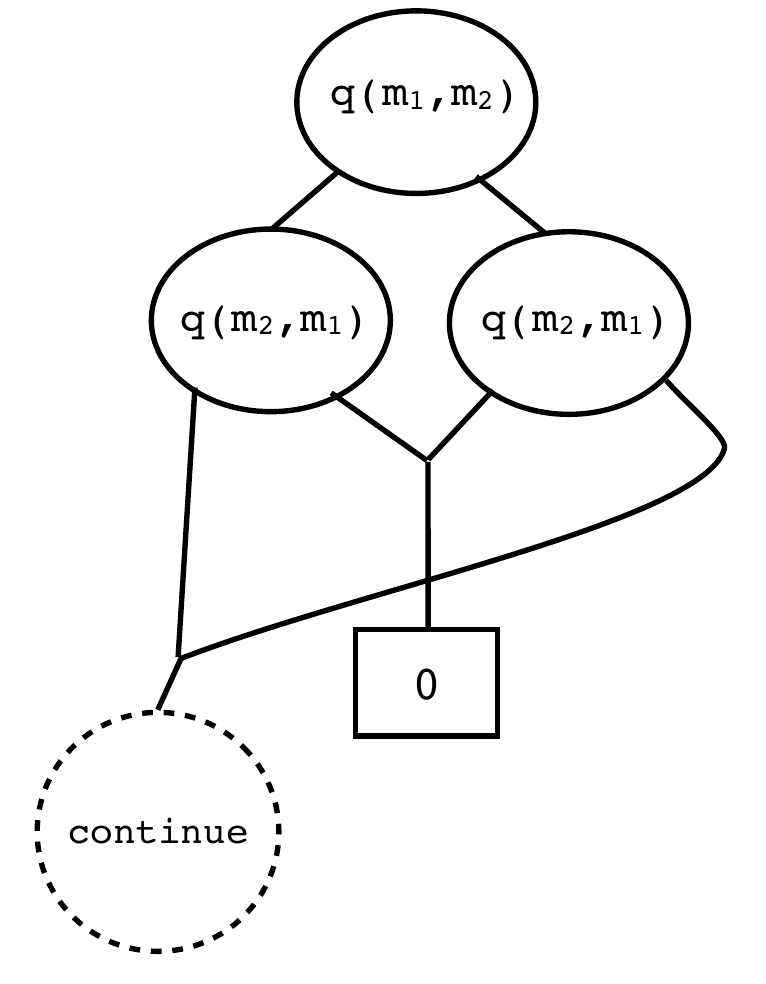}
\caption{GFODD equivalence reduction. Diagram fragment verifying that $q()$ is symmetric.
The corresponding aggregation function is $\min_{m_1,m_2}$.
}
\label{fig:gfodd-equiv-q-symmetric}
\end{figure}

\begin{figure}[t]
\centering
\includegraphics[scale = 0.50]{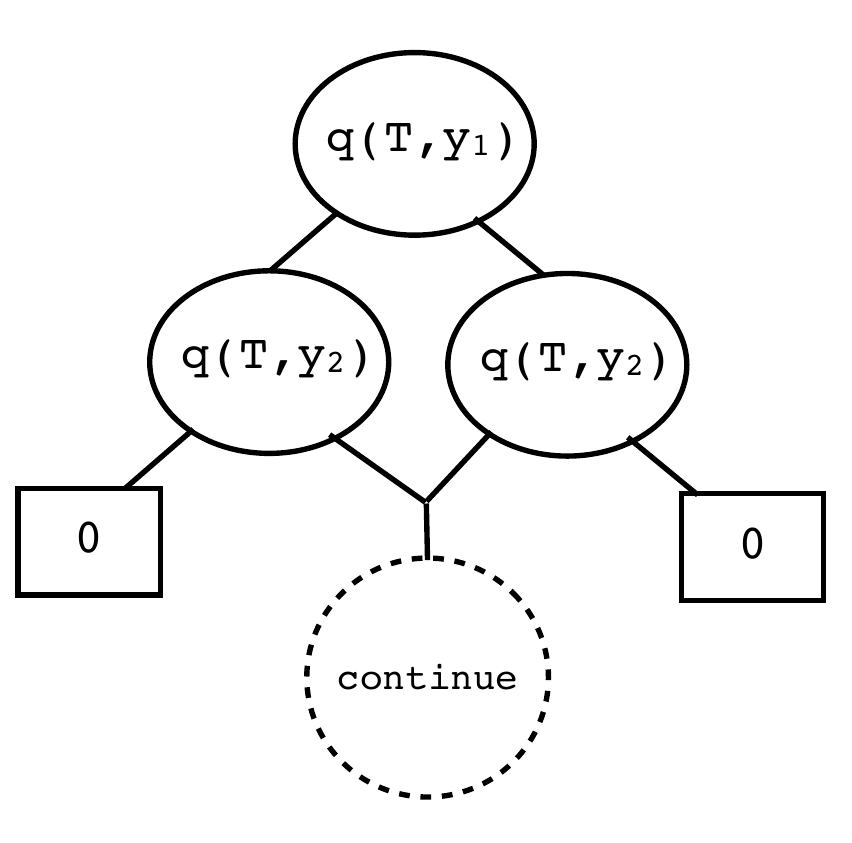}
\caption{GFODD equivalence reduction. Diagram fragment verifying that $q()$ simulates truth values of $P_T()$ from the simple construction.}
\label{fig:gfodd-equiv-q-of-PT}
\end{figure}

\begin{figure}[t]
\centering
\includegraphics[scale = 0.50]{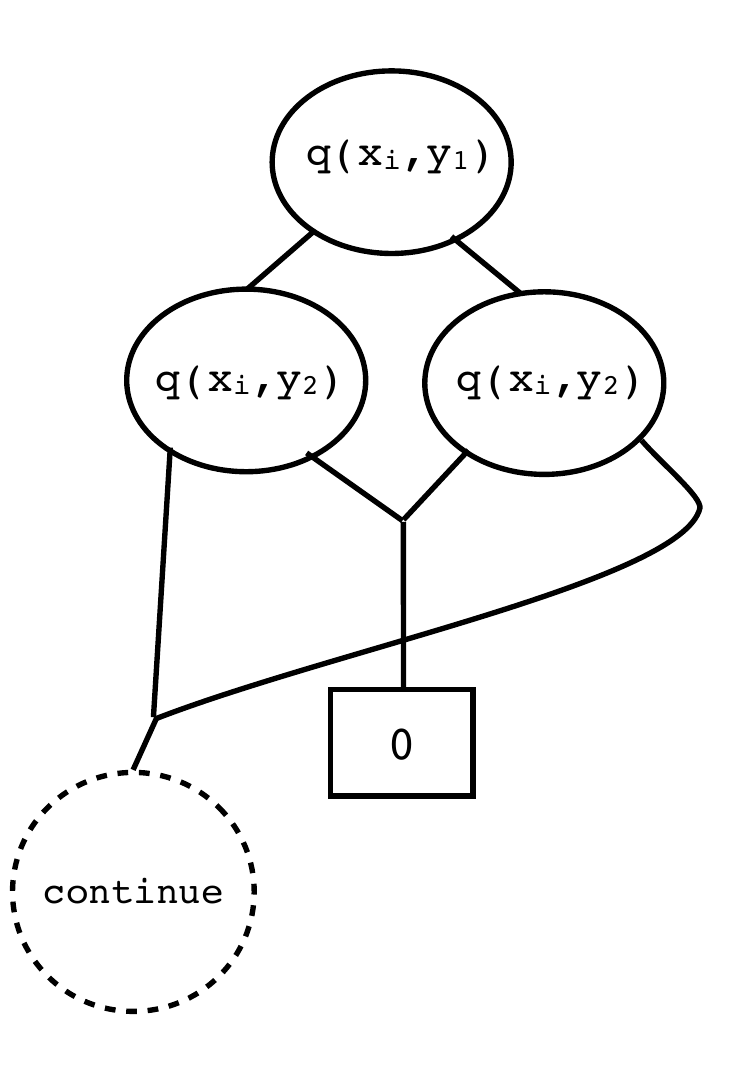}
\caption{GFODD equivalence reduction. Diagram fragment verifying that $q()$ simulates truth values of $P_{x_i}()$ from the simple construction.}
\label{fig:gfodd-equiv-q-of-xi}
\end{figure}

\begin{figure}[t]
\centering
\includegraphics[scale = 0.70]{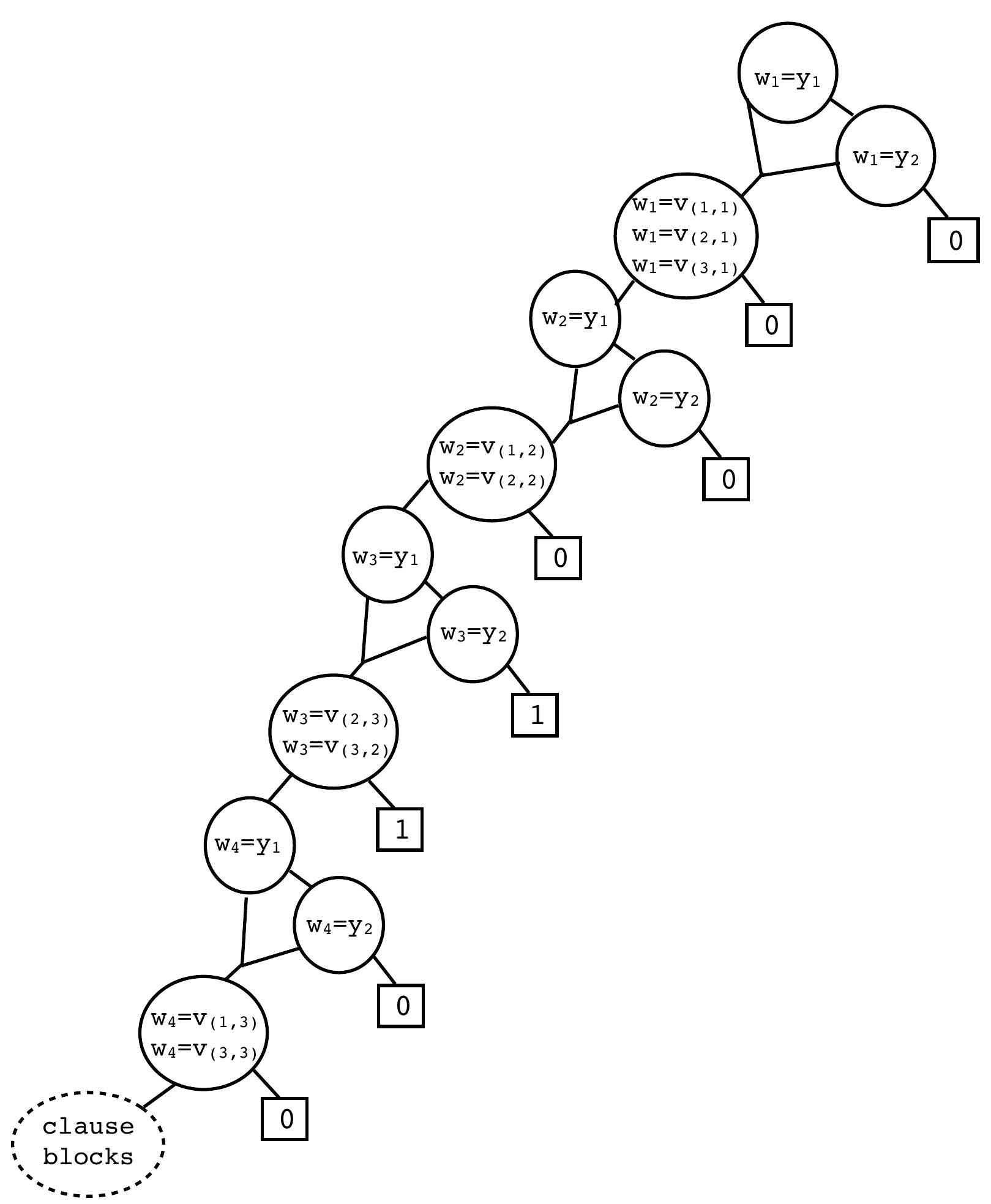}
\caption{The diagram checking block consistency for the ordered version of the GFODD equivalence reduction.}
\label{fig:gfodd-equiv-sorted-consistency}
\end{figure}

\begin{figure}[t]
\centering
\includegraphics[scale = 0.85]{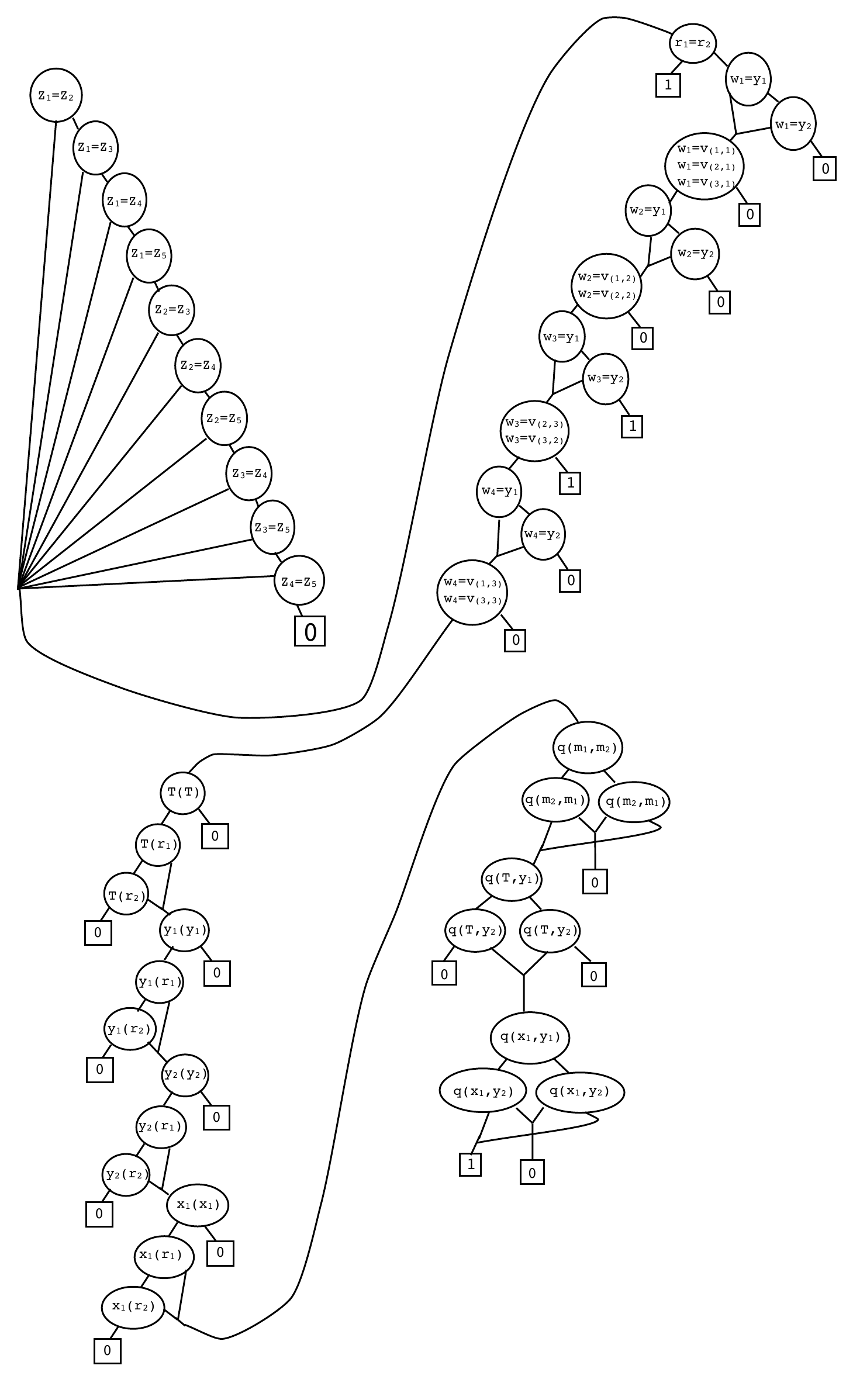}
\caption{The final version of the diagram $B_1$ for the ordered version of the GFODD equivalence reduction.
The complete aggregation function is 
$\max_{T,y_1,y_2,x_1,\ldots,x_{\ell}} $
$\max_{\ww_1,\ww_2}$
$\min_{m_1,m_2}\min_{r_1,r_2}\min_{z_1,\ldots,z_{\ell+4}}$
$\min_{\ww_3} \ldots Q^A_{\ww_k}.
$
}
\label{fig:gfodd-equiv-sorted-B1}
\end{figure}

\clearpage

Finally, we add a component that is not needed for verifying that $I$ is legal but will be useful later when we include the clauses.
In particular, we include a variable consistency block which is similar to the one in the simple construction (see Figure~\ref{fig:gfodd-equiv-consistency}) but where we force each subgroup in $\ww_i$ to bind to the same object as $y_1$ or $y_2$. This is shown in Figure~\ref{fig:gfodd-equiv-sorted-consistency}.

We chain the diagrams together 
as shown in Figure~\ref{fig:gfodd-equiv-sorted-B1} to get $B_1$ where we have moved the node labelled $r_1=r_2$ to be above the block consistency gadget.
Note that the diagram is sorted where for predicate order we have $``=" \prec T \prec y_1 \prec y_2 \prec x_i \prec q$ and for variables we have 
$z_i,T,y_i,x_i \prec r_j \prec w_l$, and $y_i \prec v_{(i_1,i_2)}$, and $m_i \prec T \prec x_j$.
The complete aggregation function is 
$$\max_{T,y_1,y_2,x_1,\ldots,x_{\ell}} 
\max_{\ww_1,\ww_2}
\min_{m_1,m_2}\min_{r_1,r_2}\min_{z_1,\ldots,z_{\ell+4}}
\min_{\ww_3} \ldots Q^A_{\ww_k}.
$$

We next show that the claim C1 holds for the extended reduction.\footnote{
Note that we can remove the variable consistency block which complicates the argument (and does not test anything per legality of $I$) and still maintain correctness of C1. But including it here simplifies the argument for diagram $B$ below and thus simplifies the overall proof.
}
We first consider all possible cases for illegal interpretations. 
\begin{itemize}
\item
If the interpretation has $\geq \ell+4$ objects then the top portion of the diagram yields 0 for some valuation of $z$'s.
Consider any valuation $\zeta_p$ to the prefix of variables
${T,y_1,y_2,x_1,\ldots,x_{\ell}}$, 
${\ww_1,\ww_2}$,
${m_1,m_2}, {r_1,r_2}$,
and any $\zeta$ which is an extension of this valuation, has the violating combination for  
${z_1,\ldots,z_{\ell+4}}$
and any valuation
for the other variables.
We have $\map_{B_1}(I,\zeta)=0$. Therefore, all aggregations from $\ww_k$ to $\ww_3$ yield a value of 0 for this prefix.
Now continuing backwards, the minimization over $z$ yields a value of 0 for $\zeta_p$, and this holds for any $\zeta_p$. Therefore, all remaining aggregations yield 0 and the final value is 0.
\item
Next, consider the case where the interpretation has $\leq \ell+3$ objects but one of the unary predicates is always false (i.e., it does not ``pick" any object).
The situation is similar to the previous case, but here we get a value of 1 for $\zeta$ where $r_1=r_2$ or where there is a block violation for some $\ww_i$ block with $\min$ aggregation.

Consider any valuation $\zeta_p$ to the prefix of variables
${T,y_1,y_2,x_1,\ldots,x_{\ell}}$, 
${\ww_1,\ww_2}$,
${m_1,m_2}$, which is block consistent on ${\ww_1,\ww_2}$
and any $\zeta$ which is an extension of this valuation, has  
$r_1\not = r_2$,
and any valuation
for the other variables.
We have two cases: if $\zeta$ is in group 1 the map is 0 because the unary predicate test fails, and if $\zeta$ is in group 2 the map is the default value of the first violated block $\ww_i$. Now, because there is at least one group 1 valuation, we can argue inductively backwards from $k$ that all aggregations from $k$ to $3$ yield 0, and the same holds for the minimization over $z$. 
Next, because the value is 0 for $r_1\not=r_2$, the minimization over $r$ yields 0 for $\zeta_p$. 
Considering any variants of $\zeta_p$ which are not block consistent on ${\ww_1,\ww_2}$ but otherwise identical we see that their value is also 0. 
As a result all remaining aggregations yield 0.  
\item
Next consider the case where the previous two conditions are satisfied but where one of the unary predicates holds for two or more objects. The argument is identical to the previous case, except that $\zeta$
has  the violating pair for $r_1, r_2$ (instead of any $r_1\not = r_2$).
\item
Next consider any interpretation that has exactly $\ell+3$ objects and where the unary predicates identify the objects corresponding to ${T,y_1,y_2,x_1,\ldots,x_{\ell}}$ but where $q()$ is not symmetric.
In this case we consider a valuation
$\zeta_p$ to the prefix of variables
${T,y_1,y_2,x_1,\ldots,x_{\ell}}$, 
${\ww_1,\ww_2}$,
and extensions $\zeta$ that have the violating pair for $m_1,m_2$.
As above, starting with $\zeta$ that are block consistent on $\ww_1, \ww_2$ and have $r_1\not= r_2$ we can argue that the aggregations down to $m$ yield 0, and as a result that the aggregation over $m$ yields 0. 
As $\zeta_p$ is arbitrary, it follows that the final value is 0. 
\item
The only remaining cases are
interpretations that are illegal only because the $q()$ simulation of $P_T()$ or $P_{x_i}()$ is not as required. In this case, the same argument as in the 2nd item in this list shows that the value is 0. 
\end{itemize}

\begin{figure}[t]
\centering
\includegraphics[scale = 0.70]{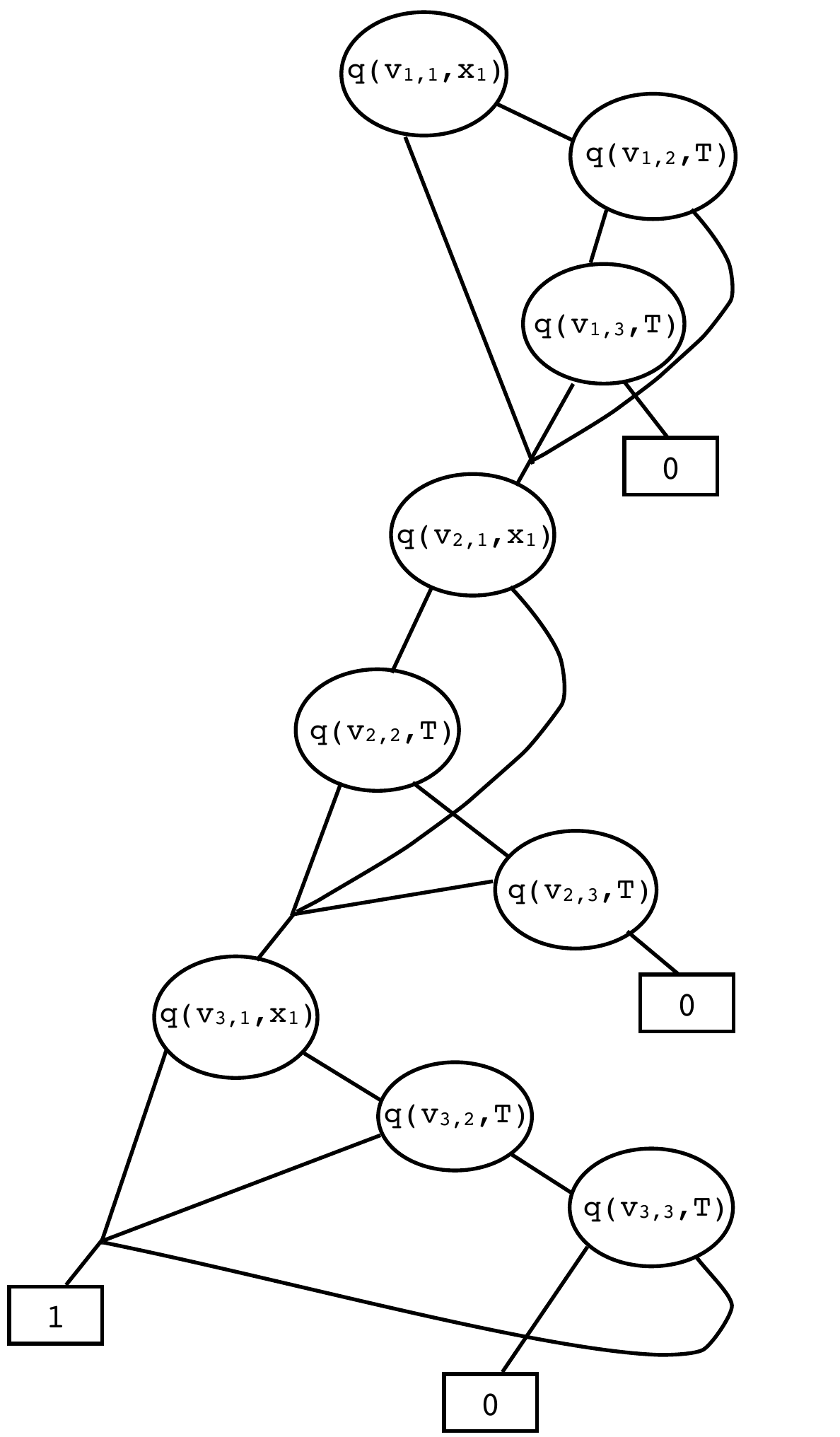}
\caption{GFODD equivalence reduction. Diagrams capturing the clause blocks.}
\label{fig:gfodd-equiv-sorted-clauses}
\end{figure}

\begin{figure}[t]
\centering
\includegraphics[scale = 0.80]{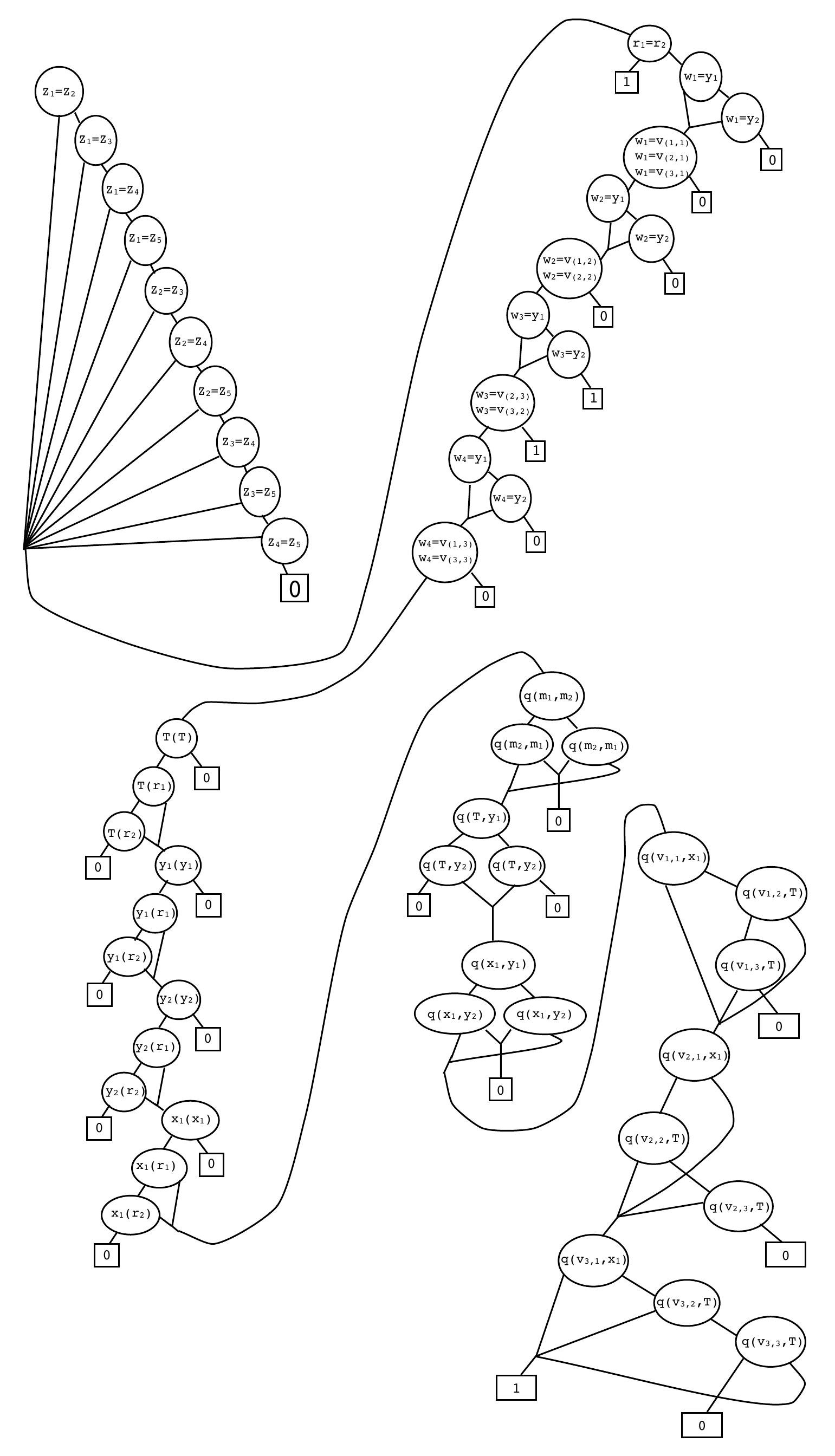}
\caption{The final version of the diagram $B$ for the ordered version of the GFODD equivalence reduction.
The complete aggregation function is 
$\max_{T,y_1,y_2,x_1,\ldots,x_{\ell}} $
$\max_{\ww_1,\ww_2}$
$\min_{m_1,m_2}\min_{r_1,r_2}\min_{z_1,\ldots,z_{\ell+4}}$
$\min_{\ww_3} \ldots Q^A_{\ww_k}.
$}
\label{fig:gfodd-equiv-sorted-B}
\end{figure}

\clearpage

Therefore, if $I$ is illegal then $\map_{B_1}(I)=0$.
Consider next any legal interpretation and the intended valuation $\zeta_p$ to $T,y_1,y_2,x_1,\ldots,x_{\ell}$. For any group 1 extension of this valuation and any valuation of the variables $m,z,r$ the diagram yields 1. Therefore, we can argue inductively that all $\ww_i$ aggregations down to $\ww_3$ yield 1, and the min aggregations over $m,z,r$ yield 1 for $\zeta_p$. Therefore, the max aggregation over $T,y_1,y_2,x_1,\ldots,x_{\ell}$ yields 1, and $\map_{B_1}(I)=1$.
This completes the proof of C1.

The diagram $B$ is obtained by adding the clause blocks below the $q()$ tests of $B_1$.
The clause blocks have the same structure as above but they use $q()$ instead of $P_T()$ and $P_{x_i}()$. 
This is shown in Figure~\ref{fig:gfodd-equiv-sorted-clauses}.
The final diagram for $B$ is shown in Figure~\ref{fig:gfodd-equiv-sorted-B} and it has the same aggregation function as $B_1$. 
Note that the diagram is also sorted using the same order as in $B_1$ with the addition that $x_i \prec V_{(i_1,i_2)}$. We next show that: 
\\
(C2') $\map_B(I)=1$ if and only if 
$I$ is legal and $Q_2\xx_2 \ldots Q_k\xx_k f((\xx_1=\alpha),\xx_2,$ $\ldots,$ $\xx_k) = 1$
where $I$
embeds the substitution $\xx_1=\alpha$. 

The property C2' is slightly different from C2 above, in that it argues directly over $B$ and includes the condition on legal interpretations. Nonetheless, the same argument from the simple case can be used to show that C1 and C2' imply that $B_1$ and $B$ are equivalent iff the QBF is satisfied.

To prove the claim first note that by the construction $B$ adds more tests on the path to a 1 leaf of $B_1$ and does not add any other paths to a value of 1. Therefore, 
for any $I$ and any $\zeta$, if $\map_B(I,\zeta)=1$ then $\map_{B_1}(I,\zeta)=1$ and as a result if $\map_B(I)=1$ then $\map_{B_1}(I)=1$. Therefore, by C1, if $\map_B(I,\zeta)=1$ then $I$ is legal.
Next, consider any legal $I$ and any unintended valuation $\zeta_p$ for the prefix $T,y_1,y_2,x_1,\ldots,x_{\ell}$.
Any valuation extending this prefix with a block consistent assignment for all $\ww_i$ and with $r_1\not = r_2$ will yield a zero. Therefore, the minimization over $r_1,r_2$ will yield 0 for this prefix, and so will the minimization over $m_1,m_2$ and maximization over $\ww_1,\ww_2$.
We conclude that the aggregated value for this prefix is 0. 
Therefore, if $\map_B(I)=1$ and thus the $\max$ aggregation over these prefix variables yields 1, it must be through the intended valuation $\zeta_p$ for the prefix $T,y_1,y_2,x_1,\ldots,x_{\ell}$. 

However, 
when $T,y_1,y_2,x_1,\ldots,x_{\ell}$ are fixed to their intended values,
the portions of the diagram testing for $\leq q+3$ objects, the uniqueness of special objects $T,y_1,y_2,x_1,\ldots,x_k$, the symmetry of $q()$ and its simulation of $P_T()$ and $P_{x_i}$
do not affect the final value in the sense that a valuation reaching them always continues to the next block. 
Therefore, if we restrict attention to such valuations we can shrink $B$ by removing these blocks and still obtain the same aggregated value. Now we observe that there is a 1-1 correspondence between valuations and values of the resulting $B$ to valuations and values of $B_2$ in the simple construction
(where we extend the notion of block consistent to enforce that $\ww_i$ bind to $y_1,y_2$). Therefore,
the claim holds by C2 of the simple construction. 
\end{proof}

\bigskip
The result for GFODD Value is similar to the FODD case.

\begin{theorem}
{ GFODD Value} for diagrams with aggregation depth $k$  (where $k\geq 2$)  is $\Sigma_{k+1}^p$-complete.
\end{theorem}
\begin{proof}
The proof of Theorem~\ref{thm:maxfodd-value} goes through almost directly and requires only a slight wording variation. For membership we get the bound on interpretation size by the assumption on the input; then the algorithm is the same.

For the reduction, we use Theorem~\ref{thm:gfodd-combine} to calculate $B=apply(B_1,B_2,+)$. 
As stated in that theorem,
we can mesh together the aggregation lists of $B_1$ and $B_2$ interleaving the max and min blocks from each diagram without increasing quantifier depth and the diagram $B$ has the same quantifier prefix and depth as those of $B_1$ and $B_2$. 
\end{proof}

\bigskip
Unlike $\max$-$k$-alternating GFODDs, for min diagrams the search for a satisfying interpretation cannot be absorbed into the first aggregation operator. This fact pushes the problem one level higher in the hierarchy.

\begin{theorem}
\label{thm:gfodd-min-sat}
GFODD  Satisfiability for $\min$-$k$-alternating GFODDs  (where $k\geq 2$) is $\Sigma_{k+1}^p$-complete.
\end{theorem}
\begin{sketch}
For membership, we guess an interpretation $I$ of the appropriate size, and then appeal to a $\Sigma_k^p$ oracle to solve GFODD evaluation for $(B,I)$.
This is clearly in $NP^{\Sigma_k^p}$.

The hardness result uses a slight modification of the equivalence proof, which we sketch next. One can verify 
that all the details of the modification go through to establish the result.

In particular, we reduce QBF satisfiability with $k\geq 3$ alternations of quantifiers to satisfiability of $\min$-$(k-1)$-alternating GFODDs. 
Here we assume a QBF whose first quantifier is $\exists$,  that is has the form $\exists_1\xx_1 \ldots Q_k\xx_k f(\xx_1,\xx_2,\ldots,\xx_k)$ where $\xx_i$ refers to a set of variables. We build $B_1$, $B_2$ and $B$ exactly as
in Theorem~\ref{thm:gfodd-equiv} with one exception: the leaf values on the diagram that checks for block consistency are flipped from the previous construction (because the corresponding aggregation operators are switched). The reduction still provides 
 $B_1$, $B_2$, and $B=apply(B_1,B_2,\wedge)$ such that the following claims hold:
\\ 
(C1) for all $I$, $\map_{B_1}(I)=1$ if and only if $I$ is legal.
\\
(C2)
if $I$ is legal and it embeds the substitution $\xx_1=\alpha$ then $\map_{B_2}(I)=1$ if and only if $Q_2\xx_2 \ldots Q_k\xx_k f((\xx_1=\alpha),\xx_2,\ldots,\xx_k) = 1$.

We then output the diagram $B$ for GFODD satisfiability. 
Now, if the QBF is satisfied then there exists a value $\alpha$ such that for $\xx_1=\alpha$ we have that $Q_2\xx_2 \ldots Q_k\xx_k f((\xx_1=\alpha),\xx_2,\ldots,\xx_k) = 1$. Therefore, by C2, for the legal $I$ that embeds $\alpha$, 
$\map_{B_2}(I)=1$,
and by Theorem~\ref{thm:gfodd-combine} we have that $\map_{B}(I)=1$.
On the other hand, if the QBF is not satisfied then for all substitutions $\xx_1=\alpha$ we have $Q_2\xx_2 \ldots Q_k\xx_k f((\xx_1=\alpha),\xx_2,\ldots,\xx_k) = 0$. Therefore, by C2, 
all legal $I$ (and any $\alpha$ they embed)
$\map_{B_2}(I)=0$ and by Theorem~\ref{thm:gfodd-combine}
we also have $\map_{B}(I)=0$. By C1, $\map_{B}(I)=0$ for non-legal interpretations. Therefore, $B$ is not satisfiable.

The construction to handle ordering and edge removal structure can be similarly modified for the current proof. 
\end{sketch}

\bigskip
Up to this point, all hardness proof in the paper
use a signature without any constants, i.e., we use equality and unary and binary predicates. 
For min-GFODDs (the case $k=1$) the use of constants affects the complexity of the satisfiability problem. In particular, for a signature without constants, if a min-GFODD is satisfied by interpretation $I$, then it is satisfied by the sub-interpretation of $I$ with just one object (any object in $I$ will do). Moreover, given diagram $B$ and a specific $I$ with one object, model evaluation is in $P$ because there is only one valuation to consider. Therefore, in this case satisfiability is in NP: we can guess the interpretation (i.e. truth values of predicates) and evaluate $\map_B(I)$ in polynomial time. On the other hand, if we allow constants in the signature the problem follows the same scheme as above and is $\Sigma_{2}^p$-complete.

\begin{theorem}
\label{thm:gfodd-min-sat-k1}
GFODD  Satisfiability for $\min$-GFODDs  is $\Sigma_{2}^p$-complete.
\end{theorem}
\begin{sketch}
Membership is as in the general case. For hardness, we use the construction in the reduction of the previous proof which yields a GFODD with aggregation $ \min^* \max^*$ (i.e., the portion starting with $\ww_3$ does not exist) where the max variables are $T,y_1,y_2,x_1,\ldots,x_{\ell}$. We then turn these variables into constants and remove the $\max$ aggregation to yield a $\min$ GFODD. One can verify that  the proof of Theorem~\ref{thm:gfodd-min-sat} goes through with very minor changes.
\end{sketch}

%% file: minkGfoddSatProof.tex
\subsection{Proof of Theorem~\ref{thm:gfodd-min-sat}}

\begin{proof}
We reduce QBF satisfiability with $k\geq 3$ alternations of quantifiers to satisfiability of $\min$-$(k-1)$-alternating GFODDs. 
The reduction borrows most of the construction from the proof of Theorem~\ref{thm:gfodd-equiv}, swapping between min and max aggregation to adjust it to the new context. 

In particular, here we assume a QBF whose first quantifier is $\exists$,  that is,
$\exists x_1, Q_2x_2 \ldots Q_mx_m$ $f(x_1,x_2,\ldots,x_m)$ where this form has 
$k$ blocks of quantifiers. To simplify the notation it is convenient to group adjacent variables having the same quantifiers into groups so that the QBF 
has the form
$Q_1\xx_1 \ldots Q_k\xx_k f(\xx_1,\xx_2,\ldots,\xx_k)$ where $\xx_i$ refers to a set of variables.

We next define a notion of ``legal interpretations" for our diagrams. 
A legal interpretation embeds the binary interpretation $I^*$ from previous proofs and in addition includes a truth setting for all the variables in the first $\exists$ block of the QBF. 
The reduction constructs diagrams $B_1$, $B_2$, and $B=apply(B_1,B_2,\wedge)$ such that the following claims hold:
\\ 
(C1) for all $I$, $\map_{B_1}(I)=1$ if and only if $I$ is legal.
\\
(C2)
if $I$ is legal and it embeds the substitution $\xx_1=\alpha$ then $\map_{B_2}(I)=1$ if and only if $Q_2\xx_2 \ldots Q_k\xx_k f((\xx_1=\alpha),\xx_2,\ldots,\xx_k) = 1$.

We then output the diagram $B$ for GFODD satisfiability. 
Now, if the QBF is satisfied then there exists a value $\alpha$ such that for $\xx_1=\alpha$ we have that $Q_2\xx_2 \ldots Q_k\xx_k f((\xx_1=\alpha),\xx_2,\ldots,\xx_k) = 1$. Therefore, by C2, for the legal $I$ that embeds $\alpha$, 
$\map_{B_2}(I)=1$. 

On the other hand, if the QBF is not satisfied then for all substitutions $\xx_1=\alpha$ we have $Q_2\xx_2 \ldots Q_k\xx_k f((\xx_1=\alpha),\xx_2,\ldots,\xx_k) = 0$. Therefore, by C2, 
all legal $I$ (and any $\alpha$ they embed)
$\map_{B_2}(I)=0$ and by Theorem~\ref{thm:gfodd-combine}
we also have $\map_{B}(I)=0$. By C1, $\map_{B}(I)=0$ for non-legal interpretations. Therefore, $B$ is not satisfiable.

We now proceed with the reduction, starting first with a simplified construction ignoring ordering, and then elaborating to enforce these constraints. 
The set of predicates includes $P_T()$ and for every QBF variable $x_i$ in the first $\exists$ block we use a predicates $P_{x_i}()$. 
Notice that each $x_i$ is a member of $\xx_1$ (the first $\exists$ group) where the typeface distinguishes the individual variables in the first block, from blocks of variables.
In the simplified construction, a legal interpretation has exactly two objects, say $a$ and $b$, where $P_T(a)\not = P_T(b)$ and where for each $P_{x_i}()$ we have $P_{x_i}(a) = P_{x_i}(b)$. That is, the assignment of an object to $v$ in $P_T(v)$ simulates an assignment to Boolean values, but the truth value of $P_{x_i}(v)$ is the same regardless of which object is assigned to $v$. 

Consider our example QBF modified to start with $\exists$ quantifier 
$\exists x_1 \forall x_2 \exists x_3 \forall x_4 (x_1 \vee \bar{x_2}  \vee x_4) \wedge (\bar{x_1} \vee x_2 \vee x_3) \wedge (x_1 \vee x_3 \vee \bar{x_4})$.
The first block includes only the variable $x_1$ and the following interpretation is legal:
$I = \{ [a, b], P_T(a)=$true, $P_T(b)=$false, $P_{x_1}(a) = P_{x_1}(b)=$false$\}$.

The diagram $B_1$ has three portions where the first two are exactly as in the previous proof, thus verifying that $I$ has two objects and that $P_T()$ behaves as stated. The third portion verifies that each $P_{x_i}()$ behaves as stated, where we use a sequence of blocks, one for each $P_{x_i}()$.
The combined diagram $B_1$ is shown in Figure~\ref{fig:gfodd-equiv-B1} and the aggregation function is $\min_{z_1,z_2,z_3},\max_{y_1,y_2}$.

To see that C1 holds consider all possible cases for non-legal interpretations. If $I$ has at most one object the map is 0 for all valuations and thus the aggregation is 0. If $I$ has at least 3 objects, then the min aggregation over $z$ yields 0. If $I$ has 2 objects but it violates the condition on $P_T$ or $P_{x_i}$ then again the map is 0 for any valuation and the aggregation is 0. On the other hand, if $I$ is legal, then for any assignment to $z$, the correct mapping to $y_1,y_2$ yields 1. Therefore the aggregation over $z$ yields 1. 

The diagram $B_2$ is constructed by modifying $B_2$ from the proof of Theorem~\ref{thm:gfodd-eval}. The first modification is to 
handle the first $\exists$ block differently.  
As it turns out, all we need to do is replace the $\max$ aggregation for the $\ww_1$ block with $\min$ aggregation and accordingly replace the default value on that block to 1. The modified variable consistency diagram is shown in [we flip leaf values in diagram] Figure~\ref{fig:gfodd-equiv-consistency}. 
The clause blocks have the same structure as in the previous construction but use $P_{x_i}(V_{(i_1,i_2)})$ when $x_i$ is a $\exists$ variable from the first block and use $P_{T}(V_{(i_1,i_2)})$ otherwise.
This is shown in Figure~\ref{fig:gfodd-equiv-clauses}. 
$B_2$ includes the variable consistency blocks followed by the clause blocks. Note that the new clause blocks are not sorted in any consistent order because the predicates $P_{x_i}()$ and $P_{T}()$ appear in an arbitrary ordering determined by the appearance of literals in the QBF. Other than this violation, all other portions of the diagrams described are sorted where the predicate order has $= \prec P_T \prec P_{x_i}$ and where variables $w_i$ are before $V_{(i_1,i_2)}$ and variables within group are sorted lexicographically. 
The combined aggregation function is $\min_{\ww_1}, \min_{\ww_2},\max_{\ww_3},\ldots,Q^A_{\ww_k}$.

We next show that claim C2 holds, which will complete the proof of the simplified construction. 
Consider any legal $I$, let the corresponding truth values for variables in $\xx_1$ be denoted $\alpha$, and consider valuations for the QBF extending $\xx_1=\alpha$. Now consider any valuation $v$ to the remaining variables in the QBF and the induced substitution to the GFODD variables $\zeta(v)$ that is easily identified from the construction. Add any consistent group assignment to $\ww_1$ (that is, we assign $a$ or $b$ to all variables in that group)
to $\zeta(v)$ to get $\hat{\zeta}(v)$.
By the construction of $B_2$ we have that $f([\xx_1=\alpha ,(\xx_2,\ldots,\xx_k)=v]) = \map_{B_2}(I,\hat{\zeta}(v))$. To see this note that there are no quantifiers in this expression, there is a 1-1 correspondence between the valuations of $\xx_2,\ldots,\xx_k$ and $\ww_2,\ldots,\ww_k$, and that as long as the assignment to the $\ww_1$ block is group consistent it does not affect the value returned. 
We call this set of valuations, that arise as translations of substitutions for QBF variables, {\em Group 1}.

The second group, {\em Group 2}, includes valuations that do not arise as $\hat{\zeta}(v)$ and therefore they violate at least one of the consistency blocks. 
Let  $\zeta$ be such a valuation
and let $Q^A_j$ be the first block from the left whose constraint is violated. By the construction of $B_2$, in particular the order of equality along paths in the GFODD, we have that the evaluation of the diagram on $\zeta$ ``exits" to a default value on the first violation. Therefore, if $j=1$, that is the violation is in the block of $\ww_1$ $\map_{B}(I,\zeta)=1$ and for $j\geq 2$ if
$Q_j$ is a $\forall$ then $\map_{B}(I,\zeta)=1$ and if $Q_j$ is a $\exists$ then $\map_{B}(I,\zeta)=0$.

We can now show the correspondence in truth values. Consider any partition of the blocks $2,\ldots, k$ into a prefix $2,\ldots, j$ and remainder $(j+1),\ldots, k$, and any valuation $v$ to the prefix blocks. We claim that for all such partitions 
\begin{eqnarray*}
& & 
Q_{j+1} \xx_{j+1}, \ldots, Q_k \xx_k, f(\xx_1=\alpha,(\xx_2,\ldots,\xx_j)=v, \xx_{j+1},\ldots,\xx_k)
\\
& = &
Q_{j+1} \ww_{j+1}, \ldots, Q_k \ww_k, \map_{B_2}(I,[\ww_1=a,(\ww_2,\ldots,\ww_j)=\zeta(v),(\ww_{j+1},\ldots,\ww_k)]).
\end{eqnarray*}
Note that when $j=1$, that is, the prefix is empty, this yields that $Q_{2} \xx_{2}, \ldots, Q_k \xx_k, f(\xx_1=\alpha,\xx_2,\ldots,\xx_k)$ is equal to 
$Q^A_{2} \ww_{2}, \ldots, Q^A_k \ww_k, \map_{B_2}(I,[\ww_1=a,\ww_2,\ldots,\ww_k)])$.
Now because the default value for violations of $\ww_1$ is 1 and because the aggregation for $\ww_1$ is $\min$ the latter expression is equal to 
$Q^A_{1} \ww_{1}, \ldots, Q^A_k \ww_k, \map_{B_2}(I,[\ww_1,\ww_2,\ldots,\ww_k)])$.
This means that $\map_{B_2}(I)$ is equal to 
$Q_{2} \xx_{2}, \ldots, Q_k \xx_k, f(\xx_1=\alpha,\xx_2,\ldots,\xx_k)$
completing the proof of C2.

As above, we prove the claim by induction, backwards from $k$ to 1. 
For the base case, $j=k-1$ and the second part includes only one block. 
Consider any concrete substitution for $\ww_k$ participating in the equation. 
If the substitution is in group 2, then the map is 1 if $Q_j$ is $\forall$ and is 0 otherwise. 
Therefore, the value of $Q^A_k \ww_k, \map_{B_2}(I,[\ww_1=a,(\ww_2,\ldots,\ww_{k-1})=\zeta(v),(\ww_k)])$ is the same as that value when restricted to substitutions in group 1. 
But, as argued above, for group 1 this is exactly the same value returned by the QBF.

For the inductive step, the valuation $v$ covers the first $j-1$ blocks. 
Note that, by the inductive assumption, for any group 1 substitution $v_j$ for $\xx_j$ and corresponding, $\zeta(v_j)$ for $\ww_j$,
$
Q_{j+1} \xx_{j+1}, \ldots, Q_k \xx_k,$ $ f(\xx_1=\alpha,(\xx_2,\ldots,\xx_{j-1})=v, (\xx_{j}=v_j),\xx_{j+1},\ldots,\xx_k)$
 $= Q^A_{j+1} \ww_{j+1}, \ldots, Q^A_k \ww_k,$ $\map_{B_2}(I,[\ww_1=a,(\ww_2,\ldots,$ $\ww_{j-1})=\zeta(v),(\ww_{j}=\zeta(v_j)),(\ww_{j+1},\ldots,\ww_k)])
$.
On the other hand, for any group 2 substitution $\zeta_j$ for $\ww_j$ and any values for $(\ww_{j+1},\ldots,\ww_k))$ we have that
the violating block for the corresponding combined $\zeta$ is block $j$ and therefore
$\map_{B_2}(I,[\ww_1=a,(\ww_2,\ldots,\ww_{j-1})=\zeta(v),(\ww_{j}=\zeta_j),(\ww_{j+1},\ldots,\ww_k)])$ gets the default value for that block. 
Therefore, as in the base case, the map is determined by group 1 valuations, which are in turn identical to the QBF value and 
$
Q_{j} \xx_{j}, \ldots, Q_k \xx_k, f(\xx_1=\alpha,(\xx_2,\ldots,\xx_{j-1})=v, (\xx_{j},\ldots,\xx_k))
= Q^A_{j} \ww_{j}, \ldots, Q^A_k \ww_k, \map_{B_2}(I,[\ww_1=a,(\ww_2,\ldots,\ww_{j-1})=\zeta(v),(\ww_{j},\ldots,\ww_k)])
$ as required. Therefore, the claim on the correspondence of values holds, and as argued above this completes the proof of C2.

\paragraph{Extending the reduction to handle ordering:}
The main idea in the extended construction is to replace the unary predicates $P_T$ and $P_{x_i}$ with one binary predicate $q(\cdot,\cdot)$ where the ``second argument" in $q()$ serves to identify the corresponding predicate and hence its truth value. In addition we force $q()$ to be symmetric so that for any $A$ and $B$ the truth value of $q(A,B)$ is the same as the truth value of $q(B,A)$. In this way we have freedom to use either $q(A,B)$ or $q(B,A)$ as the node label which provides sufficient flexibility to handle the ordering issues. To implement this idea we need a few additional constructions.

Let the set of variables in the first $\exists$ block of the QBF be $x_1,x_2,\ldots,x_\ell$. 
The set of predicates in the extended reduction includes
unary predicates $T(),  y_1(), y_2(), x_1(), x_2(),\ldots,x_\ell(), $ and one binary predicate $q(\cdot,\cdot)$. A legal interpretation includes exactly $\ell+3$  objects which are uniquely identified by the unary predicates. We therefore slightly abuse notation and use the same symbols for the objects and predicates. 
In particular, the  atoms $y_1(y_1), y_2(y_2)$, $T(T)$, and $x_1(x_1), x_2(x_2), \ldots, x_\ell(x_\ell)$ are true in the interpretation and only these atoms are true for these unary predicates (e.g., $x_1(T)$ is false).
The truth values of $q()$ reflect the simulation of $P_T()$ and $P_{x_i}()$ in addition to being symmetric. 
Thus  the truth values of $q(y_1,T)$ and $q(T,y_1)$ are the same and they are the negation of the truth values of $q(y_2,T)$ and $q(T,y_2)$. 
For all $i$, the truth values of $q(y_1,x_i)$, $q(x_i,y_1)$, $q(y_2,x_i)$ and $q(x_i,y_2)$ are the same.
The truth values of other instances of $q()$, for example, $q(x_2,T)$, can be set arbitrarily. For example, the following interpretation is legal when $\ell=1$:
$I = \{ [a, b, c, d], y_1(a), y_2(b), T(c), x_1(d), q(c,a)=q(a,c)=$true, $q(c,b)=q(b,c)=$false, $q(d,a)=q(a,d)=q(d,b)=q(b,d)=$false, $q(\cdot,\cdot)=$false$\}$ where $q(\cdot,\cdot)$ refers to any instance not explicitly mentioned in the list.

We next define the diagram $B_1$ that is satisfied only in legal interpretations. 
We enforce exactly $\ell+3$ objects using two complementary parts. The first includes ${\ell+4 \choose 2}$ inequalities on a new set of $\ell+4$ variables $z_1,\ldots,z_{\ell+4}$ with min aggregation. If we identify $\ell+4$ distinct objects we set the value to 0.
This is shown in Figure~\ref{fig:gfodd-equiv-few-objects}.

To enforce at least $\ell+3$ objects and identify them we use the following gadget. For each of the unary predicates we have a diagram identifying its object and testing its uniqueness where we use both max and min variables. This is shown for the predicate $T()$ in Figure~\ref{fig:gfodd-equiv-unary-T}. The node $T(T)$ with max variable $T$ identifies the object $T$. The nodes $T(r_1),T(r_2)$ with min variables $r_1,r_2$ make sure that $T$ holds for at most one object. 
We chain the diagrams together as shown in Figure~\ref{fig:gfodd-equiv-sorted-unary} where the variables $r_1,r_2$ are shared among all unary predicates. The corresponding aggregation function is $ \min_{r_1,r_2}, \max_{T,y_1,y_2,x_1,\ldots, x_{\ell}}$. 
This diagram associates each of the $\ell+3$ objects with one of the unary predicates and in this way provides a reference to specific objects in the interpretation.

The symmetry gadget for $q()$ is shown in Figure~\ref{fig:gfodd-equiv-q-symmetric} where the variables $m_1,m_2$ are min variables. If an input interpretation has two objects $A,B$ where $q(A,B)$ has a truth value different than $q(B,A)$ then minimum aggregation will map the interpretation to 0.

The truth value gadget for the simulation of $P_T$ is shown in Figure~\ref{fig:gfodd-equiv-q-of-PT} and  
the truth value gadget for the simulation of $P_{x_i}$ is shown in Figure~\ref{fig:gfodd-equiv-q-of-xi}. 
These diagram fragments refer to variables in other portions and they will be connected and aggregated together.

Finally, we add a component that is not needed for verifying that $I$ is legal but will be useful later when we include the clauses.
In particular, we include a variable consistency block which is similar to the one in the simple construction but where we force $\ww_i$ to bind to the same object as $y_1$ or $y_2$. This is shown in [we flip leaf values in diagram] Figure~\ref{fig:gfodd-equiv-sorted-consistency}.

We chain the diagrams together 
as shown in Figure~\ref{fig:gfodd-equiv-sorted-B1} to get $B_1$ where we have moved the node labelled $r_1=r_2$ to be above the block consistency gadget.
[here too leaf values in consistency block need to be flipped]
Note that the diagram is sorted where for predicate order we have $= \prec T \prec y_1 \prec y_2 \prec x_i \prec q$ and for variables we have 
$z_i,T,y_i,x_i \prec r_j \prec w_l$, and $y_i \prec V_{(i_1,i_2)}$, and $m_i \prec T \prec x_j$.
The complete aggregation function is 
$$
\min_{m_1,m_2}\min_{r_1,r_2}\min_{z_1,\ldots,z_{\ell+4}}
\min_{\ww_1,\ww_2}
\max_{T,y_1,y_2,x_1,\ldots,x_{\ell}} 
\max_{\ww_3} \ldots Q^A_{\ww_k}
$$

We next show that the claim C1 holds for the extended reduction.\footnote{
Note that we can remove the variable consistency block which complicates the argument (and does not test anything per legality of $I$) and still maintain correctness of C1. But including it here simplifies the argument for diagram $B$ below and thus simplifies the overall proof.
}

We first consider all possible cases for illegal interpretations. 
\begin{itemize}
\item
If the interpretation has $\geq \ell+4$ objects then the top portion of the diagram yields 0 for some valuation of $z$'s
regardless of the values of other variables. Therefore aggregation over $z$ yields 0, and then aggregations over $r$ and $m$ yield 0.
\item
Next, consider the case where the interpretation has $\leq \ell+3$ objects but one of the unary predicates is always false (i.e., it does not ``pick" any object).
The situation is similar to the previous case, but here we get a value of 1 for $\zeta$ where $r_1=r_2$ or where there is a block violation for some $\ww_i$ block with $\min$ aggregation.

Consider any valuation $\zeta_p$ to the prefix of variables up to $x_\ell$ in aggregation order
which is block consistent on ${\ww_1,\ww_2}$, has  
$r_1\not = r_2$,
and any valuation
for the other variables.
We have two cases: if $\zeta$ is in group 1 the map is 0 because the unary predicate test fails, and if $\zeta$ is in group 2 the map is the default value of the first violated block $\ww_i$. Now, because there is at least one group 1 valuation, we can argue inductively backwards from $k$ that all aggregations from $k$ to $3$ yield 0.
Therefore, the maximization over ${T,y_1,y_2,x_1,\ldots,x_{\ell}}$ yields 0, and the minimization over $z,\ww_1,\ww_2$ yields 0. 
In the minimization over $r$, we have 0 for the cases where 
$r_1\not=r_2$, and 1 otherwise. Therefore the minimization over $r$ yields 0 as well, and the minimization over $m$ yields 0. 
\item
Next consider the case where the previous two conditions are satisfied but where one of the unary predicates holds for two or more objects. The argument is identical to the previous case, except that $\zeta$
has  the violating pair for $r_1, r_2$ (instead of any $r_1\not = r_2$).
\item
Next consider any interpretation that has exactly $\ell+3$ objects and where the unary predicates identify the objects corresponding to ${T,y_1,y_2,x_1,\ldots,x_{\ell}}$ but where $q()$ is not symmetric.
In this case we consider (as in the second item on this list)
any valuation $\zeta_p$ to the prefix of variables up to $x_\ell$ in aggregation order
which is block consistent on ${\ww_1,\ww_2}$, has  
$r_1\not = r_2$, has the violating pair for $m_1,m_2$
and any valuation
for the other variables.
As above, we can argue that the aggregations down to $\ww_3$ and then down to $r$ yield 0 when $r_1\not = r_2$ and $m_1,m_2$ has the violating pair. 
As a result, for this setting of $m$, minimization over $r$ yields 0, and therefore the minimization over $m$ also yields 0. 
\item
The only remaining cases are
interpretations that are illegal only because the $q()$ simulation of $P_T()$ or $P_{x_i}()$ is not as required. In this case, the same argument as in the 2nd item in this list shows that the value is 0. 
\end{itemize}

Therefore, if $I$ is illegal then $\map_{B_1}(I)=0$.
Consider next any legal interpretation and the intended valuation $\zeta_p$ to $T,y_1,y_2,x_1,\ldots,x_{\ell}$. For any group 1 extension of this valuation and any valuation of the variables $m,z,r$ the diagram yields 1. Therefore, we can argue inductively that all $\ww_i$ aggregations down to $\ww_3$ yield 1.
This implies that the maximization over $T,y_1,y_2,x_1,\ldots,x_{\ell}$ yields 1, and the remaining minimizations yield 1. 
This completes the proof of C1.

The diagram $B$ is obtained by adding the clause blocks below the $q()$ tests of $B_1$.
The clause blocks have the same structure as above but they use $q()$ instead of $P_T()$ and $P_{x_i}()$. 
This is shown in Figure~\ref{fig:gfodd-equiv-sorted-clauses}.
The final diagram for $B$ is shown in Figure~\ref{fig:gfodd-equiv-sorted-B} [as above the only difference is that we flip the leaf values in the consistency block] and it has the same aggregation function as $B_1$. 
Note that the diagram is also sorted using the same order as in $B_1$ where we have in addition that $x_i \prec V_{(i_1,i_2)}$. We next show that: 
\\
(C2') $\map_B(I)=1$ if and only if 
$I$ is legal and $Q_2\xx_2 \ldots Q_k\xx_k f((\xx_1=\alpha),\xx_2,\ldots,\xx_k) = 1$
where $I$
embeds the substitution $\xx_1=\alpha$. 

The same argument from the simple case can be used to show that C1,C2' imply that $B_1$ and $B$ are equivalent iff the QBF is satisfied.

To prove the claim first note that by the construction $B$ adds more tests on the path to a 1 leaf of $B_1$ and does not add any other paths to a value of 1. Therefore, 
for any $I$ and any $\zeta$, if $\map_B(I,\zeta)=1$ then $\map_{B_1}(I,\zeta)=1$ and as a result if $\map_B(I)=1$ then $\map_{B_1}(I)=1$. Therefore, by C1, if $\map_B(I)=1$ then $I$ is legal.
Next, consider any legal $I$ and any unintended valuation $\zeta_p$ for the block of variables $T,y_1,y_2,x_1,\ldots,x_{\ell}$.
As above, the aggregated value down to $\ww_3$ for this prefix is 0. Therefore, if $\map_B(I)=1$ and thus the $\max$ aggregation over these variables yields 1 (for the prefix valuation for $M,R,Z,\ww_1,\ww_2$), it must be through the intended valuation for $T,y_1,y_2,x_1,\ldots,x_{\ell}$. 
However, 
when $T,y_1,y_2,x_1,\ldots,x_{\ell}$ are fixed to their intended values,
the portions of the diagram testing for $\leq q+3$ objects, the uniqueness of special objects $T,y_1,y_2,x_1,\ldots,x_k$, the symmetry of $q()$ and its simulation of $P_T()$ and $P_{x_i}$
do not affect the final value in the sense that a valuation reaching them always continues to the next block. 
Therefore, if we restrict attention to such valuations we can shrink $B$ by removing these blocks and still obtain the same aggregated value. Now we observe that there is a 1-1 correspondence between valuations and values of the resulting $B$ to valuations and values of $B_2$ in the simple construction
(where we extend the notion of block consistent to enforce that $\ww_i$ bind to $y_1,y_2$). Therefore,
the claim holds by C2 of the simple construction. 
\end{proof}

%% file: minGfoddSatProof.tex
\subsection{Proof of Theorem~\ref{thm:gfodd-min-sat-k1}}

\begin{proof}
We reduce QBF satisfiability with $2$ alternations of quantifiers to satisfiability of min GFODDs. 
The construction follows the same steps as in Theorem~\ref{thm:gfodd-min-sat} 
except that $k=2$ and we swap the $\max$ variables with constants.

In particular, here we assume a QBF whose first quantifier is $\exists$,  that is,
$\exists x_1, Q_2x_2 \ldots Q_mx_m$ $f(x_1,x_2,\ldots,x_m)$ where this form has 
$2$ blocks of quantifiers. To simplify the notation it is convenient to group adjacent variables having the same quantifiers into groups so that the QBF 
has the form
$\exists \xx_1 \forall \xx_2 f(\xx_1,\xx_2)$ where $\xx_i$ refers to a set of variables.

We next define a notion of ``legal interpretations" for our diagrams. 
A legal interpretation embeds the binary interpretation $I^*$ from previous proofs and in addition includes a truth setting for all the variables in the first $\exists$ block of the QBF. 
The reduction constructs diagrams $B_1$, $B_2$, and $B=apply(B_1,B_2,\wedge)$ such that the following claims hold:
\\ 
(C1) for all $I$, $\map_{B_1}(I)=1$ if and only if $I$ is legal.
\\
(C2)
if $I$ is legal and it embeds the substitution $\xx_1=\alpha$ then $\map_{B_2}(I)=1$ if and only if $\forall \xx_2 f((\xx_1=\alpha),\xx_2) = 1$.

We then output the diagram $B$ for GFODD satisfiability. 
Now, if the QBF is satisfied then there exists a value $\alpha$ such that for $\xx_1=\alpha$ we have that $\forall\xx_2 f((\xx_1=\alpha),\xx_2) = 1$. Therefore, by C2, for the legal $I$ that embeds $\alpha$, 
$\map_{B_2}(I)=1$. 

On the other hand, if the QBF is not satisfied then for all substitutions $\xx_1=\alpha$ we have $\forall \xx_2 f((\xx_1=\alpha),\xx_2) = 0$. Therefore, by C2, 
all legal $I$ (and any $\alpha$ they embed)
$\map_{B_2}(I)=0$ and by Theorem~\ref{thm:gfodd-combine}
we also have $\map_{B}(I)=0$. By C1, $\map_{B}(I)=0$ for non-legal interpretations. Therefore, $B$ is not satisfiable.

We now proceed with the reduction, starting first with a simplified construction ignoring ordering of node labels, 
and then elaborating to enforce these constraints. 
The set of predicates includes $P_T()$ which is as before and for every QBF variable $x_i$ in the first $\exists$ block we use a predicates $P_{x_i}()$. 
Notice that each $x_i$ is a member of $\xx_1$ (the first $\exists$ group) where the typeface distinguishes the individual variables in the first block, from blocks of variables.
In the simplified construction, a legal interpretation has exactly two objects, say $a$ and $b$, where $P_T(a)\not = P_T(b)$ and where for each $P_{x_i}()$ we have $P_{x_i}(a) = P_{x_i}(b)$. That is, the assignment of an object to $v$ in $P_T(v)$ simulates an assignment to Boolean values, but the truth value of $P_{x_i}(v)$ is the same regardless of which object is assigned to $v$. In addition we extend the signature to include $y_1,y_2$ as constants and a legal interpretation maps $y_1,y_2$ to $a,b$ in a 1-1 manner.

Consider our example QBF modified to start with $\exists$ quantifier and to have two alternations
$\exists x_1 \forall x_2 \forall x_3 \forall x_4 (x_1 \vee \bar{x_2}  \vee x_4) \wedge (\bar{x_1} \vee x_2 \vee x_3) \wedge (x_1 \vee x_3 \vee \bar{x_4})$.
The first block includes only the variable $x_1$ and the following interpretation is legal:
$I = \{ [a, b], [y_1/a, y_2/b], P_T(a)=$true, $P_T(b)=$false, $P_{x_1}(a) = P_{x_1}(b)=$false$\}$.

The diagram $B_1$ has three portions where the first two are exactly as in the previous proof, thus verifying that $I$ has two objects and that $P_T()$ behaves as stated. The third portion verifies that each $P_{x_i}()$ behaves as stated, where we use a sequence of blocks, one for each $P_{x_i}()$.
The combined diagram $B_1$ is shown in Figure~\ref{fig:gfodd-equiv-B1} and the aggregation function is $\min_{z_1,z_2,z_3}$.

To see that C1 holds consider all possible cases for non-legal interpretations. If $I$ has at most one object the map is 0 for all valuations and thus the aggregation is 0. The same holds if $y_1,y_2$ are mapped to the same object. If $I$ has at least 3 objects, then the min aggregation over $z$ yields 0. If $I$ has 2 objects but it violates the condition on $P_T$ or $P_{x_i}$ then again the map is 0 for any valuation and the aggregation is 0. On the other hand, if $I$ is legal, then for any assignment to $z$, the diagram yields 1. Therefore the aggregation over $z$ yields 1. 

The diagram $B_2$ is constructed by modifying $B_2$ from the proof of Theorem~\ref{thm:gfodd-eval}. The first modification is to
handle the first $\exists$ block differently.  
As it turns out, all we need to do is replace the $\max$ aggregation for the $\ww_1$ block with $\min$ aggregation and accordingly replace the default value on that block to 1. The modified variable consistency diagram is shown in [we modify leaf values again because of the change of quantifiers in the current example; in this construction all exit values are 1] Figure~\ref{fig:gfodd-equiv-consistency}. 
The clause blocks have the same structure as in the previous construction but use $P_{x_i}(V_{(i_1,i_2)})$ when $x_i$ is a $\exists$ variable from the first block and use $P_{T}(V_{(i_1,i_2)})$ otherwise.
This is shown in Figure~\ref{fig:gfodd-equiv-clauses}. 
$B_2$ includes the variable consistency blocks followed by the clause blocks. Note that the new clause blocks are not sorted in any consistent order because the predicates $P_{x_i}()$ and $P_{T}()$ appear in an arbitrary ordering determined by the appearance of literals in the QBF. Other than this violation, all other portions of the diagrams described are sorted where the predicate order has $= \prec P_T \prec P_{x_i}$ and where variables $w_i$ are before $V_{(i_1,i_2)}$ and variables within group are sorted lexicographically. 
The combined aggregation function is $\min_{\ww_1}, \min_{\ww_2}$.

We next show that claim C2 holds, which will complete the proof of the simplified construction. 
Consider any legal $I$, let the corresponding truth values for variables in $\xx_1$ be denoted $\alpha$, and consider valuations for the QBF extending $\xx_1=\alpha$. Now consider any valuation $v$ to the remaining variables in the QBF and the induced substitution to the GFODD variables $\zeta(v)$ that is easily identified from the construction. Add any consistent group assignment to $\ww_1$ (that is, we assign $a$ or $b$ to all variables in that group)
to $\zeta(v)$ to get $\hat{\zeta}(v)$.
By the construction of $B_2$ we have that $f([\xx_1=\alpha ,(\xx_2)=v]) = \map_{B_2}(I,\hat{\zeta}(v))$. To see this note that there are no quantifiers in this expression, there is a 1-1 correspondence between the valuations of $\xx_2$ and $\ww_2$, and that as long as the assignment to the $\ww_1$ block is group consistent it does not affect the value returned. 
We call this set of valuations, that arise as translations of substitutions for QBF variables, {\em Group 1}.

The second group, {\em Group 2}, includes valuations that do not arise as $\hat{\zeta}(v)$ and therefore they violate at least one of the consistency blocks. 
Let  $\zeta$ be such a valuation
and let $Q^A_j$ be the first block from the left whose constraint is violated. By the construction of $B_2$, in particular the order of equality along paths in the GFODD, we have that the evaluation of the diagram on $\zeta$ ``exits" to a default value on the first violation. 
In the current construction there are only two blocks and the default value for both blocks is 1. 
Now, since all aggregation operators are $\min$ operators, the final aggregated value in both blocks is determined by group 1 valuations. This implies that
$\forall  \xx_2, f(\xx_1=\alpha,\xx_2) = 
\min \ww_2, \map_{B_2}(I,[\ww_1=a,\ww_2])$ as required.

\paragraph{Extending the reduction to handle ordering:}
The main idea in the extended construction is to replace the unary predicates $P_T$ and $P_{x_i}$ with one binary predicate $q(\cdot,\cdot)$ where the ``second argument" in $q()$ serves to identify the corresponding predicate and hence its truth value. In addition we force $q()$ to be symmetric so that for any $A$ and $B$ the truth value of $q(A,B)$ is the same as the truth value of $q(B,A)$. In this way we have freedom to use either $q(A,B)$ or $q(B,A)$ as the node label which provides sufficient flexibility to handle the ordering issues. To implement this idea we need a few additional constructions.

Let the set of variables in the first $\exists$ block of the QBF be $x_1,x_2,\ldots,x_\ell$. 
The set of predicates in the extended reduction includes
unary predicates $T(),  y_1(), y_2(), x_1(), x_2(),\ldots,x_\ell(), $ and one binary predicate $q(\cdot,\cdot)$. A legal interpretation includes exactly $\ell+3$  objects which are uniquely identified by the unary predicates. 
In addition, the same objects are identified by $\ell+3$ constants in the signature. 
We therefore slightly abuse notation and use the same symbols for 3 different entities: the objects, the constants identifying them,  and the predicates identifying them. 
In particular, the  atoms $y_1(y_1), y_2(y_2)$, $T(T)$, and $x_1(x_1), x_2(x_2), \ldots, x_\ell(x_\ell)$ are true in the interpretation and only these atoms are true for these unary predicates (e.g., $x_1(T)$ is false).
The truth values of $q()$ reflect the simulation of $P_T()$ and $P_{x_i}()$ in addition to being symmetric. 
Thus  the truth values of $q(y_1,T)$ and $q(T,y_1)$ are the same and they are the negation of the truth values of $q(y_2,T)$ and $q(T,y_2)$. 
For all $i$, the truth values of $q(y_1,x_i)$, $q(x_i,y_1)$, $q(y_2,x_i)$ and $q(x_i,y_2)$ are the same.
The truth values of other instances of $q()$, for example, $q(x_2,T)$, can be set arbitrarily. For example, the following interpretation is legal when $\ell=1$:
$I = \{ [a, b, c, d], [y_1/a, y_2/b, T/c, x_1/d], y_1(a), y_2(b), T(c), x_1(d), q(c,a)=q(a,c)=$true, $q(c,b)=q(b,c)=$false, $q(d,a)=q(a,d)=q(d,b)=q(b,d)=$false, $q(\cdot,\cdot)=$false$\}$ where $q(\cdot,\cdot)$ refers to any instance not explicitly mentioned in the list.

We next define the diagram $B_1$ that is satisfied only in legal interpretations. 
We enforce exactly $\ell+3$ objects using two complementary parts. The first includes ${\ell+4 \choose 2}$ inequalities on a new set of $\ell+4$ variables $z_1,\ldots,z_{\ell+4}$ with min aggregation. If we identify $\ell+4$ distinct objects we set the value to 0.
This is shown in Figure~\ref{fig:gfodd-equiv-few-objects}.

To enforce at least $\ell+3$ objects and identify them we use the following gadget. For each of the unary predicates we have a diagram identifying its object and testing its uniqueness using min variables. This is shown for the predicate $T()$ in Figure~\ref{fig:gfodd-equiv-unary-T}. The node $T(T)$ with max variable $T$ identifies the object $T$. The nodes $T(r_1),T(r_2)$ with min variables $r_1,r_2$ make sure that $T$ holds for at most one object. 
We chain the diagrams together as shown in Figure~\ref{fig:gfodd-equiv-sorted-unary} where the variables $r_1,r_2$ are shared among all unary predicates. The corresponding aggregation function is $ \min_{r_1,r_2}$. This diagram associates each of the $\ell+3$ objects with one of the unary predicates and in this way provides a reference to specific objects in the interpretation. [Note that we could use the constants and some inequalities for the same purpose but here simply adapt the previous construction for consistency of presentation].

The symmetry gadget for $q()$ is shown in Figure~\ref{fig:gfodd-equiv-q-symmetric} where the variables $m_1,m_2$ are min variables. If an input interpretation has two objects $A,B$ where $q(A,B)$ has a truth value different than $q(B,A)$ then minimum aggregation will map the interpretation to 0.

The truth value gadget for the simulation of $P_T$ is shown in Figure~\ref{fig:gfodd-equiv-q-of-PT} and 
the truth value gadget for the simulation of $P_{x_i}$ is shown in Figure~\ref{fig:gfodd-equiv-q-of-xi}.

Finally, we add a component that is not needed for verifying that $I$ is legal but will be useful later when we include the clauses.
In particular, we include a variable consistency block which is similar to the one in the simple construction but where we force $\ww_i$ to bind to the same object as $y_1$ or $y_2$. This is shown in [all leaf values change to 1 in this example] Figure~\ref{fig:gfodd-equiv-sorted-consistency}.

We chain the diagrams together 
as shown in Figure~\ref{fig:gfodd-equiv-sorted-B1} to get $B_1$ where we have moved the node labelled $r_1=r_2$ to be above the block consistency gadget
[here too leaf values need to be adjusted in the consistency block; we also need a slight  reordering of nodes so that the symmetry block testing $m_1,m_2$ is below the block testing simulation of $P_{x_i}$].
Note that the diagram is sorted where for predicate order we have $= \prec T \prec y_1 \prec y_2 \prec x_i \prec q$ and for variables we have 
$z_i,T,y_i,x_i \prec r_j \prec w_l$, and $y_i \prec V_{(i_1,i_2)}$, and 
[was $m_i \prec T \prec x_j$ now change to]
$T \prec m_i \prec x_j$.
The complete aggregation function is 
$$
\min_{m_1,m_2}\min_{r_1,r_2}\min_{z_1,\ldots,z_{\ell+4}}
\min_{\ww_1,\ww_2}
$$

We next show that the claim C1 holds for the extended reduction.\footnote{
Note that we can remove the variable consistency block which complicates the argument (and does not test anything per legality of $I$) and still maintain correctness of C1. But including it here simplifies the argument for diagram $B$ below and thus simplifies the overall proof.
}

We first consider all possible cases for illegal interpretations. Since $B_1$ is a min diagram it is clear that if a zero leaf is reached along any path then the aggregated value is 0.
If the interpretation has $\geq \ell+4$ objects then the top portion of the diagram yields 0 for some valuation of $z$'s. The same is true if any of the unary predicates does not identify any variable, or is true for more than one variable, or if $q()$ is not symmetric, or $q()$'s simulation of $P_T()$ and $P_{x_i}$ are not as intended, or if the contestant mapping is not as expected (this leads to 0 via the unary predicates).
Therefore, if $I$ is illegal then $\map_{B_1}(I)=0$.

Consider next any legal interpretation. For any group 1 valuation of $\ww_1,\ww_2$ and any valuation of the variables $m,z,r$ the diagram yields 1. 
Since group 2 valuations always yield 1 in this construction the final value is 1.
This completes the proof of C1.

The diagram $B$ is obtained by adding the clause blocks below the $q()$ tests of $B_1$.
The clause blocks have the same structure as above but they use $q()$ instead of $P_T()$ and $P_{x_i}()$. 
This is shown in Figure~\ref{fig:gfodd-equiv-sorted-clauses}.
The final diagram for $B$ is shown in Figure~\ref{fig:gfodd-equiv-sorted-B} 
[here too leaf values need to be adjusted in the consistency block] 
and it has the same aggregation function as $B_1$. 
Note that the diagram is sorted using the same order as in $B_1$.

We next show that $\map_B(I)=1$ if and only if 
$I$ is legal and $\forall \xx_2  f((\xx_1=\alpha),\xx_2) = 1$
where $I$
embeds the substitution $\xx_1=\alpha$. 
This has the same consequences as having $B=apply(B_1,B_2,\wedge)$ in the simple construction.

To prove the claim first note that by the construction $B$ adds more tests on the path to a 1 leaf of $B_1$ and does not add any other paths to a value of 1. Therefore, 
for any $I$ and any $\zeta$, if $\map_B(I,\zeta)=1$ then $\map_{B_1}(I,\zeta)=1$ and as a result if $\map_B(I)=1$ then $\map_{B_1}(I)=1$. Therefore, by C1, if $\map_B(I)=1$ then $I$ is legal.
Finally, when the interpretation is legal the portions of the diagram testing for $\leq q+3$ objects, the uniqueness of special objects $T,y_1,y_2,x_1,\ldots,x_k$, the symmetry of $q()$ and its simulation of $P_T()$ and $P_{x_i}$
do not affect the final value in the sense that a valuation reaching them always continues to the next block. 
Therefore, for legal interpretations we can shrink $B$ by removing these blocks and still obtain the same aggregated value. Now we observe that there is a 1-1 correspondence between valuations and values of the resulting $B$ to valuations and values of $B_2$ in the simple construction
(where we extend the notion of block consistent to enforce that $\ww_i$ bind to $y_1,y_2$). Therefore,
the claim holds by C2 of the simple construction. 
\end{proof}